\documentclass[11pt]{article}

\usepackage{amsfonts,amssymb,amsmath,amsthm,epsfig,euscript,bbm,amsthm}
\usepackage[usenames,dvipsnames]{xcolor}
\usepackage{algorithm}
\usepackage{algpseudocode}
\usepackage{amsthm}
\usepackage{tikz}
\usetikzlibrary{arrows}
\usepackage{graphicx}
\usepackage{xcolor}
\usepackage{enumerate}
\usepackage{subcaption}
\usepackage{caption}
\usepackage{natbib}
\usepackage[titletoc,title]{appendix}
\usepackage[pdftex,colorlinks]{hyperref}

\usepackage{jbradic_definitions}

\hypersetup{
colorlinks =true,
citecolor    = blue,
}

  \renewenvironment{thebibliography}[1]{%
    \begin{oldthebibliography}{#1}%
      \setlength{\parskip}{0ex}%
      \setlength{\itemsep}{0ex}%
  }%
  {%
    \end{oldthebibliography}%
  }

 \newtheorem{thm}{Theorem}
\newtheorem{cor}{Corollary}

\usepackage[T1]{fontenc}
\usepackage{booktabs}
\usepackage{dcolumn}
\usepackage{units}
\usepackage{array}

\makeatletter

\newcolumntype{f}{>{$}l<{$}}
\newcolumntype{n}{l}
\newcolumntype{N}{>{\scriptsize}l}
\newcolumntype{v}[1]{>{\raggedright\hspace{0pt}}p{#1}}
\newcolumntype{V}[1]{>{\scriptsize\raggedright\hspace{0pt}}p{#1}}
\newcolumntype{B}[1]{>{\boldmath\DC@{.}{,}{#1}}l<{\DC@end}}
\newcolumntype{d}[1]{>{\DC@{.}{,}{#1}}l<{\DC@end}}
\newcolumntype{i}[1]{>{\DC@{.}{,}{#1}\mathnormal\bgroup}l<{\egroup\DC@end}}
\newcolumntype{s}[1]{>{\DC@{.}{,}{#1}\mathsf\bgroup}l<{\egroup\DC@end}}
\newcolumntype{R}[1]{%
  >{\begin{turn}{90}\begin{minipage}{#1}\scriptsize\raggedright\hspace{0pt}}l%
  <{\end{minipage}\end{turn}}%
}
\newcolumntype{x}{>{\scriptsize\raggedright\hspace{0pt}}X}
\makeatother

\definecolor{coquelicot}{rgb}{1.0, 0.22, 0.0}
 
\newcommand{\jelena}[1]{\textcolor{coquelicot}{\bf Jelena: #1}}

\long\def\symbolfootnote[#1]#2{\begingroup
\def\thefootnote{\fnsymbol{footnote}}\footnote[#1]{#2}\endgroup}

\newcommand{\Prob}{\mathbb{P}}

\newcommand{\realR}{\mathbb{R}}

\newcommand{\E}{\mathbb{E}}

\newcommand{\1}{\mathbbm{1}}
\newcommand{\cref}[1]{Corollary \ref{corollary:#1}}

\newcommand*{\addheight}[2][.5ex]{%
  \raisebox{0pt}[\dimexpr\height+(#1)\relax]{#2}%
}

\newtheorem*{condition1}{Assumption 1}
\newtheorem*{condition2}{Assumption 2}

\title{  Boosting in the presence of outliers: 
\\
adaptive classification with non-convex loss
functions }

\author{
 Alexander  Hanbo Li$^{\dag}$ and Jelena Bradic$^{\dag}$ \\ [1em]
\it  Department of Mathematics\\
\it University of California at San Diego$^{\dag}$ \\
}
\date{}

\begin{document}
\maketitle

\begin{abstract}
 This paper  examines  the  role  and efficiency of  the non-convex loss functions   for binary classification problems. In particular, we investigate how to design a  simple and   effective boosting algorithm that is robust to the outliers in the data.  The analysis of  the role of a particular non-convex loss for prediction accuracy varies depending on the diminishing  tail properties  of the gradient of the loss -- the ability of the loss to efficiently adapt to  the outlying data,  the local convex properties of the loss and the proportion of the contaminated data.   In order to use these properties efficiently, we propose a new family of non-convex losses  named $\gamma$-robust losses.  Moreover, we present a new  boosting framework, {\it Arch Boost},  designed  for augmenting the existing work   such that its corresponding classification algorithm is significantly more adaptable to the   unknown data contamination. 
 Along with the Arch Boosting framework, the non-convex losses lead to
 the new class of boosting algorithms, named adaptive, robust, boosting (ARB).
 Furthermore, we present theoretical examples that demonstrate the robustness properties of the proposed algorithms. In particular, we develop a new  breakdown point analysis and a new    influence function analysis  that demonstrate gains in robustness.   Moreover, we   present new theoretical results, based only on local curvatures, which may be used to establish statistical and optimization properties of the proposed Arch boosting algorithms with highly non-convex loss functions. Extensive numerical calculations are used to illustrate these theoretical properties and  reveal advantages over the existing boosting methods when data exhibits a number of outliers.

   \end{abstract}

\section{Introduction} \label{introduction}
 Recent advances  in technologies for cheaper and faster data acquisition and storage  have led to an explosive growth of data complexity in a variety of research areas such as high-throughput genomics, biomedical imaging, high-energy physics, astronomy and  economics. As a result,  noise accumulation, experimental variation and data inhomogeneity have become substantial. Therefore, developing  classification methods that are highly efficient and accurate in such settings,  is a   problem that is  of great practical importance.  However,
classification in such settings is known to poses many statistical challenges and calls for new methods and theories.
For binary classification problems, we assume  the presence of separable, noiseless data that belong to two classes and in which an adversary has corrupted a number of observations from both classes independently.  
There are a number of setups that belong to this general framework. 
A random flipped label design,  in which the   labels of the class membership  were randomly flipped, is one example that can occur very frequently, as labeling is prone to a number of errors, human or otherwise.
  Another example   is the presence of outliers in the observations, in which   a small number of observations from both classes have a variance that is larger than the noise of the rest of the observations. Such situations may naturally occur with the new era of big and heterogeneous data,  in which data are corrupted (arbitrarily or maliciously) and subgroups may behave differently; a subgroup might
only be one or a few individuals in  small studies that   would appear to be outliers within class data.  

Considerable effort has therefore been focused on finding methods that adapt to the relative error in the data.
 Although this has resulted in algorithms, e.g. \cite{GD04}, that achieve provable guarantees \citep{NDRT:13,K07} when contamination model \citep{Scott:13} is known or when multiple noisy copies of the data are available \citep{CSS+:11}, good generalization errors in the test set are by no means guaranteed. This problem is compounded when the contamination model is unknown, where outliers need to be detected automatically. Despite progress on outlier-removing algorithms, significant practical challenges (due to exceedingly restrictive conditions imposed therein) remain.
In this paper, we concentrate on the ensemble algorithms.
Among these,    AdaBoost   \citep{FreundSchapire:97} has 
  proven to be simple and effective in solving classification problems of many different kinds.
 The aesthetics and simplicity of AdaBoost and other forward greedy algorithms, such as LogitBoost
\citep{FHT:00}, also facilitated a tacit defense from overfitting, especially when
combined with early termination of the algorithm \citep{ZY05}.     
  \cite{FHT:00} developed a powerful statistical perspective,
which  views  AdaBoost as a gradient-based incremental search for a good
additive model using the exponential loss.
 The gradient boosting \citep{Friedman:99} and
AnyBoost \citep{Mason:99}   have used this approach to generalize the boosting
idea to wider families of problems and loss functions.  This criterion was motivated by the fact that the exponential loss is a convex surrogate of the hinge  or $0-1$ loss. Nevertheless, in the presence of label noise and/or outliers, the performance of all of them deteriorates rapidly \citep{D00}.  Although  algorithms like LogitBoost,  MadaBoost \citep{DomingoWatanabe:00},  Log-lossBoost \citep{CollinsSchapireSinger:02} are able to better tolerate noise than AdaBoost, they are still not insensitive to outliers. Hence, they are efficient when the data is observed with little or no noise. However, \cite{LS:10} pointed out that any boosting algorithm with convex loss functions is highly susceptible to a random label noise model. They constructed a simple example, from hereon denoted  Long/Servedio problem, that cannot be ``learned'' by the boosting algorithms above.  
  
Center to our analysis is the work by
\cite{Freund:09}. He proposed a  robust boosting algorithm based on  the Boost by majority (BBM) \citep{Freund:95}  and BrownBoost \citep{Freund:01} algorithm, which uses a non-convex loss function. Instead of  maximizing the margin, the algorithms achieve robustness by allowing  a preassigned $\theta$ error of margin maximization \citep{Servedio:03}. Moreover, in each step, the algorithms update and  solve a differential equation and update the preassigned remaining time $c$  or   the target error $\epsilon$. As the loss function changes with each iteration of the algorithm,  they do not agree with the general boosting interpretation  through additive models. Furthermore, with at least two preassigned parameters, each of them  is  difficult to implement and is highly inconsistent  with respect to minor changes in the settings.   Surprisingly, statistical and convergence properties of these two algorithms are still unknown. 
This leads to a natural question: how do we     develop  a simple but   effective boosting algorithm that has a non-convex loss function, that preserves boosting interpretation and that is robust to the noise in the data?  In this paper, we address this question and propose a fully automatic estimator, with no tuning parameters to be chosen, that has provable guarantees. 

To design the new framework we will explore and amend the drawbacks of the AdaBoost algorithm in the contaminated data setting. We successfully  identify that the
 AdaBoost's   sensitivity to outliers  comes from   the unbounded   weight assignment of the   misclassified observations. As outliers are more likely to be misclassified,  they are very likely to be assigned large weights and will be repeatedly refitted in the following iterations. This refitting will deteriorate seriously   the generalization performance of the algorithm, as the algorithm ``learns'' incorrect data distribution.  To achieve robustness, the algorithm should be able to abandon observations that are  on the extreme, incorrect side of the boundary.   Here, we theoretically and computationally investigate the applicability of non-convex loss functions for this purpose.     We  illustrate  that  the best weight updating rule is to assign a weight of $-\phi^{'}(y_i F(x_i))$ to each data point $(x_i,y_i)$, in which $\phi$ is the appropriate non-convex loss function.  
 We use a tilting argument, with non-convex losses. It is shown 
 that, if we use a non-convex loss, sufficiently tilted, i.e. such that  $-\phi^{'}(v)$ is small for all $v \ll 0$, then the outliers are eliminated successively. Hence, constant tilting or ``trimming'' is not sufficient for outlier removal.  In tilting the loss function, we are effectively preserving as much fidelity to the data as possible, while redistirbuting emphasis to different observations.
%
%
%
%
%
We propose a new {\it Arch Boosting} framework that  implements  the above tilting method and adjusts for optimality by a new search of the optimal weak hypothesis.  Moreover, we show that  the new framework  avoids overfitting much similar to the AdaBoost.  We  propose  a sufficient set of conditions   needed for a  loss function   to  allow for    good properties of the ArchBoost. We show that not every non-convex function satisfies such conditions; an example is the  sigmoid loss. 
However, we propose a family of loss functions that  balances  both the benefits of non-convexity and the empirical risk interpretation of boosting. 


 Properties  of  the boosting algorithms based on convex loss functions have been extensively studied (e.g. \cite{K03,ZY05}). Comparatively little is known for the non-convex losses, as the existing techniques  do not apply.  
We show that local convexity properties are sufficient for statistical consistency.
Furthermore, even though the proposed loss function is shown to enjoy the aforementioned local  convexity, it is largely unknown whether numerical algorithms can identify this local minimizer. Moreover, as our algorithm is not defined as a  gradient descent algorithm, we require a new approach for the proof of numerical convergence.
We develop a  new  sufficient optimality condition based on the ``hardness condition''   in the technical proofs.
By hardness property, we mean the orthogonality of the  new reweighted classifier   and the class membership vector.
Furthermore, we address the robustness and efficiency of the proposed method, with respect to the outliers. 
Although it is straightforward to provide such analysis for parametric linear models,   computations for classification with the nonparametric boundaries are far more challenging. We provide a novel  analysis, for which we propose a new finite sample breakdown point theory \citep{H68} and  show that    the influence function \citep{Hampel:74} is bounded  for appropriate class of classification problems.
To the best of our knowledge, this is the first result regarding  the robustness properties of the boosting algorithms, with respect to the presence of outliers. Our analysis allows for both  convex and non-convex loss function. 
We finalize the analysis with a proof of statistical consistency of the proposed method that includs many non-convex losses; we do so by  exploring local curvatures of the loss.

 In essence, this paper investigates the effects of  non-convex losses on a variety of boosting   algorithms in the  presence of unknown contamination of the data. In particular, we focus on how to design a new boosting framework  in order to improve the prediction accuracy of classification methods for the data with outliers.    
 The rest of the paper is organized as follows.
 We present a new {\it Arch boosting} framework  in Section \ref{Arch boost}; it is  designed  for augmenting the  boosting framework  such that its corresponding classification algorithm is significantly more adaptable to outliers. Section \ref{sec:3}  outlines a new family of loss functions that explores non-convexity  and present sufficient conditions for non-convex loss to be robust.  We present   theoretical analysis in Section \ref{theory}:  the  numerical and statistical convergence of the proposed algorithms are discussed in \ref{sec:4} and \ref{sec:5}, respectively. Moreover, we show theoretical robust properties in the Section \ref{sec:robust} with the breakdown point discussed in Section \ref{sec:bp} and influence function in Section \ref{sec:if}. Section \ref{experiments} contains numerical experiments on a number of examples of the loss functions belonging to the introduced family:  $\gamma$-robust loss, least squares, logistic, exponential  and truncated exponential loss.    We demonstrate both how to use these methods in practice  and compare them to the alternative of applying the non-augmented AdaBoost algorithm to the noisy data. The subsection \ref{gamma}     varies the  $\gamma$ parameter and  considers  examples of a   contaminated Gaussian distribution. All examples clearly illustrate  that the methods outlined in Section \ref{sec:3} are  far  more successful than the existing boosting methods.  Section \ref{ls} deals with the more complex situation of  Long/Servedio data for which we show that our Arch Boost method  outperforms the robust boosting method of  \cite{Freund:09}.  We also discuss outlier detection examples in Section \ref{outliers}, and apply our methods to three real datasets in Section \ref{sec:realdata}.

\section{Arch Boost} \label{Arch boost}

We consider a binary classification problem, with $\mathcal{X}$ denoting the domain of the $d$-dimensional variable $X$, and $\mathcal{Y}$ denoting the class label set that equals $Y \in \{-1,1 \}$. 
We aim to estimate a function $F(X):\mathbb R^p \to \mathbb R$ and assume that the training data $\{(X_i,Y_i), i=1,\dots,n \}$ are i.i.d. copies of $(X,Y)\in \mathcal{X} \times \mathcal{Y}$ with unknown distribution. The data consists of samples from the contaminated distribution that is composed of the true (uncontaminated) data and a fixed and unknown number of outliers in each of the classes.

From now on, we  let $\phi$ denote any differentiable loss function. Note that   $\phi$ does not need to be convex necessarily.  For such $\phi$, we define the $\phi-$risk $R_{\phi}$ and the empirical $\phi-$risk $\hat{R}_{\phi,n}$   as
\begin{eqnarray} \label{eq:3}
R_{\phi}(F) = \E_{(X,Y)\sim \mathbb{P}} [\phi(YF(X))],\; \;
\hat R_{\phi,n}(F) = \frac{1}{n} \sum_{i=1}^n \phi(Y_iF(X_i)),
\end{eqnarray}
where $n$ is the sample size, $\mathcal{S}_n=\{(X_i,Y_i)\}_{i=1}^n$ are i.i.d observations, and $\mathbb{P}$ is the true probability distribution on $\mathcal{X} \times \mathcal{Y}$. For simplicity of notation, we   write $R_{\phi}(F) = \E[\phi(YF(X))]$.  Note that  the observed samples $(X_i,Y_i)$ can come from a contaminated distribution, i.e.,  they are   i.i.d. samples from $ \delta \PP  + (1-\delta) \Delta$ for small, but positive contamination $\delta>0$.

We view boosting as a method that iteratively builds  an additive model,
 \[
 F_T(x)=\sum_{t=1}^T \alpha_t h_{j_t}(x)
 \]
where $ h_{j_t}$ belongs to a large (but we will assume finite VC-dimension) space of weak hypotheses, denoted with $\mathcal{H}$. The use of our framework in the presence of countably infinite features, also known as the task of
feature induction, can easily be established. 
%
%
Next, we  introduce  the new framework of the boosting in the presence of the noise, which we call an {\it Arch Boosting} framework. 

We design the Arch Boosting framework as  a stage-wise, iterative, minimization of the $\phi$-risk \eqref{eq:3}. However, when   $\phi$ is  non-convex this cannot be done by simply applying the well known Friedman's Gradient Boosting  \citep{Friedman:99}. The explicit updates are usually unavailable and standard numerical methods like Newton-Raphson,  which are suitable only for convex functions, are used. Instead,  we constrain the stage-wise minimization of the $\phi$-risk   to keep one additional important property of the boosting algorithms, namely  the {\it hardness condition} of   \cite{FreundSchapire:97}.
It is shown that hardness condition can easily distinguish between the outliers and the center of the data, when non-convex loss is used. 
 Moreover, it allows for  a fine tuning of the appropriate  optimal hypothesis  assignments   such that the minimization of the $\phi$-risk is approximately kept. Therefore, hardness condition
  allows us to simultaneously  escape the non-convexity in the   minimization  and to use non-convexity to separate the outliers from the  inliers.
This property,  from iteration $t$ to $t+1$,
is defined as
\begin{eqnarray} \label{eq:2}
\E_{w_{t+1}} [Yh_t(X)|X=x] = 0,
\end{eqnarray}
where $\E_{w}$ is defined as a weighted conditional expectation  
\begin{eqnarray} \label{eq:expectation}
\E_w[g(X,Y)|X=x]=\frac{\E[w(X,Y) g(X,Y)|X=x]}{\E[w(X,Y)|X=x]}.
\end{eqnarray}
The equation \eqref{eq:2} can be explained as progressing from step $t$ to $t+1$; the weights are updated from $w_t$ to $w_{t+1}$, such that $h_t(X)$  is orthogonal to  $Y$  with respect to  the inner product   defined on the reweighed data $w_{t+1}$.  In a certain sense, the weights $w_{t+1}$ are chosen as the most difficult for the weak hypothesis $h_t$. For the weight vector $w_{t+1}$ the hypothesis $h_t$ is not better than a random guess.   

Recall that  our goal is to find the optimal $F^* \in \mathcal{F}$, which minimizes the $\phi$-risk $R_{\phi}(F)$ for a suitable class of measurable functions $\mathcal{F}$. In the binary case, each instance $X_i$ is associated with a label  $Y_i \in \{-1,1 \}$ and the goal of learning is then to find a classifier $F_T$, such that the sign of $F_T(X_i)$ is equal to $Y_i$. In order to minimize $R_{\phi}(F)=\E[\phi(YF(X))]$, we minimize it at every point $x \in \mathcal{X}$ -- that is, given any $x \in \mathcal{X}$, considering $F(x)$ as a parameter  and denoting $\Phi(F(x)) = \E[\phi(YF(X))|X = x] = \E_Y [\phi(YF(x))]$, the problem is to find
\begin{eqnarray} \label{eq:4}
F^*(x) = \argmin_{F\in \mathcal{F}} \Phi(F(x)).
\end{eqnarray}
Table \ref{tab:1} contains a list of commonly used loss functions $\phi$ and the corresponding optimal classifiers $F^*$ when $\mathcal{F} = \mathcal{M}$, which is the family of all measurable functions.

  \begin{table}
    \centering
    \footnotesize
    \caption{The list of commonly used loss functions and its corresponding $F^*$}
   \label{tab:1}
    \begin{tabular}{@{}nd{11.1}*{5}{d{31.2}}d{3.1}d{3.2}@{}}
      \toprule
        \multicolumn{1}{@{}N}{Classification Method} &
        \multicolumn{3}{N@{}}{Population parameters} &
        \\
      \cmidrule(lr){2-3}
        &
        \multicolumn{1}{V{4.5em}}{Loss function $\phi(v)$} &
        \multicolumn{1}{V{6.5em}}{Optimal Minimizer $F^*(x)$} 
         \\
      \cmidrule(lr){2-2}\cmidrule(lr){3-3}   
        Logistic    &  \log(1+e^{-v}) & {    \left(  \log  {\PP(y=1|x)} - \log {\PP(y=-1|x)} \right) }     \\
        Exponential  &  e^{-v} & \frac{1}{2} \left(  \log  {\PP(y=1|x)} - \log {\PP(y=-1|x)} \right)          \\
        Least Squares   &  (v-1)^2   &     \PP(y=1|x) - \PP(y=-1|x)     \\
        Modified Least Squares   &  [(1-v)_+]^2   &  \PP(y=1|x) - \PP(y=-1|x)   \\
      \bottomrule
    \end{tabular}
  \end{table}

Provided that $\mathcal{F} = \mathcal{M}$ and for any $x \in \mathcal{X}$, $\Phi(F(x))$ has only one critical point at $F^*(x)$ that is the global minimum, we can find $F^*(x)$ by the first order optimality condition
\begin{equation}\label{eq:5-a}
\frac{d}{dF(x)} \Phi|_{F=F^*} (F(x))= 0.
\end{equation}
Expanding on the above first order optimality conditions, we obtain
 $$\E[Y\phi^{'} (YF^*(X)) | X = x] = 0,$$
 where $\phi'$ is defined as the first order derivative $\frac{d}{dv} \phi(v)$. In classification problems, the input $v$ of loss function $\phi$ is $v=YF(X)$ -- that is, the margin of a classifier $F$ applied to  a data point $(X,Y)$.
Rewriting  the expectation in terms of the class probabilities, we obtain the following representation of the first order optimality conditions
\begin{eqnarray} \label{eq:5}
\frac{\phi^{'}(-F^*(x))}{\phi^{'}(F^*(x))} = \frac{\Prob(Y=1|X=x)}{\Prob(Y=-1|X=x)}.
\end{eqnarray}
From now on, without confusion, we   write $\PP(Y=1|x)$ instead of $\PP(Y=1|X=x)$. We aim to mimic  equation ebove in each of the iteration steps of the proposed framework. In more details,   at iteration $t$, with the  current estimate $F_{t-1}(x) = h_1(x) + \cdots + h_{t-1}(x)$ at hand,
 we wish to find a new weak hypothesis $h_{t} \in \mathcal{H}$, such that $F_t(x) = F_{t-1}(x) + h_{t}(x)$ with   $h_t(x)$ that solving the following equation
\begin{eqnarray} \label{eq:6}
\frac{\phi^{'}(-F_{t-1}(x)-h_{t}(x))}{\phi^{'}(F_{t-1}(x)+h_{t}(x))} = \frac{\Prob(Y=1|x)}{\Prob(Y=-1|x)}.
\end{eqnarray}
If we can always accurately find such a weak hypothesis (i.e. $\mathcal{H}$ is rich enough and we know the true distribution), then this process will terminate in just one step. However, complications arise from two aspects. First, we only have a restricted family $\mathcal{H}$ of weak hypotheses. Therefore, at each iteration, an approximated weak hypothesis, which is the closest approximation in $\mathcal{H}$, will be found.
Secondly, the equation \eqref{eq:6} alone  cannot be efficiently utilized  since $\mathbb{P}(Y=1|x)$, which is the ultimate goal of classification problems, is unknown to us. However, we propose to solve   approximated equation, where we replace $\mathbb{P}(Y=1|x)$  with   the weighted conditional probability $\mathbb{P}_{w_t}(Y=1|x)$ at each iteration $t$ of the proposed algorithm. Hence,   $\PP_{w_t}(Y=1|x) = \E_{w_t}[\1_{[Y=1]}|X=x]$ with the weighted conditional expectation  defined in \eqref{eq:expectation}.

In more details, at each iteration $t$, for a given  $\mathbb{P}_{w_t}(Y=1|x)$, we find $h_t(x)$ such that it solves the estimating equation 
\begin{eqnarray} \label{eq:6a}
\frac{\phi^{'}(-F_{t-1}(x)-h_{t}(x))}{\phi^{'}(F_{t-1}(x)+h_{t}(x))} = \mathbb{C}_{t-1}(x) \ \frac{\Prob_{w_t}(Y=1|x)}{\Prob_{w_t}(Y=-1|x)},
\end{eqnarray}
for a suitable function $ \mathbb{C}_{t-1}(x)$ to be specified later.
In order to do this, we will first discuss how to define the weights ${w_t}$  at each iteration  $t$ and postpone the method  of finding  $h_t$ for later.

Since $h_{t}$ is only an approximation to the optimal increment at step $t$, instead of adding itself to $F_{t-1}$ we multiply it by a constant $\alpha_{t}$ and search for the best constant $\alpha_t$. The best $\alpha_t$ should be such that the updating classifier  $\alpha_t h_t(x)$ 
approaches the optimal one, defined in equation \eqref{eq:5}.    Hence, we 
  define the optimal $\alpha_t$ as the solution to the following optimization problem  
\begin{eqnarray} \label{eq:7}
\alpha_{t} = \argmin_{\alpha \in \mathbb{R}} \E  \left[ \phi(Y(F_{t-1}(X)+\alpha h_t(X))) \right].
\end{eqnarray}
Similarly to the AdaBoost, at each iteration $t$, the data will be reweighed according to the weights $w_t$. We explore relation \eqref{eq:7} further to find the optimal weights $w_t$. 
For that purpose, we observe that for a given $F_{t-1}$ (from a previous iteration) and $h_t$ (solving \eqref{eq:6a}) the optimal $\alpha_t$ then satisfies 
\begin{eqnarray}\label{eq:7a}
\E \left[ -\phi^{'}(YF_{t-1}(X)+Y\alpha_{t} h_t(X)) \cdot \alpha_{t} Yh_{t}(X) \right] = 0.
\end{eqnarray}
Then, we recall the hardness condition of the AdaBoost algorithm:  
  the weights $w_t$ should be updated such that the weighted misclassification error of the most recent weak hypothesis is about $50\%$.  According to \eqref{eq:2},
  this can be achieved by defining  the weights $w_{t+1}(X,Y)$ such that 
  \begin{eqnarray} \label{eq:8}
\E \left[ w_{t+1}(X,Y) \cdot Y\alpha_{t}h_{t}(X) \right] = 0.
\end{eqnarray}
    Now, in the  Arch Boost framework the most recent weak hypothesis is $\alpha_{t}h_{t}$. 
 Hence, by contrasting the last two equations, \eqref{eq:7a} and \eqref{eq:8},  we define the weights to be
\begin{eqnarray*}
w_{t+1}(X,Y) = -\phi^{'}(YF_{t-1}(X)+Y\alpha_{t} h_t(X)) = -\phi^{'}(YF_{t}(X))
\end{eqnarray*}
such that both the    hardness condition  and the optimality of the updating hypothesis are satisfied.

Having defined the weight updating rule, we go back to equation \eqref{eq:6}  to define the optimal weak hypothesis $h_t(x)$.  We do so by finding a relationship between ${\Prob_{w_t} (Y=-1|x)}$ and ${\Prob  (Y=-1|x)}$ and    defining the function $ \mathbb{C}_{t-1}$ of \eqref{eq:6a}. We multiply \eqref{eq:6} with $\frac{\phi^{'}(F_{t-1}(x))}{\phi^{'}(-F_{t-1}(x))}$ on both sides to obtain
\begin{align} \label{eq:9}
\frac{\phi^{'}(F_{t-1}(x))\phi^{'}(-F_{t-1}(x)-h_t(x))}{\phi^{'}(-F_{t-1}(x))\phi^{'}(F_{t-1}(x)+h_t(x))} = \frac{\phi^{'}(F_{t-1}(x))\Prob(Y=1|x)}{\phi^{'}(-F_{t-1}(x))\Prob(Y=-1|x)}.
\end{align}
Then, it is easy to observe that 
\begin{eqnarray*}
\phi^{'}(F_{t-1}(x))\Prob(Y=1|x) &=& \E[\1_{[Y=1]}\phi^{'}(YF_{t-1}(X))|X=x] \\
&=& -\E[\1_{[Y=1]}w_t(X,Y) |X=x] = -\E_{w_t}[\1_{[Y=1]}|X=x] \\
&=& -\PP_{w_t}(Y=1|x).
\end{eqnarray*}
Hence, we have 
\begin{align} \label{eq:9a}
\frac{\phi^{'}(F_{t-1}(x))\phi^{'}(-F_{t-1}(x)-h_t(x))}{\phi^{'}(-F_{t-1}(x))\phi^{'}(F_{t-1}(x)+h_t(x))}
=
\frac{\Prob_{w_t} (Y=1|x)}{\Prob_{w_t} (Y=-1|x)},
\end{align}
that is, equation \eqref{eq:6a} is true for 
$$\mathbb{C}_{t-1}(x) = \phi^{'}(-F_{t-1}(x))/\phi^{'}(F_{t-1}(x)).$$
Observe that the right hand side of \eqref{eq:9a} can be estimated by a weak learner at each iteration $t$. For reference, we list the weak hypothesis $h(x)$ for several commonly used loss functions in the Table \ref{tab:2}. 
  \begin{table}[h]
    \centering
    \footnotesize
    \caption{The list of commonly used loss functions and their weak hypotheses $h$}
   \label{tab:2}
    \begin{tabular}{@{}nd{11.1}*{4}{d{38.2}}d{3.1}d{3.2}@{}}
      \toprule
        \multicolumn{1}{@{}N}{Classification Method} &
        \multicolumn{3}{N@{}}{Population parameters} &
        \\
      \cmidrule(lr){2-3}
        &
        \multicolumn{1}{V{4.5em}}{Loss function $\phi(v)$} &
        \multicolumn{1}{V{7.5em}}{Optimal weak hypotheses $h(x)$} 
         \\
      \cmidrule(lr){2-2}\cmidrule(lr){3-3}   
        Logistic    &  \log(1+e^{-v}) & {     \log  {\PP_w(Y=1|x)} - \log {\PP_w(Y=-1|x)} }     \\
        Exponential  &  e^{-v} &  \frac{1}{2} \left(  \log  {\PP_w(Y=1|x)} - \log {\PP_w(Y=-1|x)} \right)         \\
        Least Squares   &  (v-1)^2   &   C(1-F(x))(1+F(x))/(CF(x) + 1)  \\
        Modified Least Squares   &  [(1-v)_+]^2   &  C(1-F(x))(1+F(x))/(CF(x) + 1)  \\
      \bottomrule
      * $C = \Prob_{w}(Y=1|x)-\Prob_{w}(Y=-1|x)$
    \end{tabular}
  \end{table}

Note that the weak hypotheses of the least squares loss   and modified least squares loss  depend  on the current estimate $F(x)$ and the weighted conditional probability  $\mathbb{P}_{w_t}(Y=1|x)$, which is different from that of Gradient boosting \citep{Friedman:99}.
Lastly, we summarize the above procedure in the following Algorithm \ref{boosting}, which we call {\it Arch Boost}. The assumption that $\Phi(F(x))$ has only one critical point does not require the loss function $\phi$ to be convex and hence the Arch Boost algorithm can be applied to many non-convex loss functions. We illustrate this in  Section \ref{sec:3}  by proposing a family of  robust boosting algorithms based on the Arch Boosting framework.

In the step (b) of the Algorithm \ref{boosting},  any classifier can be used; one example is a decision tree for which case $\PP_{w_t}(Y=1|x)$ is the proportion of  the training samples with label $1$ in the terminal node, where $x$ ends up. For instance, in terminal region $R$, let $\mathcal{I}_+$ stands for the index set of data points with positive label, then $\PP_{w_t}(Y=1|x) = \sum_{i \in \mathcal{I}_+}w_t(x_i,y_i) / \sum_{j 
\in R} w_t(x_j,y_j)$.

Note that in the step $(c)$, whenever $\Prob_{w_t}(Y=1|x)\Prob_{w_t}(Y=-1|x)=0$, then $h_t(x)$ can be $\infty$ or $-\infty$. In order to deal with this problem, one method is to only update the points with $-\infty < F_t(x) < \infty$. If $F_t(x)$ becomes infinity after some step $t$, then we keep it to be infinity and do not update further. Finally, we just set $sign(\infty) = 1$ and $sign(-\infty)=-1$. Another method is simply applying a map $v \mapsto a v + \frac{1-a}{2}$ to the class probability estimations where $a<1$ is a constant very close to 1 (e.g. $a=0.9999$).

\begin{algorithm} [h]
\caption{Arch Boost ($\phi$)} 
\label{boosting}
\begin{algorithmic}[3]
\Procedure{Arch Boost($\phi$)}{}
    \State Given: $(x_1,y_1), \dots, (x_n, y_n)$
    \State Initialize the vector of weights $w_0$, e.g. $w_0(x_i,y_i) = 1/n$
    
    \For {$t = 1, \dots, T$}
        \State (a) Normalize the weight by assigning  $ w_t \gets \ {w_{t}}/{\sum_i w_{t}(x_i,y_i)}$  
        \State (b) Fit the classifier to obtain a class probability estimate $\Prob_{w_t}(Y=1|x) \in [0,1]$
        \State using current weights $w_t$ on the training data.
        \State (c) Set $h_t(x)$ to be the solution of \eqref{eq:9a}.
        \State (d) Find $\alpha_t$ by solving \eqref{eq:7}. 
        \State (e) Set $F_{t}(x) = F_{t-1}(x)+\alpha_{t} h_{t}(x)$
	\State (f) Update the weights  $w_{t} \gets -\phi^{'}(yF_{t}(x))$
   \EndFor
   
   \State Output the classifier:
   \begin{equation*}
	sign \left( F_T(x) \right)
   \end{equation*}
\EndProcedure
\end{algorithmic}
\end{algorithm}

\section{Robust boosting algorithms} \label{sec:3}

In this section, we propose a new class of boosting algorithms especially designed to be resilient to the presence of extraneous noise in the data.  Whether this noise comes in the forms of mislabeled data points,   additional variance within class data, or malicious data, the proposed algorithm adapts to it with great success. 
We begin the section by  proposing a new loss function, which we have  named two-robust loss function. Furthermore, we characterize regularity conditions that a loss function needs to satisfy, in order to be appropriate for Arch Boosting framework.  Finally, we provide a family of loss functions that is not convex and that satisfies the newly defined regularity conditions and the corresponding family of robust boosting algorithms called ARB-$\gamma$.


\subsection{A non-convex loss}

An approach of non-convex functions has been recognized as successful in the existing literature. Both \cite{Freund:01} and \cite{Freund:09}   have utilized it to propose boosting methods that are resilient to the presence of outliers in the data. However,  both loss functions require a number of tuning parameters.  The behavior of the algorithm is hindered by the  optimal choice of such parameters. Moreover, such optimal choices are data dependent and not isolated as universal by either approach. 

Driven by the need to propose an alternative loss function that can be easily used in many situations, we propose the following   non-convex loss function:
\begin{eqnarray} \label{eq:11}
\phi(v) =  {4}{(1+e^{v})^{-2}}.
\end{eqnarray}
The proposed loss function \eqref{eq:11} has optimal $F^*$, the weak hypothesis, and the weight updating rule as presented in Table \ref{tab:3}. We call the function \eqref{eq:11} {\it two-robust loss} and it  belongs to a family of loss functions that will be discussed in Section \ref{sec:3c}. By adopting the Arch Boosting framework of Algorithm \ref{boosting} we illustrate the nice  downweighting property of the  new weight updating rule with this new loss function. 
  \begin{table} [h]
    \centering
    \footnotesize
    \caption{Properties of the two-robust loss}
   \label{tab:3}
    \begin{tabular}{@{}nd{3}*{3}{d{11.2}}d{25.1}d{13.2}@{}}
      \toprule
        \multicolumn{1}{@{}N}{ Arch 2-boosting  } &
        \multicolumn{1}{N@{}}{Optimal parameters } &
        \multicolumn{1}{N@{}}{ } 
        \\
      \cmidrule(lr){2-5}
        &
        \multicolumn{1}{V{4.5em}}{Loss function $\phi(v)$} &
        \multicolumn{1}{V{6.5em}}{Optimal minimizer $F^*(x)$} 
        &
          \multicolumn{1}{V{5.5em}}{Weak hypothesis $h(x)$} 
          &
            \multicolumn{1}{V{4.5em}}{Weight vector  $w(x,y)$} 
         \\
      \cmidrule(lr){2-2}\cmidrule(lr){3-3}   \cmidrule(lr){4-4} \cmidrule(lr){5-5}    
        Arch     &  {4}{(1+e^{v})^{-2}} & {      \log  \frac{\PP(Y=1|x)}{\PP(Y=-1|x)}  }  & {\log  \frac{\PP_{w_t}(Y=1|x)}{\PP_{w_t}(Y=-1|x)}} & {   \frac{  e^{yF_t(x)}  }{(1+e^{yF_t(x)})^{3}} }    \\
      \bottomrule
    \end{tabular}
  \end{table}
In Figure \ref{fig:1}, we plot the two-robust loss \eqref{eq:11}, together with exponential and logistic regression losses and the corresponding weight updating rules.

We observe that the two-robust loss is non-convex and bounded from above when $v \to -\infty$. Therefore, the outliers, even if large in size, will only have bounded influence on the classification. Moreover, by investigating the weight updating rule, we  observe
that the more  misclassified the data point is, the smaller  the weight updating function will be. The algorithm,   in fact, abandons the data points that are repeatedly misclassified and are far  from the Bayes boundary. This phenomenon disappears when one uses exponential loss or logistic loss and, in fact, any commonly used convex loss function. To the best of our knowledge, no existing loss function satisfies this nice property without requiring a-priori fixed tuning parameters. 
By plugging the two-robust loss into Arch Boost, we obtain new boosting algorithm that we named  { \it Adaptive Robust Boost-2}, denoted with  ARB-2 from hereon.

 \begin{figure}[h!]
\centering
\includegraphics[scale=0.4]{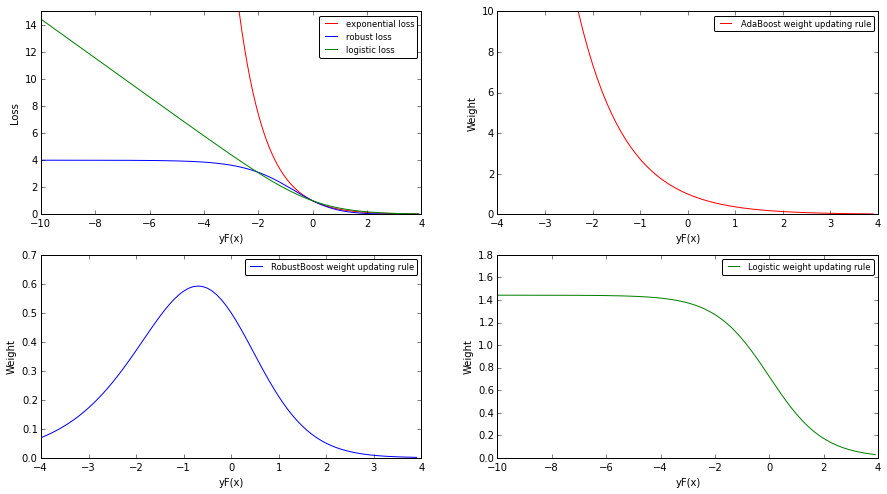}
\caption{Comparison of the two-robust loss, the exponential loss and the  logistic loss functions (left-top) together with the corresponding weight updating functions.}
\label{fig:1}
\end{figure}

Next, we  consider a simple two dimensional binary classification example consisting of two mislabeled points near the classification boundary.  The example is artificially created to illustrate the point of an adaptive classifier,  when there are outliers in the data.   We evenly put 121 points in the region $[0,3] \times [0,3]$, that is, $\mathcal{X} = \{{\mathbf x}_{i,j}: i = 0, \cdots, 10, \; j = 0, \cdots, 10 \} \subset \mathbb{R}^2$, where ${\mathbf x}_{i,j} = (0.3i, 0.3j)$. For each $ {\mathbf x}_{i,j}$, the corresponding $y_{i,j} = 1$ if $i+j \le 60$ and $y_{i,j} = -1$ otherwise. Then we flip the labels of $( {\mathbf x}_{2,4}, y_{2,4})$ and $( {\mathbf x}_{3,4}, y_{3,4})$. The step sizes, $\alpha_t$, for both algorithms are set to be constant 0.5, and the number of iterations is set as $400$. We show the difference between the Real AdaBoost and the ARB-2 in Figure \ref{fig:example}. It can be seen that the decision boundary of AdaBoost are influenced by the two blue ``abnormal'' data points. But for ARB-2, the decision boundary stays  the same as if there were no such abnormal points, hence it adapts to the outliers.

\begin{figure}[h]
\centering
\includegraphics[scale=0.35]{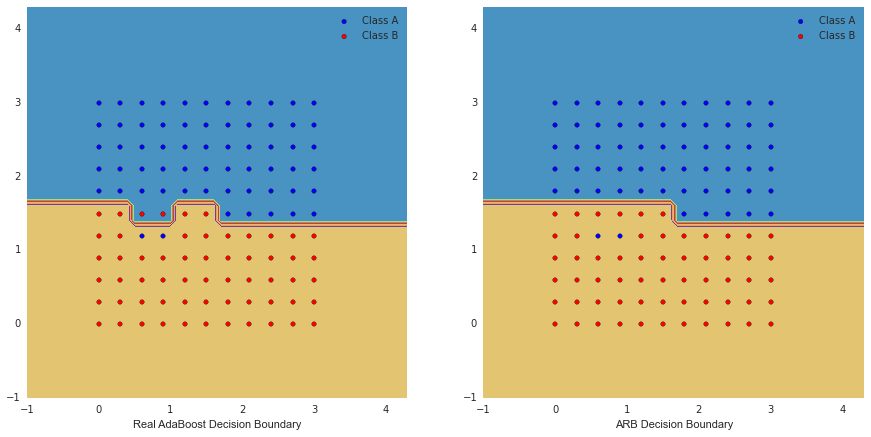}
\caption{ The classification boundary of the Real AdaBoost (left) and of the ARB-2 (right).} \label{fig:example}
\end{figure}

\subsection{Loss functions for Arch Boost}
In the previous section, we have seen an example of a robust boosting algorithm.
Nevertheless, there are plenty of other non-convex functions which leads to the question: can we use any of them as a loss function for the Arch Boost? The answer is no, if we do not impose   auxiliary  conditions. In this section, we discuss what kinds of  conditions need to be imposed so that a non-convex function becomes a suitable loss function for the Arch Boost framework \ref{boosting}.

Recall that in a binary classification problem, we want to minimize $\Phi(F(x))$,   to find
\begin{eqnarray} \label{eq:13}
  \min_{v \in \mathbb{R}} \Bigl[ \Prob(Y=1|x) \phi(v) + \Prob(Y=-1|x) \phi(-v) \Bigl].
\end{eqnarray}

Since $\Prob(Y=1|x) + \Prob(Y=-1|x) = 1$, the equation above becomes  a convex combination of $\phi(v)$ and $\phi(-v)$. Therefore, one necessary condition on loss functions is that \eqref{eq:13} has  a unique optimal solution in $\mathbb{R}$. Observe that if $\mathcal{F}$ is a class of all measurable functions, then $F(x)$ can take any   real value  for every $x \in \mathcal{X}$. This condition is not equivalent to the convexity of the loss function but rather to the local convexity around the true parameter  of interest, $F^*(x)$.  In the next section we  present a family of non-convex loss functions that possess this property.
Hence, we present the set of regularity conditions in the Definition below.

\begin{definition} \label{def:1}
A function $\phi$ is an \textbf{Arch boosting loss function} if it is differentiable and
\begin{enumerate} [(i)]

\item $\phi(v) \ge 0$ for all $v \in \realR$ and $\inf_{v\in \mathbb{R}} \phi(v) = 0$;

\item for any $0 < \alpha < 1$, $\alpha \phi(v) + (1-\alpha)\phi(-v)$ has only one critical point $v^*$ which is the global minimum;  

\item for any $0 \le \alpha \le 1$ and $\alpha \neq \frac{1}{2}$, 
$\inf \{ \alpha \phi(v) + (1-\alpha)\phi(-v): v(2\alpha-1) \le 0 \} > \inf \{ \alpha \phi(v) + (1-\alpha)\phi(-v): v \in \mathbb{R} \}$.

\end{enumerate}
\end{definition}

Conditions $(i)$ and $(iii)$ together   imply  that  $\phi$ is an upper bound of the $0$-$1$ loss up to a constant scaling. 
Condition $(iii)$ is  a ``classification calibration''  \citep{Bartlett:06} and is considered the weakest possible condition imposed on $\phi$ for which  a measurable function $F(x)$ which minimizes the $\phi-$risk $R_{\phi}(F)$  will also have the   risk close to the minimal one; in other words, close to that of  the optimal $F^*(c)$  that  creates a Bayes boundary.  
If the function $\phi$ is convex, then condition $(iii)$ is satisfied as long as $\phi$ is differentiable and $\phi^{'}(0) < 0$.  However, when considering non-convex losses, the set of regularity conditions doesn't exists in the current literature. The above  framework    includes non-convex loss functions $\phi$ and  differs from the existing literature on convex losses  in that it includes an additional condition $(ii)$   (see Section \ref{theory}).

\begin{lemma} \label{lemma:2}
A  positive function $\phi$ that is continuously differentiable, convex and such that $\phi^{'}$ is not  equal to a constant satisfies Condition (ii).
\end{lemma}
By Lemma \ref{lemma:2}, we know that the logistic, the exponential, the least square and the modified least square loss   are all Arch boosting loss functions.  Differentiability of the loss function is a technical condition and is not crucial for the proposed framework.
The hinge loss $\phi_{hinge}(v) =(1-v)_+$  is not differentiable but can be shown to satisfy Conditions $(i)$-$(iii)$. However, if we plug in the hinge loss to the equation \eqref{eq:5}, we cannot easily find the optimal solution $F^*(x)$.
However, not every differentiable and non-convex loss function satisfies regularity conditions above.


\begin{remark}
The sigmoid loss $\phi_{sig}(v) =(1+e^v)^{-1}$ is differentiable and satisfies Condition (iii), but it does not satisfy Condition (ii) and hence is not an Arch Boosting loss function.
\end{remark}

Finally, we provide a more general characterization for the  loss functions, including not necessarily convex loss functions, which satisfy Condition $(ii)$.
\begin{lemma} \label{lem:iii}
Let $\phi$ be a positive, continuously differentiable loss function such that 
  $\phi^{'}(v) \neq 0$ for all $v \in \mathbb{R}$.
  Let $g:(0,\infty) \to (0,1)$, be defined as  
  $g(v) := \frac{\phi^{'}(v)}{\phi^{'}(-v)}$.
   Then $\phi$ satisfies  Condition (ii) as long as
  function $g$
  is a decreasing, bijection.
\end{lemma}

\subsection{A family of non-convex functions and ARB-$\gamma$ algorithms} \label{sec:3c}

Recall that the optimal   classifier $F^*(x)$   satisfies equation \eqref{eq:5}. We observe that the right hand side of this equation does not depend on the loss function $\phi$ and can take values in the positive   real line, $\RR_+$. Hence, we can parameterize it with any real-valued function whose range is  $\RR_+$, as follows
\begin{eqnarray} \label{eq:14-a}
\frac{\phi^{'}(-v)}{\phi^{'}(v)} = g(v)\end{eqnarray}
for any surjective, decreasing  function $g : \RR \to \RR_+$.
The classical motivation for reparametrization \citep{MN89} -- often called link functions -- is that often  one uses a parametric representation that has a natural scale   matching the desired one.
We  choose to use the function $g (v) = e^{(\gamma-1) v}$ for any constant $\gamma > 1$, which is  a surjection. This parametrization is not unique but it  admits a solution to the following 
   differential equation
\begin{eqnarray} \label{eq:14}
\frac{\phi^{'}(-v)}{\phi^{'}(v)} = e^{(\gamma-1)v}.
\end{eqnarray}
Solving it for $\phi$ (exact steps  are presented in the  Appendix A), we  obtain a family of non-convex functions  
\begin{eqnarray} \label{eq:15}
\phi_{a,\gamma}(v) = \frac{2^\gamma}{(1+e^{av})^\gamma}, \; a>0.
\end{eqnarray}
\begin{figure}[h]
\centering
\includegraphics[scale=0.30]{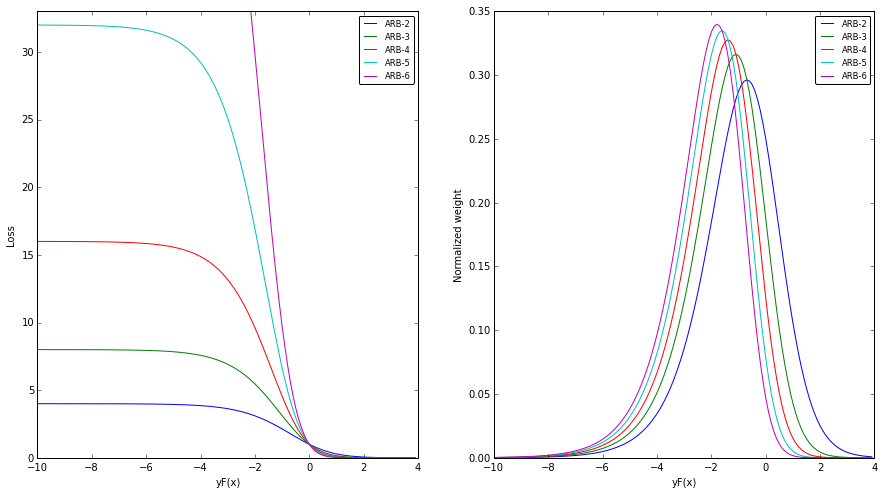}
\caption{Comparison of $\phi_{1,\gamma}$ with different $\gamma$ values and the corresponding normalized weight updating rules.}
\label{fig:2}
\end{figure}
Observe that parameters $a$ and $\gamma$ are not tuning parameters, but rather  an  index set of a family of non-convex losses much like  Huber and Tukey's biweight losses are.
 Note that when $\gamma > 1$, the right hand side of \eqref{eq:14} can take all the real values and is monotonically  increasing  and consequently the corresponding  loss $\phi$ will always give a unique solution to \eqref{eq:5}. We   name each element in the family \eqref{eq:15} as \textit{$\gamma$-robust loss function}. Later, we will show that the positive parameter $a$ is irrelevant in our algorithms, so we can fix it to be $1$. Note that for  $a=1$ and $\gamma=2$,  we obtain the non-convex  loss function \eqref{eq:11}. 
 Moreover,  each loss function $\phi_{a,\gamma}$ is a bounded function with upper bound  equal to $2^\gamma$. Therefore, the effects  of the outlier will be necessarily bounded. Moreover, the weight updating rule will also downweight the largely misclassified data points. 
We plot $\phi_{1,\gamma}$ with different $\gamma$ values and the corresponding normalized weight updating rules in the  Figure \ref{fig:2}. From Figure \ref{fig:2}, we can see that for different $\gamma$, the peak point of each of the weight updating rules shifts to the left when $\gamma$ increases. When $\gamma=1$, the weight updating curve is symmetric about $yF(x)=0$ and in fact, this is equivalent to the Sigmoid loss function $\phi(v) = 1-\tanh(\lambda v)$ when $\lambda = \frac{1}{2}$ \citep{Mason:99}. Moreover, for $a=2$ and $\gamma=2$ the loss $\phi_{a,\gamma}$ is equivalent to the Savage loss function $\phi(v) = (1 + e^{2v})^{-2}$   of  \cite{V09}, in which they used the probability elicitation technique from \cite{Savage} to design new loss functions.

\begin{lemma} \label{lemma:3.3}
For all $a>0, \gamma>1$, $\phi_{a,\gamma}$ is an  Arch Boosting loss function.
\end{lemma}
  Lemma \ref{lemma:3.3} allows us to plug $\phi_{a,\gamma}$ with $a>0, \; \gamma > 1$ into the Arch Boost framework and  obtain a family of robust boosting algorithms which we name  { \it Adaptive Robust Boost-$\gamma$} (\textit{ARB-}$\gamma$). The details are  presented  in the   Algorithm \ref{alg:3}  with computations relegated to the Appendix A.  
\begin{algorithm}[h!]
\caption{ARB-$\gamma$ ($\gamma \ge 1$)}
\label{alg:3}
\begin{algorithmic}[3]
\Procedure{ARB-$\gamma$}{}
    \State Given: $(x_1,y_1), \dots, (x_n, y_n)$
    \State Initialize the weight vector $w_0$, e.g. $w_0(x_i,y_i) = 1/n$
    \For {$t = 1, \dots, T$}
        \State (a) Normalize the weight vector  $w_t \gets {w_{t}}/{\sum_i w_{t}(x_i,y_i)}$
        \State (b)  Compute the classifier to obtain a class probability estimate $\Prob_{w_t}(Y=1|x) \in [0,1]$,
        \State using weights $w_t$ on the training data.
        \State (c) Set $h_t(x) \gets \log \frac{\Prob_{w_t}(Y=1|x)}{\Prob_{w_t}(Y=-1|x)} \in \bar{\mathbb{R}}$.  
         \State (d) Find $\alpha_t$ by solving \eqref{eq:7}. 
        \State (e) Set $F_{t}(x) \gets F_{t-1}(x)+\alpha_{t} h_{t}(x)$
	\State (f) Set $w_{t+1} \gets  {e^{yF_t(x)}}{(1+e^{yF_t(x)})^{-\gamma-1}}$
   \EndFor
   
   \State Output the classifier:
   \begin{equation} \label{eq:16}
	sign \left( F_T(x) \right)
   \end{equation}

\EndProcedure
\end{algorithmic}
\end{algorithm}
Note that the form of the weak hypothesis in Algorithm \ref{alg:3} is $h_t(x) = \log \frac{\PP_{w_t}(Y=1|x)}{\PP_{w_t}(Y=-1|x)}$ instead of the solution $\frac{1}{\gamma-1} \log \frac{\PP_{w_t}(Y=1|x)}{\PP_{w_t}(Y=-1|x)}$ to equation \eqref{eq:6a}, because we absorb the constant $\frac{1}{\gamma-1}$ into the step size $\alpha_t$ at each iteration $t$.

%

\section{Theoretical Considerations} \label{theory}

Despite the substantial
body of  existing work on  Gradient Boosting classifiers and AdaBoost in particular (e.g. \cite{Bartlett:06,Bartlett:07,Freund:95, FHT:00,Schapire:99b,SFB+:98,ZY05, Breiman:04,K03}),  research on  robust boosting classifiers  has mostly been limited to  methodological proposals with little  supporting theory (e.g., \cite{GentileLittlestone:99,KearnsLi:93,Littlestone:91}). 
  However, whereas the loss functions studied in those papers
are finite-sample versions of globally convex functions, many important robust classifiers, such as
those arising in \cite{Freund:09} and the proposed ARB-$\gamma$, only possess convex curvature over local regions 
even at the population level. 
In this paper, we present new theoretical results, based only
on local curvatures, which may be used to establish statistical and optimization
properties of the proposed  Arch boosting algorithms, with highly non-convex loss functions.

\subsection{Numerical convergence of Arch Boost algorithm } \label{sec:4}
We show that the empirical risk will always decrease and   that the Arch Boost can be viewed as the  step-wise iterative minimization of the empirical risk $\hat{R}_{\phi,n}$. 
Similar results can be found in 
 \cite{K03} and \cite{ZY05}. The main difference is that the authors  use the gradient descent rule
in the first or an  approximate minimization in the second paper, while here we use the hardness condition to select the optimal weak hypothesis $h$. 
We use  $\{w(X_i,Y_i)\}_{i=1}^n$ to denote the weights on the data such that $\sum_{i=1}^n w(X_i,Y_i) = 1$.  Recall that for any classifier $h$ and any data point $(X,Y)$, the term $Yh(X)$ always stands for the margin. For any weak hypothesis $h_t$, we denote the  {\it expected margin} as
$
\mu(h, w) = \E_w \left[Yh(X) \right],
$
and the { \it empirical margin} as
$
\hat \mu(h, w) = \sum_{i=1}^n w(X_i,Y_i) Y_i h(X_i).
$
To introduce the notation used in the theorem, for any family of weak classifiers $\mathcal{H}$, we denote
\begin{equation}\label{eq21}
\mathcal{F}^T = \left \{ F: F = \sum_{t=1}^T \alpha_t h_t, \alpha_t \in \realR, h_t \in \mathcal{H} \right \},
\end{equation} 
a set of $T$-combinations ($T \in \mathbb{N}$) of functions in $\mathcal{H}$. Let $\{\bar{f}_t\}$ be a sequence of functions with empirical risk converging to $R^*_{\phi,n}$, defined as $R_{\phi,n}^* = \inf_{F \in \cup_{T=1}^{\infty}\mathcal{F}^{T}} \hat R_{\phi,n}(F)$. Then, $\bar f_t \in \mathcal{H}_f \subset \cup_{T=1}^\infty \mathcal{F}^T  $ can be represented as $\sum_{h \in H_f}  \alpha^h h$. In this respect, we define its $l_1$ norm as $ {|H_f|^{-1}}  \sum_{h \in H_f} | \alpha^{(h)}|  $.

\begin{thm} \label{numerical convergence}
Assume $\phi$ is an Arch boosting loss function. Furthermore, assume that the weak learner is able to provide disjoint regions on the domain $\mathcal{X}$ at each iteration $t$ (e.g. decision tree). We apply the Arch Boost ($\phi$) algorithm to a sample $\mathcal{S}_n = \{ (X_1,Y_1), \cdots, (X_n,Y_n) \}$ for $T$ iterations.
\begin{enumerate} [(i)]
\item
If at each iteration $t$ the weak hypotheses $h_t \in \mathcal{H}$ satisfies $\hat \mu(h_t, w_t) > 0$, then $\hat{R}_{\phi,n}(F_T)$ will converge in $\mathbb{R}$ as $T \to \infty$.

\item
At each iteration $t$, let $\{R_t^j\}_{j=1}^J$ be the set of disjoint regions on $\mathcal{X}$ returned by the weak hypothesis $h_t$, and $p_t^j:=\mathbb{P}_{w_t}(Y=1|x\in R_t^j)$ be the class probability estimation in that region. If there exists a strictly increasing function $\theta: [0,1] \to \RR$ with $\theta(\frac{1}{2}) = 0$, and a positive constant $K>0$ such that $h_t$ satisfies the representation 
$$h_t(x) = K \theta (p_t^j)$$
for all $x\in R_t^j$, then $\hat \mu(h_t, w_t) > 0$.
\item
Using any Arch boosting loss function $\phi$ for the Arch Boost Algorithm \ref{boosting}, at any iteration $t$, there exists such a function $\theta$ that satisfies (ii).
\item
Let $\{\bar{f}_t\}$ be a sequence of functions with empirical risk converging to $R^*_{\phi,n}$ and such that  
$$||\bar{f}_t - F_t||_1 = o(\log t), \qquad ||\bar{f}_t - F_t||_2^2 \le \frac{||\bar{f}_t - F_t||_1^2}{t^{c_t}}$$ where $c_t \in (0,1)$ and $c_t \to 0$ as $t \to \infty$.
Furthermore, assume    $\phi$ is  Lipschitz differentiable, $\theta$ is  bounded and $\hat{\mu}(h_t,w_t) \to 0$ as $t \to \infty$.
Then, if a sequence of step sizes $\alpha_t$ is such that 
$$\sum_{t=1}^{\infty} \alpha_t = \infty, \ \sum_{t=1}^{\infty} \alpha_t^2 < \infty,
\sum_{t=1}^{\infty} \frac{\alpha_{t+1}\xi_t \log t }{ t^{c_t}} < \infty,$$ 
for a sequence of positive numbers  $\xi_t =o(1)$,
we have $\hat{R}_{\phi,n}(F_T) \to R^*_{\phi,n}$ as $T \to \infty$.
\end{enumerate}
\end{thm}
 Theorem \ref{numerical convergence} suggests that for any ARB-$\gamma$ algorithm, the weak hypothesis $h_t$ at each iteration $t$  has a corresponding function $\theta$ that satisfies $(ii)$. Therefore,   $h_t$ satisfies  $\hat \mu(h_t, w_t) > 0$ and with it, by $(i)$,  the ARB-$\gamma$ algorithm   converges when the number of iterations increases. The conditions in part (iv) are somewhat different compared to the equivalent one obtained for the gradient boosting with convex losses \citep{ZY05}. The reference sequence $\bar f_t$ needs to be in a local neighboorhood of $F_T$ in the sence that $\bar{f}_t - F_t$ cannot blow up too rapidly.  Moreover, the   size  of $\mathcal{H} \supset \mathcal{H}_t \ni \bar f_t$ cannot be smaller than $t^{c_t}$ implying that the   size of $\mathcal{H}$ needs to converge to $\infty$ faster than a polynomial of $T$.
 Additionally, the choice of $\alpha_t$ depends on $c_t$ and $\xi_t$. The classical    conditions that are guarding against infinitely small step sizes are now supplemented with an additional constraint $\sum_{t=1}^{\infty} \frac{\alpha_{t+1}\xi_t \log t }{ t^{c_t}} < \infty$.  For example, if $\xi_t = O(\frac{1}{t})$, then we can choose $\alpha_t = O(\frac{1}{t^c})$ where $c \in (\frac{1}{2},1)$ and $c_t$ can converge to 0 at any speed. However, if $\xi_t = O(\frac{1}{\log t})$, then we need $c_t \to 0$ slowly (e.g. $O(\frac{1}{\log \log t})$) and $\alpha_t$ can be chosen as $O(\frac{1}{t})$. The additional constraint in the step size choice   acts as a penalty on  allowing non-convex loss functions. However, unlike   existing results Theorem \ref{numerical convergence} does not require any additional algorithmic tuning parameters (see Theorem 3.1 of  \cite {ZY05} and choices of $\varepsilon_t$, $\Lambda_t$). Results in \cite{Bartlett:07} (e.g., Theorem 6) provide similar bounds under   an
assumption of an unbounded step size of the boosting algorithm  and assume a positive lower bound on  the hessian of the empirical risk (a condition strictly violated for non-convex losses).
 
 We provide   examples of the function $\theta$ in the Table \ref{tab:4}. 
 Note that this distribution $\PP$ may potentially be a contaminated   distribution -- the proof is not affected by the contamination.
  \begin{table}[h!]
    \centering
    \footnotesize
    \caption{$\theta$ functions for several loss functions}
   \label{tab:4}
    \begin{tabular}{@{}nd{11.1}*{5}{d{31.2}}d{3.1}d{3.2}@{}}
      \toprule
        \multicolumn{1}{@{}N}{Classification Method} &
        \multicolumn{3}{N@{}}{Population parameters} &
        \\
      \cmidrule(lr){2-3}
        &
        \multicolumn{1}{V{4.5em}}{Loss function $\phi(v)$} &
        \multicolumn{1}{V{6.5em}}{$\theta$ functions} 
         \\
      \cmidrule(lr){2-2}\cmidrule(lr){3-3}   
        Logistic    &  \log(1+e^{-v}) & \log(t) - \log(1-t)     \\
        Exponential  &  e^{-v} &  \log(t) - \log(1-t)        \\
        Least Squares   &  (v-1)^2   &   \frac{(2t-1)(1-F^2)}{(2t-1)F+1}, \; F\in [-1,1]  \\
        Modified Least Squares   &  [(1-v)_+]^2   & \frac{(2t-1)(1-F^2)}{(2t-1)F+1}, \; F\in [-1,1]  \\
        $\gamma$-robust ($\gamma > 1$) & 2^{\gamma}/(1+e^{v})^{\gamma} &   \log(t) - \log(1-t)  \\
      \bottomrule
    \end{tabular}
  \end{table}
Furthermore, we remark that the result of Theorem \ref{numerical convergence} (i) to (iii)   requires very weak conditions. Namely, the approximate minimization step \eqref{eq:7} can be inexact (by contrast, see Theorem 6 of \cite{Bartlett:07}). The weak hypothesis $h_t$ at each iteration is obtained by preserving the ``hardness'' property of the AdaBoost, and this method is applicable to non-convex functions. For a certain loss function $\phi$, $h_t$ may not apriori point to the gradient descend  direction of the empirical $\phi-$risk. Hence, we cannot use numerical methods like Newton's method as needed in the gradient descent, but provide a novel way to find a weak hypothesis that is also suitable for non-convex loss functions. Then, we show that   the direction of such weak hypothesis will indeed be a descending direction for the non-convex loss, that is, we  guarantee that $\hat{R}_{\phi,n}(F+\alpha h) \le \hat{R}_{\phi,n}(F)$ with appropriate step size $\alpha$, where $F$ is the current estimate.

\subsection{Robustness}\label{sec:robust}
In Section \ref{Arch boost}, we gave an informal explanation why non-convex losses   lead to a more robust algorithm. The weight updating rule or the first derivative of the loss functions plays an important role in their robustness. Unlike convex functions, many non-convex functions can have a diminishing first derivative $\phi^{'}(v)$ when  $v$ tends to both infinity and negative infinity. In this section, we   quantify the robustness and justify the robustness of Arch Boosting Algorithms \ref{alg:3} through the point of view of the finite sample Breakdown point, as well as that of the  Influence Functions, the population measure of robustness.

\subsubsection{An invex function view of robustness}
In this section, we will use invex function properties to show why our non-convex functions leads to more robust algorithms. We first recall some definitions \citep{Ben:86}.

\begin{definition}[Invex function]
Assume $X \subset \RR^n$ is an open set. The differentiable function $f:X \to \RR$ is invex if there exists a vector function $\eta: X \times X \to \RR^n$ such that
\begin{eqnarray}
f(x) - f(y) \ge \eta(x,y)^T \nabla f(y), \;\; \forall x,y \in X.
\end{eqnarray}
\end{definition}

It is well known that if $g:\RR^n \to \RR$ is differentiable and convex and $A:\RR^r \to \RR^n \;(r \ge n)$ is differentiable with $\nabla A$ of rank $n$, then $f = g \circ A$ is invex \citep{Mishra:08}. \cite{Craven:85} proved that a function $f$ is invex if and only if every stationary point is a global minimizer. So the second condition (ii) in Definition \ref{def:1} is equivalent to say that $\alpha \phi(v) + (1-\alpha)\phi(-v)$ is an invex function with exactly one critical point for all $\alpha \in (0,1)$. We can then show that with $\phi$ being  our two-robust loss function, for any sample $\mathcal{S}$, the empirical risk $\hat{R}_{\phi,n}$ is an invex function on the set $\mathbb{F}^n :=\{(F(x_1), \cdots, F(x_n)): F \in \mathcal{F}\} \subset \RR^n$. Now the problem is whether we can decompose this empirical risk as a composition of a convex function and a differentiable function, that is, whether we can write $\hat{R}_{\phi,n} = g \circ A$ where $g: \RR^n 
\to \RR$ is a differentiable convex function and $A:\RR^n \to \RR^n$ is a differentiable vector function with $\nabla A$ of rank $n$. We let $g$ be the empirical risk $\hat{R}_{\exp,n}$ with exponential loss function. Then we want to find a function $A(\vec{x}) = (A_1(\vec{x}), \cdots, A_n(\vec{x}))$ such that
\begin{eqnarray} \label{eq:invex}
\hat{R}_{\phi,n}({ \mathbf{F}}) = \frac{1}{n} \sum_{i=1}^n \frac{1}{1+e^{y_i F(x_i)}} = \frac{1}{n} \sum_{i=1}^n e^{-y_i A_i({ \mathbf{F}})} = \hat{R}_{\exp,n}(A({ \mathbf{F}})),
\end{eqnarray}
where ${ \mathbf{F}} := (F(x_1), \cdots, F(x_n))$ and we ignore the constant 4 in the two-robust loss. Comparing each term in \eqref{eq:invex}, we get $\frac{1}{1+e^{y_i F(x_i)}} = e^{-y_i A_i({ \mathbf{F}})}$, that is,
\begin{eqnarray} \label{eq:A}
y_i\log (1+e^{y_i F(x_i)}) = A_i({ \mathbf{F}}).
\end{eqnarray}
The equation \eqref{eq:invex} also means that minimizing empirical risk $\hat{R}_{\phi,n}$ on the set $\mathbb{F}^n$ is equivalent to minimizing the empirical risk $\hat{R}_{\exp,n}$ on the transformed set $A(\mathbb{F}^n)$, where $A$ is defined in \eqref{eq:A}. On each data point, we have $F_{\exp}(x_i) = A_i({ \mathbf{F}}) = y_i\log (1+e^{y_i F(x_i)})$. We plot the function $l(v) := \log(1+e^{v})$ together with the identity function in Figure \ref{fig:invex}. If we compare two empirical risk minimization (ERM) problems (1) $\min_{\mathbf{F} \in \mathbb{F}^n} \hat{R}_{\exp,n}(\mathbf{F})$ and (2) $\min_{\mathbf{F} \in \mathbb{F}^n} \hat{R}_{\exp,n}(A(\mathbf{F}))$, then $A$ can be viewed as an "influence trimming" function. For any current estimate $\mathbf{F}$, we observe that whenever $y_iF(x_i) \gg 0$, then $F_{\exp}(x_i) \approx F(x_i)$; otherwise if $y_iF(x_i) \ll 0$, then $F_{\exp}(x_i) \approx 0$, that is, instead of saying $F$ made a severe mistake at $(x_i,y_i)$, after the $A$-transformation, we just say $F_{\exp}$ is uncertain about this point.

\begin{figure}[h!]
\centering
  \includegraphics[width=.35\linewidth]{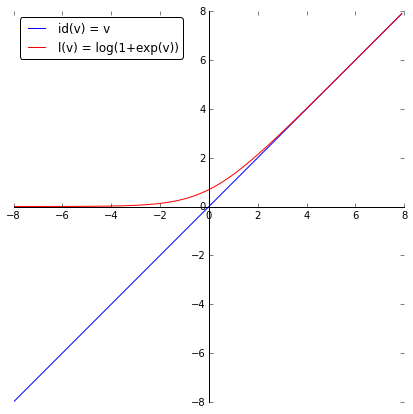}
\caption{Comparison of function $l$ and the identity function.}
\label{fig:invex}
\end{figure}

\subsubsection{Breakdown point}\label{sec:bp}
Empirical robustness properties    defined as breakdown point  in   \cite{DH83} has proved most successful in the context of location,
scale and regression problems (e.g. \cite{R84,SR92,T94}, etc.).   This success has sparked many  attempts to extend the concept
to other situations (e.g. \cite{RW01, GL03,DG05},etc.). However, very little work has been done in the classification context. 

The breakdown point, as defined in \cite{H68}, is  roughly  the smallest amount of contamination that may cause an estimator to take on arbitrarily large aberrant values. The breakdown points of $1/n$ for the
mean and $1/2$ for the median do reflect their finite-sample behavior.
However, an alternative view is desired in the classification context as the magnitude of an estimator may not relate to necessarily bad classification -- that is, 
 the size of the weak hypothesis is only marginally related to the classification boundary.  \cite{H04}  exploits  cluster stability   for  the breakdown point  analysis of estimators in the mixture models. However, in the context of boosting it is not clear if the same concept can differentiate between AdaBoost and LogisticBoost algorithms; the first is known not to be robust and the second is known to have better finite sample analysis.
  We argue  that  the case when the gradient of the classification loss takes on alternating directions  is in a certain sense  similar to  diverging  estimator values in regression setting. In this sense, the proposed breakdown study resembles Hampel's  original definition.
Hence, we look for the largest number of perturbed data points that keep the gradient of the risk minimization in the correct direction.  Moreover, we relate this concept to the weight updating rule of Arch Boost algorithms.


Let $\mathcal{S}_n = \{(X_1,Y_1), \cdots, (X_m, Y_m)\} \cup \{(X_{m+1},Y_{m+1}), \cdots, (X_n,Y_n) \}$ be a set of observed, contaminated samples among which $\mathcal{O} _n=  \{(X_{m+1},Y_{m+1}), \cdots, (X_n,Y_n) \}$ is a set of outliers. 
On the original dataset $\mathcal{S}$, let $h_t$ be the weak hypothesis at iteration $t$, and denote $ {\mathbf h}_t = (h_t(X_1), \cdots, h_t(X_n))$. We also denote $- {\mathbf {g}}_t = (-g_t(X_1), \cdots, -g_t(X_n))= (Y_1 w_t(X_1,Y_1), \cdots, Y_nw_t(X_n,Y_n))$ as the negative gradient of the empirical risk $\hat{R}_{\phi,n}$ on the sample $\mathcal{S}$. Similarly,   $- {\mathbf {g}_o} = (-g_t(X_1), \cdots, -g_t(X_m), 0, \cdots, 0)$ 
$= (Y_1 w_t(X_1,Y_1), \cdots, Y_mw_t(X_m,Y_m), 0, \cdots, 0)$ is obtained by embedding the negative gradient  of the empirical risk $\hat{R}_{\phi,m}$ on  the  sample $\mathcal{S} \setminus \mathcal{O}$ without outliers, into $\mathbb{R}^n$.
We  consider the inner product of the  weak hypothesis $ {\mathbf h}_t$ and the modified negative gradient  $- {\mathbf {g}_o}$. We have the following theorem.

\begin{thm}\label{thm:breakdown}
Using the same notations as above, 
let $\eta_j =  {\min^{-1}(p_j,1-p_j)}{|p_j-\frac{1}{2}|}$, for every region $R^j$, where $p_j \in (0,1)$ and $p_j \neq \frac{1}{2}$.
Then, if any Arch Boost algorithm, at iteration $t$, conditional on the realizations   $\left \{(X_i,Y_i)=(x_i,y_i) \right \}_{i=1}^n$, satisfies for all $R_j$    
\begin{eqnarray*}
\sum_{i : x_i \in \mathcal{O}_n \cap R^j} w_t(x_i,y_i) \le \eta_j \sum_{i : x_i \in R^j \setminus \mathcal{O}_n} w_t(x_i,y_i),
\end{eqnarray*}
 then  the gradient descent direction is preserved, that is, $- \langle  {\mathbf {g}_o} ,  {\mathbf h}_t \rangle \ge 0$.
\end{thm}

Theorem \ref{thm:breakdown} suggests that any Arch Boosting algorithm that satisfies conditions above, preserves the direction  of the descent of the non-contaminated empirical $\phi$-risk, hence disregarding the outliers. It is a sense it is an oracle property as it minimizes the risk on the ``clean'' data.
A few remarks are imminent.  Conditions of the above theorem are very mild. 
When $p_j = \frac{1}{2}$ -- that is, the total weight on positive labels is the same as that of the negative ones in region $R^j$ --   the elements of $ {\mathbf h}$ corresponding to the points in $R^j$ are $0$ and consequently have no influence on the sign of $- \langle  {\mathbf {g}_o} ,  {\mathbf h} \rangle$. Moreover, the case of  $p_j = 0$ or $1$ is not of the main interest  as in this case  when 
all the data in that region  have the same labels and hence it is reasonable to say there are no outliers.
For all other cases, when $p_j \in (0,1) \setminus \{ \frac{1}{2} \}$, Theorem \ref{thm:breakdown} establishes that  whenever  
the summation of weights on the outliers  is no larger than a constant $\eta_j$ times the summation of weights on the data points that are not outliers, then $ {\mathbf h}$ will also be the direction along which the empirical risk of the non-contaminated data will decrease.  Figure \ref{fig:2} shows the weight updating rule of the ArchBoost algorithm for various $\gamma$ and clearly illustrates that the condition above is more likely to be satisfied than the AdaBoost or LogitBoost algorithm.  For example, in the case of the ArchBoost,   if $y_i = -1$ and $\Prob(Y=-1|X=x_i) = 0.001$, then for Real AdaBoost, $\frac{w(x_i,y_i)}{w_b} \simeq 32$, and for ARB-2, $\frac{w(x_i,y_i)}{w_b} \simeq 0.008$ where $w_b$ is the weight for the outlying data point $(x_b,y_b)$ on the boundary, that is, the data such that  $F^*(x_b) = 0$. Hence, it can be seen that AdaBoost  puts 4000 times larger weight on this data compared to the ARB-$2$.

\subsubsection{Influence function}\label{sec:if}

The richest quantitative robustness measure is provided by the influence
function $u \to IF(u; T, G)$  of  $T$ at $G$  proposed by \cite{Hampel:74} and \cite{Hampel:86}. It is defined as
the first G$\hat{\mbox{a}}$teaux derivative of $T$ at $G$, where the point $u$ plays the role of the
coordinate in the infinite-dimensional space of probability distributions.  
In more details, 
the influence function of  any statistical functional $T$ at a distribution $\Prob$ is the special G$\hat{\mbox{a}}$teaux derivative (if it exists)
\begin{eqnarray*}
IF(z;T,\Prob) = \lim_{\epsilon \to 0^+} \frac{T((1-\epsilon)\Prob + \epsilon \Delta_z) - T(\Prob)}{\epsilon},
\end{eqnarray*}
where $\Delta_z$ is the Dirac distribution at the point $z$ such that $\Delta_z(\{z\}) = 1$. It  describes
the effect of an infinitesimal contamination at the point $u$ on the estimate, standardized
by the mass of the contamination. Additionally, it gives   the effect that an outlying
observation may have on an estimator.  If it is bounded, then the effect of an outlier is infinitesimal. 
It is straightforward to obtain the influence functions associated with the parametric 
estimators in linear regression models.  However,  the nonparametric regression models are far more complicated.

To simplify the analysis, we consider a subclass of binary classification models, in which the true boundary $F^*$ is assumed to belong to a class of functions    $H$. Here, $H$ is  defined as  a Reproducing Kernel Hilbert Space (denoted with RKHS from now on) with a bounded kernel $k$ and the induced norm $|| \cdot ||_{H}$.
Observe that ArchBoost,  much like AdaBoost and LogitBoost, belongs to a broad class of empirical risk minimization methods (see Theorem \ref{numerical convergence}) and converges only if it is properly regularized (stopped after a certain number of steps; see Theorem \ref{preconsistency}). Hence, to study its robustness properties we consider 
\begin{eqnarray*}
f_{\mathbb{P},\lambda} = \argmin_{f \in H} \left\{\E_{\Prob} \left[\phi(Y, f(X)) \right]+ \lambda ||f||_{H}^2 \right\},
\end{eqnarray*}
which can be viewed as the population quantity of interest.
Here,  the loss $\phi$ is a function of tuple $(Y,f(X))$ for convenience.
The solution $f_{\Prob, \lambda}$ is viewed as a map  that for every fixed value of the regularization parameter, $\lambda$, assigns an element of the RKHS $H$ to every distribution $\Prob$ on a given set $Z \subset \mathcal{X} \times \mathcal{Y}$ . Observe that $\PP$ is not contaminated  as the analysis herein is completed at the population level.
In the risk minimization context, $T(\Prob) = f_{\Prob, \lambda}$, and for each $\Prob$, the influence function $IF(z;T,\Prob) \in H$. We also denote the  map  $\Psi:\mathcal{X} \to H$, defined as $\Psi(x) = k(x,\cdot)$. The influence function of $f_{\PP,\lambda}$ then takes the form 
\begin{eqnarray} \label{eq:17}
IF(z;T,\Prob) = -S^{-1} \circ J,
\end{eqnarray}
where $\circ$ is defined to mean $S^{-1}$ acting on $J$ and operators $S: H \to H$ and $J \in H$  are defined as
\begin{align*}
S & =    \E_{\Prob} \left[\phi^{''} (Y, f_{\Prob,\lambda}(X)) \langle \Psi(X), \cdot \rangle \Psi(X) \right]+2\lambda id_H ,
\\
J&=  \phi^{'}(z_y, f_{\Prob,\lambda}(z_x)) \Psi(z_x)-\E_{\Prob} [\phi^{'}(Y, f_{\Prob,\lambda}(X))\Psi(X)],
\end{align*}
where $id_H:H \to H$ is the identity mapping and $z = (z_x,z_y) \in \mathcal{X} \times \mathcal{Y}$ is the contamination point.
In the above display, 
the derivative  is defined as $\phi^{'}(u, v) := \frac{\partial}{\partial v} \phi(u, v)$.
For the proof of this result, we  refer to the proof of Theorem 4 in \cite{Christmann:04} and of Theorem 15 in \cite{Christmann:07}. 
   In   the above the convexity of the loss function with respect to the second argument is required for  $S$ to be  invertible for all $\lambda > 0$.
   For a non-convex loss function $\phi$, $\phi^{''}$ is not guaranteed to be nonnegative. However,  we observe that it is sufficient to have the  non-negativity of the expectation rather than of the second derivative itself. Moreover the non-negativity   is only required in a local neighborhood of the true parameter of interest, $F^*$. 
   In the  following lemma, we show  the properties of   $\EE_{\PP} \left[\phi '' (Y, \cdot)\right]$. 

\begin{lemma}\label{lemma:hessian}
For a binary classification problem, given any distribution $\PP$, whenever $\phi$ is a twice continuous differentiable Arch boosting loss function and $F^*$ is obtained by \eqref{eq:5}, then $\E_{\Prob} \left[ \phi^{''}(Y,F^*(X))q^2(X) \right] \ge 0$ for any measurable function $q: \mathcal{X} \to \mathbb{R}$.

Furthermore, if $\PP$ and $\mathcal{X}$ are such  that $\PP(Y=1|X=x) \in [\delta, 1-\delta]$ for some $0<\delta< \frac{1}{2}$, and if $p\phi^{''}(1,v_p^*) + (1-p)\phi^{''}(-1,v_p^*) > 0$ at the global minimum $v_p^*$ for all $p \in [\delta,1-\delta]$, then there exists $r>0$ such that $\E_{\PP}\left[ \phi^{''}(Y,G(X))q^2(X)\right] \ge 0$ for all measurable function $G$ with $||G - F^*||_{\infty} < r$.
\end{lemma}

A few comments are necessary. Conditions of the above lemma are satisfied for the case of the $\gamma$-robust loss function.
For example,  for the case of the two-robust loss function, one can show that for any $x$, $\E_{Y} [\phi^{''}(Y,F^*(X))q^2(X)|X=x]
= 2p_x^2(1-p_x)^2 q^2(x) \ge 0$ where $p_x = \Prob(Y=1|X=x)$. Thus,   $\E_{\Prob} \phi^{''}(Y,F^*(X))q^2(X) \ge 0$. Furthermore, if $  p_x \in [\delta,1-\delta]$ for some $\delta \in (0,\frac{1}{2})$, then $p_x\phi^{''}(1,F^*(x)) + (1-p_x)\phi^{''}(-1,F^*(x))= 2p_x^2(1-p_x)^2  \ge 2\delta^2(1-\delta)^2 > 0$ for all $p_x \in [\delta,1-\delta]$. 
Observe that the condition of $  p_x \in [\delta,1-\delta]$ for some $\delta \in (0,\frac{1}{2})$ restricts our setting to the  ``low-noise''  setting where the true probability of the class membership is bounded away from $0$ or $1$. An example of a  model where such condition holds is    $\PP(Y=1|X=x) =(1+e^{-\gamma f(x)})^{-1}$ with $\mathcal{X}$ being a   compact space.

Recall that  $F^*$ is the solution to   the risk minimization problem defined over  all measurable functions. If we restrict the search  to the  RKHS, $H$, with bounded kernel, then $F^* \neq f_{\Prob,\lambda}$ in general. However,  as long as  $H$ is rich enough such that $f_{\Prob,\lambda}$ is close to $F^*$ in the sense that $||f_{\PP,\lambda} - F^*||_{\infty}$ is small, we have the following Theorem.
This Theorem  should be treated as a clue of the robustness of  non-convex loss functions as it 
shows  that the influence function is  bounded in $H$ and decreases when the contamination point $z$  is more of an outlier. Furthermore, this theorem again emphasizes that the robustness mainly comes from the fact that $|\phi^{'}| \to 0$, when the magnitude of the ``margin''   is large. 

\begin{thm} \label{thm:influence}
For a binary classification problem, let $\phi: \mathbb{R} \to [0,\infty)$ be a twice continuously differentiable Arch boosting loss function and let $H$ be a RKHS with bounded kernel $k$. Assume $\Prob$ is a distribution on $\mathcal{X} \times \mathcal{Y}$ such that $\PP(Y=1|X=x) \in [\delta, 1-\delta]$ for all $x \in \mathcal{X}$ and for some $0<\delta< \frac{1}{2}$ and $p \phi^{''}(1,v_p^*) + (1-p) \phi^{''}(-1,v_p^*) > 0$ at the global minimum $v_p^*$ for all $p \in [\delta,1-\delta]$. Then there exists $r > 0$ such that for all $||f_{\PP,\lambda} - F^*||_{\infty} < r$ 
\begin{eqnarray} \label{eq:18}
||IF(z;f_{\Prob,\lambda},\Prob)||_H \le  \sqrt{ \frac{C_{\phi}}{\lambda}} + \frac{M_k |\phi^{'}(z_y,f_{\PP,\lambda}(z_x))|}{2 \lambda},
\end{eqnarray}
where $M_k$ is the upper bound of the kernel $k$ and $C_{\phi} = \phi(0,0)$.
\end{thm}

For any non-convex Arch boosting loss function, due to  Assumption 2, we have $|\phi^{'}(z_y, f_{\Prob,\lambda}(z_x))| \to 0$ when $z_y f(z_x) \to -\infty$ or $z_y f(z_x) \to \infty$. Result of Theorem \ref{thm:influence} implies that for all smooth enough classification boundaries $F^*$,  the Arch Boosting algorithm has a bounded influence function
whenever  $\lambda >0$.
If we plot $\| IF(z;f_{\Prob,\lambda},\Prob)\|_H$ versus $z_y f_{\Prob,\lambda}(z_x)$, then it will decrease towards a constant far from the origin, just like the loss function of a redescending M-estimator \citep{H81}.
 Moreover, we observe that $\| IF(z;f_{\Prob,\lambda},\Prob)\|_H$ is unbounded for the exponential loss (AdaBoost) and bounded but not diminishing for the logistic loss (Logit Boost).  
  
\subsection{Statistical Consistency of Arch Boost} \label{sec:5}
Per Theorem \ref{numerical convergence}, we observe that the  Arch Boost Algorithm \ref{boosting}, can be formulated as  an empirical risk minimization procedure, for which the consistency has been established in the case of convex loss functions \citep{Bartlett:06, Bartlett:07, Jiang04, ZY05}. However, we extend  this framework to include non-convex losses  by exploring local curvatures.
Given any training sample $\mathcal{S}_n = \{(X_i,Y_i)\}_{i=1}^n$, 
 we compute a classifier $f_n$ on $\mathcal{S}_n$, and denote the misclassification error to be $L(f_n) = \Prob(f_n(X) \neq Y|\mathcal{S}_n)$. Moreover,   the Bayes risk is  $L^* = \inf_{f \in \mathcal{M}} L(f) = \E_X [\min(\eta(X),1-\eta(X))]$, where $\eta(X) = \Prob(Y=1|X)$. Let $\mathcal{M}$ be the family of measurable functions. The three key steps for proving consistency  \citep{Bartlett:2007} include: (I) utilizing the property of the loss that whenever the  $\phi-$risk $R_{\phi}(f_n)$ converges to the minimal risk 
 $$R_{\phi}^* = \inf_{f \in \mathcal{M}} R_{\phi}(f),$$ 
 the misclassification error  converges to the Bayes risk; (II) there exists a deterministic sequence $\tilde{f_n}$ for which the $\phi-$risk approaches  the minimal risk $R_{\phi}^*$ as $n \to \infty$ ; (III) $\phi-$risk of the estimated $f_n$ can be approximated by the $\phi$-risk of the deterministic, reference sequence $\tilde{f_n}$.
%
Observe that  $f_n$ is sample dependent, that is, $f_n$ is a function of $\mathcal{S}_n$, and hence $L(f_n)$ is again a random variable. Hence, in  step (II), we choose a deterministic reference sequence such that $F_n$ is only a function of  the sample size $n$.
The first condition (I) is automatically satisfied if we use Arch boosting loss function because of the classification calibration condition. 
%
Note that the above inequalities are true for any sample $\mathcal{S}_n$ and sample size $n$.
From now on, we will have several assumptions. The first one is about the class $\mathcal{H}$ and the distribution $\PP$.

\begin{condition1} \label{assumption:1}
Let the class of the weak hypothesis $\mathcal{H}$ and the probability distributions $\PP$ be such that  
$
\lim_{T \to \infty} \inf_{f \in \mathcal{F}^{T}} R_{\phi}(f) = R_{\phi}^*
$
and $d_{VC}\{\mathcal{H}\} <\infty$.
\end{condition1}
 For certain rich enough class $\mathcal{H}$, Assumption 1 is true for any distribution $\PP$ \citep{Bartlett:07}. One example is the class $\mathcal{T}$ of binary trees with the number of terminal nodes larger or equal to $d+1$, where $d$ is the dimension of $\mathcal{X}$ \citep{Breiman:04}. The next assumption is about the loss function $\phi$. 

\begin{condition2} \label{assumption:2}
Let the loss function $\phi$ be a bounded, decreasing, Lipschitz function that belongs to the class of Arch Boosting loss functions.
\end{condition2}
 
 If $\phi$ satisfies Assumption 2, then we know both $\lim_{v \to \infty} \phi(v)$ and $\lim_{v \to -\infty} \phi(v)$ exist in $\mathbb{R}$. A class of  loss functions that satisfy Assumption 2  incorporates
 a class of loss functions for which the
 first derivative  converges to zero away from the origin; an example is a redescending Hampel's three part loss function.    This lessens the effect of gross outliers and in turn leads to many
good robust properties of the resulting estimator. All the $\gamma$-robust loss functions ($\gamma \ge 1$) satisfy this condition. Commonly used loss functions like least squares, exponential loss and logistic loss do not satisfy this condition because they are not bounded.

Next we  state two results    important for proving the consistency of the Arch Boosting estimator.  Much like the consistency of AdaBoost, the proof hinders upon the optimal choice of stopping times. However,  the statement is not dependent on an additional truncation level of the functions $f$  (like Lemma 4 in \cite{Bartlett:07}) because of the boundedness of our loss function $\phi$. The nice decaying property of the derivative of the proposed Arch Boosting loss enables us to avoid additional parameters. The proof is given in  the Appendix A.

\begin{thm} \label{preconsistency}
Assume $\mathcal{H}$ and distribution $\PP$ satisfy Assumption 1. Let $\phi$ be an Arch boosting loss function satisfying Assumption 2. Then for any sample $\mathcal{S}_n  $ and for  a  non-negative sequence of stopping times $T_n = n^{1-\varepsilon}$ with $\varepsilon \in (0,1)$,
we have for $f^*_{n} = \argmin_{f \in \mathcal{F}^{T_n}} \hat{R}_{\phi,n}(f)$ as $n \to \infty$,
 
 \centerline{(a)  $\sup_{f \in \mathcal{F}^{T_n}} |\hat{R}_{\phi,n}(f) - R_{\phi}(f)| \to 0  \;\; a.s.$
   \ \ and \ \ (b)   {$R_{\phi}(f^{*}_n) \to R_{\phi}^* \;\; a.s.$}}

\end{thm}

Theorem \ref{preconsistency} illustrates the uniform deviation between the $\phi$-risk and the empirical $\phi$-risk. Note that we want $T_n \to \infty$ as $n \to \infty$ but not too fast (slower than $\mathcal{O}(n)$) in order to make sure that the uniform deviation  converges to zero when $n \to \infty$. Moreover, from part (b) of Theorem \ref{preconsistency}, we know there exists a sequence of samples $\{S^*_n\}_{n=1}^{\infty}$ such that $R_{\phi}(\tilde{f}_n) \to R^*_{\phi}$ as $n \to \infty$ where $\tilde{f}_n$ is the optimal classifier obtained by minimizing the empirical risk on $S^*_n$. In another word, from $\omega \in \Omega$ such that $R_{\phi}(f^{*}_n(\omega)) \to R_{\phi}^*$ as $n \to \infty$, we pick a $\omega_0$ and set $\tilde{f}_n = f^*_n(\omega_0)$. We let $\{\tilde{f}_n\}$ be our reference sequence, and note that $\tilde{f}_n$ only depends on sample size $n$, that is, $\{\tilde{f}_n\}$ is a deterministic sequence. This is necessary  to decouple the dependence of the estimator and the sample. 
Next we state the intermediary lemma that connects the reference sequence, $\tilde f_n$, to the Arch Boost estimator, $F_{T_n}$.
 
\begin{lemma} \label{preconsistency2}
For the above reference sequence $\{\tilde{f}_n\}_{n=1}^{\infty}$ and a  non-negative sequences $T_n = n^{1-\varepsilon}$, $\varepsilon \in (0,1)$, we have as $n \to \infty$,
 (a) $\left(\hat{R}_{\phi,n}(\tilde{f}_n) - R_{\phi}(\tilde{f}_n)\right)_+ \to 0 \;\; a.s.$ and 
  (b)  $\left(\hat{R}_{\phi,n}(F_{T_n}) - \hat{R}_{\phi,n}(\tilde{f}_n)\right)_+ \to 0  \;\; a.s.$
 
\end{lemma}

Lemma \ref{preconsistency2} is proved easily with the help of Theorems \ref{numerical convergence} and \ref{preconsistency}. Together with the result of 
 Theorem \ref{preconsistency} we can finally state the consistency result.

\begin{cor}   \label{consistency}
Assume the class $\mathcal{H}$ and distribution $\PP$ satisfy Assumption 1 and that the Arch boosting loss function $\phi$ satisfies Assumption 2. Then, as $n \to \infty$, the sequence of classifiers $F_{T_n}$ returned by the Arch Boost  algorithm  stopped at the step $T_n$, chosen in Theorem \ref{preconsistency}, satisfies $L(sign(F_{T_n})) \to L^*$ a.s.
\end{cor}


\section{Numerical Experiments} \label{experiments}

 
 
In this section, we will test the Arch Boost $(\phi)$ algorithm under different loss functions $\phi$ on binary classification problems. 

\subsection{Simulated Examples}\label{gamma}
We will generate datasets using 'make-hastie-10-2' and 'make-gaussian-quantiles'  \citep{Pedregosa:11}. For 'make-hastie-10-2', the data have 10 features that are standard independent Gaussian and the target $y_i=1$ if $\sum_{i=1}^{10} x_i^2 > \chi^2_{10,1/2}$ \citep{FHT}. For 'make-gaussian-quantiles', in our case, the dataset is constructed by taking a 20-dimensional normal distribution $\mathcal{N}(\mathbf 0, 2I_{20})$ and defining two classes  as $y = 1$ if $\sum_{i=1}^{20} x_i^2 > 4 \chi^2_{20,1/2}$. In both datasets, the classes are separated by a concentric multi-dimensional spheres with origin at $\mathbf 0$ such that roughly equal numbers of samples are in each class.

We test the robustness of ARB-$\gamma$ algorithms by adding noise to different percentages of the training samples. In the following two examples, we add independent t-distribution (df = 4) noise to the features of a \% of selected training samples. The results are summarized in Figure \ref{fig:test}. For each dataset, we generate 14000 samples using the corresponding methods. Among the $14000$ data,  we use 2000 for training, 2000 for cross validation, and the rest 10000 for testing. The number of weak classifiers is $1000$, and by cross validation, we set the step sizes, $\alpha_t$, to be $0.78$ for ARB-1.5, 0.45 for ARB-2, 0.28 for ARB-3, 0.20 for ARB-4, 0.14 for ARB-5, 0.10 for ARB-6, and 0.80 for Real AdaBoost. For RobustBoost, we tune the target parameter for each percentage of errors using bisection search. In each figure, we plot the average test errors and the corresponding $95\%$ confidence intervals.

\begin{figure}[h!]
\centering
\begin{subfigure}{.5\textwidth}
  \centering
  \includegraphics[width=.9\linewidth]{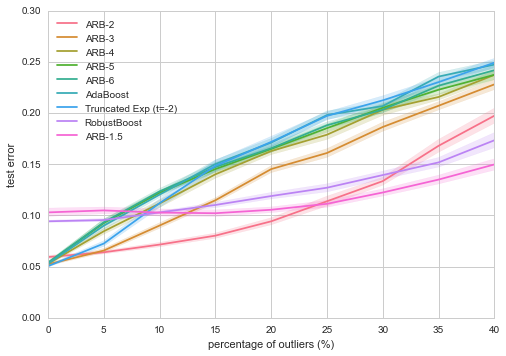}
  \caption{make-hastie-2 dataset.}
  \label{fig:sub1}
\end{subfigure}%
\begin{subfigure}{.5\textwidth}
  \centering
  \includegraphics[width=.9\linewidth]{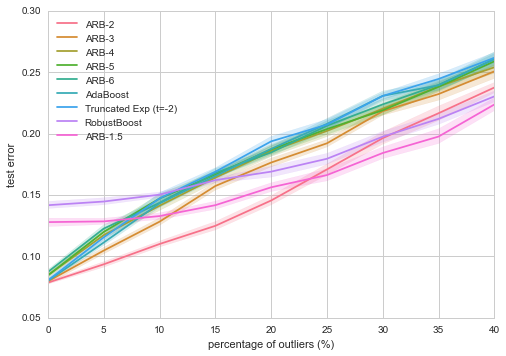}
  \caption{Gaussian quantiles dataset.}
  \label{fig:sub2}
\end{subfigure}
\caption{Comparison of average test errors of ARB-$\gamma$, AdaBoost, Arch Boost with truncated exponential loss (truncated at $t=-2$) and RobustBoost.}
\label{fig:test}
\end{figure}

From Figure \ref{fig:test}, we have several observations.
First, the test errors of the ARB-$\gamma$ algorithms are all less than that of the Real AdaBoost. Moreover, they are smaller than that of the truncated exponential loss, indicating that the introduced non-convex family of losses substantially outperforms traditional truncation losses.
Second, when the percentage of outliers is less than a certain level (around $23\%$ in Figure \ref{fig:sub1} and Figure \ref{fig:sub2}), the performances of ARB-2 is the best. When the noise level is higher, ARB-1.5 behaves the best.
Moreover,  RobustBoost has higher test error when the noise level is low.  We run the RobustBoost algorithm for $1000$ iterations and obtain an error that is larger than zero. However, if we run RobustBoost for long enough   when error rate is 0, its performance should converge to that of Real AdaBoost.
RobustBoost and ARB-1.5 have very similar performance. ARB-2 is worse than RobustBoost when the noise level is very high. However,  note that we tune a number of tuning  parameters of the RobustBoost at each noise level.  If  we were to also "tune" $\gamma$ for ARB-$\gamma$ algorithms, for example, choose ARB-2 when noise level is less than $25\%$ and ARB-1.5 otherwise in Figure \ref{fig:sub1}, then we could observe that ARB-$\gamma$ is uniformly better than the RobustBoost.

To illustrate the importance of the loss function choice and the Arch Boosting method, we implement a Gradient Descend Boosting algorithm with  a ``trimmed'' version of the exponential loss, i.e., the truncated exponential loss function. We observe that the improvement over AdaBoost is extremely minor and it disappears when the dimensionality of the problem grows. For 'make-hastie-10-2' dataset truncated exponential loss is better than $\gamma$-robust losses for $\gamma \geq 4$ and $\varepsilon <15\%$ after which point it  is very much indistinguishable from the AdaBoost. Situation is even better for 'make-gaussian-quantiles' dataset as truncated exponential is almost  identical as the AdaBoost as soon as $\varepsilon >7\%$. This suggests that the Arch Boosting framework is essential    for robust  and generalizable performance.

\subsection{Long/Servedio problem} \label{ls}
\cite{LS:10} constructed a challenging classification setting described as follows. The input $X \in \realR^{21}$ with binary features $X_i \in \{-1,+1\}$ and label $y_i \in \{-1,+1\}$. First, the label $y$ is chosen to be $-1$ or $+1$ with equal probability. Then for any given $y$, the features $X_i$ are generated according to the following mixture distribution:
\begin{itemize}
\item
\textbf{Large margin:} With probability $\frac{1}{4}$, set $X_i = y$ for all $1 \le i \le 21$.
\item
\textbf{Pullers:} With probability $\frac{1}{4}$, set $X_i = y$ for $1 \le i \le 11$ and $X_i = -y$ for $12 \le i \le 21$.
\item
\textbf{Penalizers:} With probability $\frac{1}{2}$, randomly choose 5 coordinates from the first 11 features and 6 from the last 10 to be equal to $y$. The remaining features are set to $-y$.
\end{itemize} 

The data from this distribution can be perfectly classified by $sign(\sum_i X_i)$. We  generate 800 samples from this distribution and flip each label with probability $\epsilon \in [0,0.5)$. Then we train the classifier on the noisy data and test the performance on the original clean data. We first generate 20 datasets according to the distribution, and on each of them, we randomly flip $\epsilon = 10\%$ of the labels. The result is in Table \ref{tab:LS}, where we record the average test error and also report the sample deviations in the brackets. We can see that the ARB-2 outperforms Real AdaBoost and LogitBoost \citep{FHT:00}, and is even better than RobustBoost (target parameter $\theta=0.15$) \citep{Freund:01}.

\begin{table}[h!]
\centering
  \begin{tabular}{ | p{2.7cm} | p{3cm} | p{3cm} | p{2.7cm} | p{2.7cm} | }
    \hline
    data type & Real AdaBoost  & LogitBoost  & RobustBoost ($\theta=0.15$) & ARB-2\\ \hline
    noise($\epsilon=0.1$) & $28.24\%(1.53\%)$ &$26.61\%(1.51\%)$  & $11.04\%(0.67\%)$& $\boldsymbol{9.82\%}(0.43\%)$ \\ \hline
    clean & $25.07\%(1.92\%)$  & $22.59\%(1.74\%)$ & $0.21\%(0.35\%)$ & $\boldsymbol{0.02\%}(0.04\%)$\\ \hline
  \end{tabular}
  \caption{Long/Servedio problem}
  \label{tab:LS}
\end{table}

We also compare the performance of different ARB-$\gamma$ and plot the average test errors and $95\%$ confidence intervals in Figure \ref{fig:LS}. We can see from Figure \ref{fig:LS} that ARB-1.5 behaves the best on this dataset among all these algorithms. When $\gamma$ increases, the performance of ARB-$\gamma$ is approaches that of the Real AdaBoost. The breakdown point will get higher when $\gamma \to 1^+$, implying that the smaller $\gamma$ lead to the better robustness properties. When $\gamma = 2$, then breakdown point is about $15\%$, and when $\gamma = 1.5$ and $1.3$, the breakdown point is about $20\%$. But for ARB-1.3, the test error when labels are not flipped is not zero.

\begin{figure}[h!]
\centering
\includegraphics[scale=0.50]{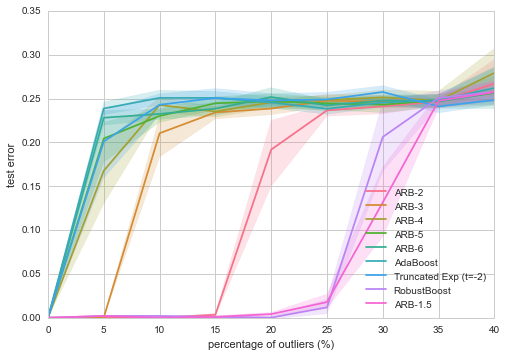}
\caption{Comparison of ARB-$\gamma$ on Long/Servedio problem with different $\epsilon$ (the probability of flipping the labels).}
\label{fig:LS}
\end{figure}

\subsection{Outlier  detection} \label{outliers}
We have shown in previous sections that ARB-$\gamma$ algorithms are more robust to the noise. Therefore, because of the robustness, ARB-$\gamma$ should be able to detect the outliers. Intuitively, if a point is an outlier, then it should be misclassified by most of the weak hypotheses of ARB-2. In this experiment, we generate $2000$ data points using 'make-hastie-10-2' and randomly shuffle them. Then we add a noise drawn from a t-distribution (df $= 4$) to each of the 10 features of the first $\epsilon$ percentage data points. After running the algorithms for $800$ iterations, we record the times that each data point is misclassified, and count the number of points that are misclassified more than $600$ times (denoted as $T$), and count how many of them (denoted as $T_o$) actually belong to the noisy set that we add noise to. Finally we calculate the ratio $\frac{T_o}{T}$. This ratio describes the chance that a data point is an outlier. By cross-validation, we set the step size  $\alpha=0.5$ for the ARB-2 and $\alpha = 0.8$ for the Real AdaBoost. The results are shown in Table \ref{tab:outliers}. The x-axis stands for the index of the training points ranging from $1$ to $2000$, and the y-axis stands for the times a point is misclassified, ranging from $0$ to $800$. We can see that when the percentage of outliers is less than $15\%$, for the ARB-2, more than $99\%$ of the points that have been misclassified for $600$ times or above, are indeed the outliers, but for the Real AdaBoost, this number is only around $31\%$. Informally, for ARB-2, when $\epsilon \le 15\%$, we have more than $99\%$ ``confidence'' to conclude that a data point, that is misclassified for more than $600$ times, is an outlier.

\begin{table}[h]
\centering
\begin{tabular}{|c|c|c|}
      \hline
      $\epsilon$ & ARB-2 & Real AdaBoost \\
      \hline
      $5\%$ &
      \addheight{\includegraphics[scale=0.35]{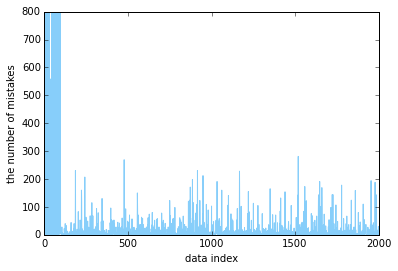}} &
      \addheight{\includegraphics[scale=0.35]{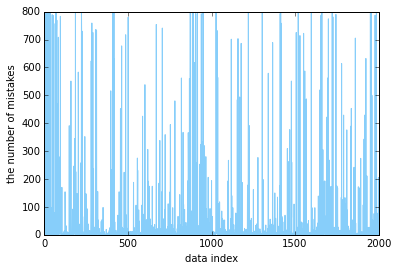}} \\
      \hline
      
      $T_o/T$ & $100\%$ & $30.49\%$ \\
      \hline
      
    $10\%$ &
      \addheight{\includegraphics[scale=0.35]{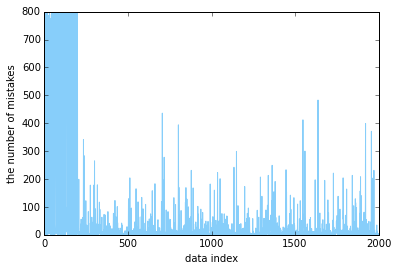}} &
      \addheight{\includegraphics[scale=0.35]{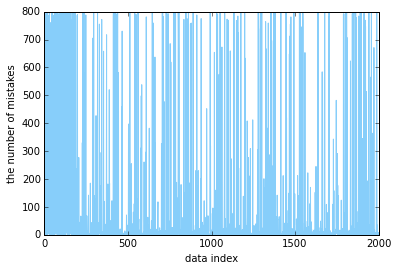}} \\
      \hline
      
      $T_o/T$ & $100\%$ & $32.22\%$ \\
      \hline
      
    $15\%$ &
      \addheight{\includegraphics[scale=0.35]{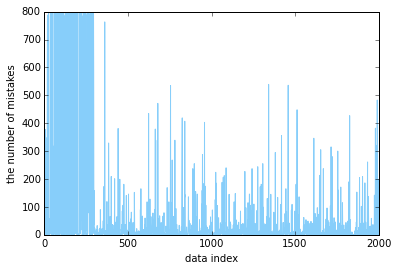}} &
      \addheight{\includegraphics[scale=0.35]{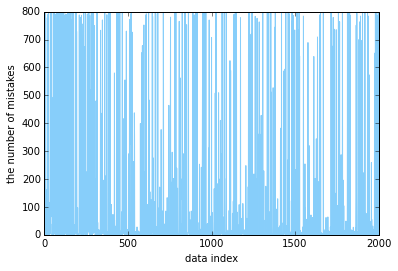}} \\
      \hline    
        
        $T_o/T$ & $99.04\%$ & $37.38\%$ \\
      \hline
      
      $20\%$ &
      \addheight{\includegraphics[scale=0.35]{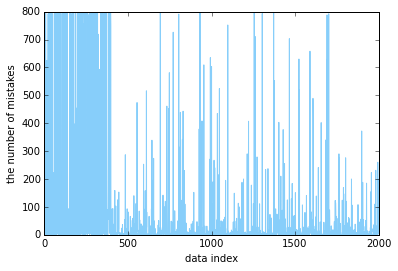}} &
      \addheight{\includegraphics[scale=0.35]{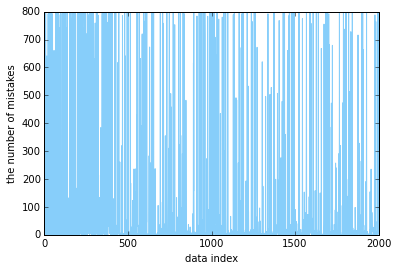}} \\
      \hline
      
      $T_o/T$ & $85.48\%$ & $44.40\%$ \\
      \hline
      
\end{tabular}
\caption{Outliers detection. The x-axis stands for the index of the training points ranging from 1 to 2000, and the y-axis stands for the times a point is misclassified, ranging from 0 to 800.}\label{tab:outliers}
\end{table}

\clearpage
\subsection{Real data application} \label{sec:realdata}
\subsubsection{Wisconsin (diagnostic) breast cancer data set}
We test our ARB-$\gamma$ algorithms on the Wisconsin (diagnositc) breast cancer data set of \cite{Street:93}, which is available on the machine learning repository website at the University of California, Irvine: \url{https://archive.ics.uci.edu/ml/datasets/Breast+Cancer+Wisconsin+(Diagnostic)}. The data set was created by taking measurements from a digitized image of a fine needle aspirate of a breast mass for each of 569 individuals, with 357 benign and 212 malignant instances. 
Ten real-valued features are computed for each cell nucleus: 
  radius, texture, perimeter, area, smoothness, compactness, concavity,  concave points, symmetry, fractal dimension.

 \begin{table}[h!]
    \centering
    \footnotesize
    \caption{Comparison of the average test errors and sample deviation of four algorithms on the Wisconsin breast cancer dataset.}
  \label{tab:wdbc}
    \begin{tabular}{@{}nd{13.1}*{15}{d{13.2}}d{14.1}d{14.2}d{13.2}@{}}
      \toprule
        \multicolumn{1}{@{}N}{Percentage of flipped labels} &
        \multicolumn{3}{N@{}}{Methods} &
        \\
      \cmidrule(lr){2-5}
        &
        \multicolumn{1}{V{6.5em}}{ARB-$2$} &
        \multicolumn{1}{V{6.5em}}{ARB-$1.5$} 
            &
        \multicolumn{1}{V{6.5em}}{Robust Boost} &
        \multicolumn{1}{V{6.5em}}{Ada Boost} 
         \\
      \cmidrule(lr){2-2}\cmidrule(lr){3-3}   \cmidrule(lr){4-4} \cmidrule(lr){5-5} 
         0$\%$   & ${\bf{3.35}}\%(1.29\%)$& $3.49\%(1.23\%)$   &  $4.52\%(1.61\%)$ & $4.06\%(1.48\%) $ \\
     5$\%$ & $4.71\%(1.76\%)$  & ${\bf{4.43\%}}(1.63\%)$ & $4.78\%(1.66\%)$ & $5.38\%(1.92\%)$\\ 
     10$\%$ & $5.71\%(1.71\%)$  & ${\bf{5.03\%}}(1.60\%)$ & $5.35\%(1.77\%)$ & $6.24\%(1.97\%)$\\ 
    15$\%$ & $6.92\%(2.06\%)$  & ${\bf{5.84\%}}(2.04\%)$ & $6.47\%(2.14\%)$ & $7.01\%(2.28\%)$\\        
      \bottomrule
    \end{tabular}
  \end{table}

We randomly split the data into an equally balanced training set with 150 benign samples and 150 malignant samples, and the rest samples were used for testing. The maximum iterations is set to to be 200, and a five-fold cross-validation is implemented on the training set to select the step size and stopping time ($\le 200$) for each algorithm. This procedure was repeated for 100 times and an average of the test error is reported in Table \ref{tab:wdbc}. 
Boxplots of the test error are presented in  Figure \ref{fig:wdbc}. 
\begin{figure}[h]
\centering
\begin{subfigure}{.3\textwidth}
  \centering
  \includegraphics[width=1.\linewidth]{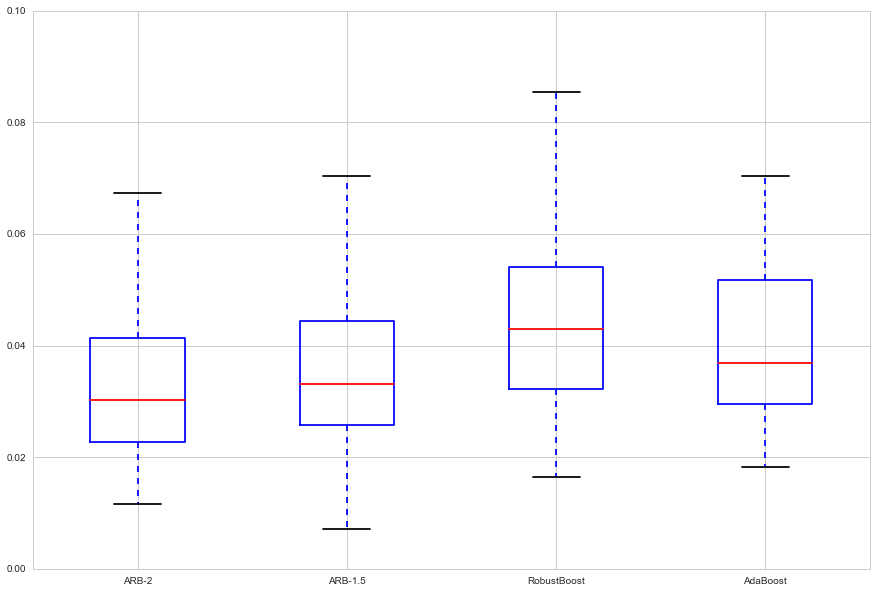}
  \caption{$0\%$.}
  \label{fig:wdbcsub1}
\end{subfigure}%
\begin{subfigure}{.3\textwidth}
  \centering
  \includegraphics[width=1.\linewidth]{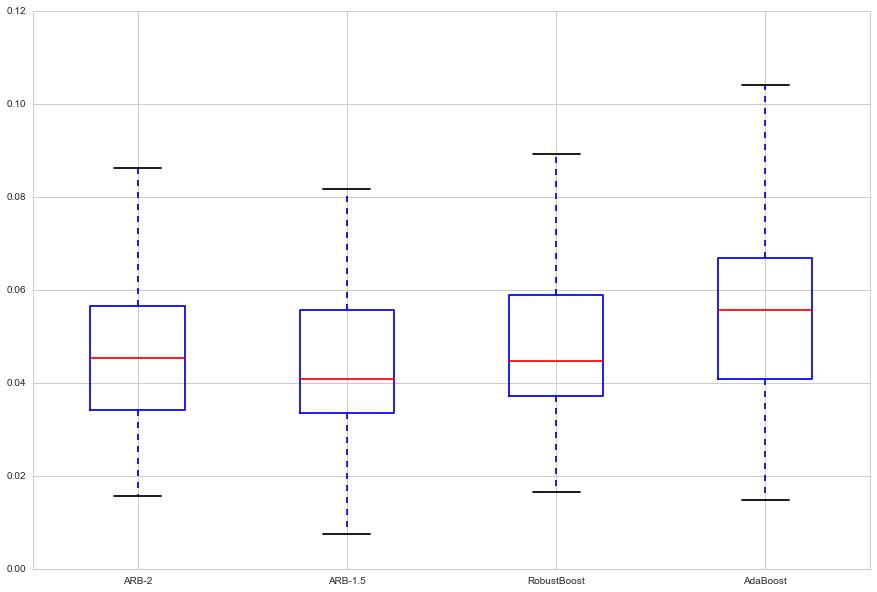}
  \caption{$5\%$.}
  \label{fig:wdbcsub2}
\end{subfigure}
\begin{subfigure}{.3\textwidth}
  \centering
  \includegraphics[width=1.\linewidth]{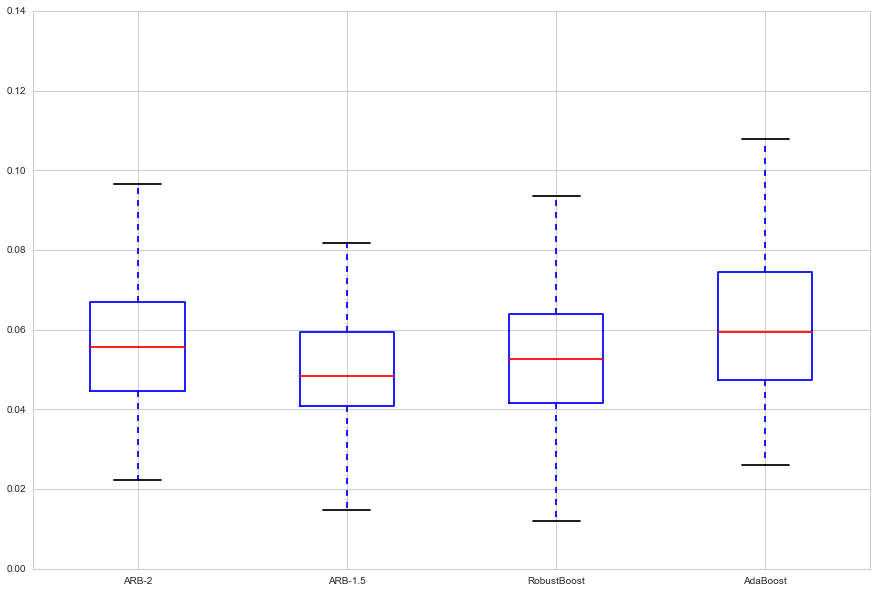}
  \caption{$10\%$.}
  \label{fig:wdbcsub3}
\end{subfigure}%
\begin{subfigure}{.3\textwidth}
  \centering
  \includegraphics[width=1.\linewidth]{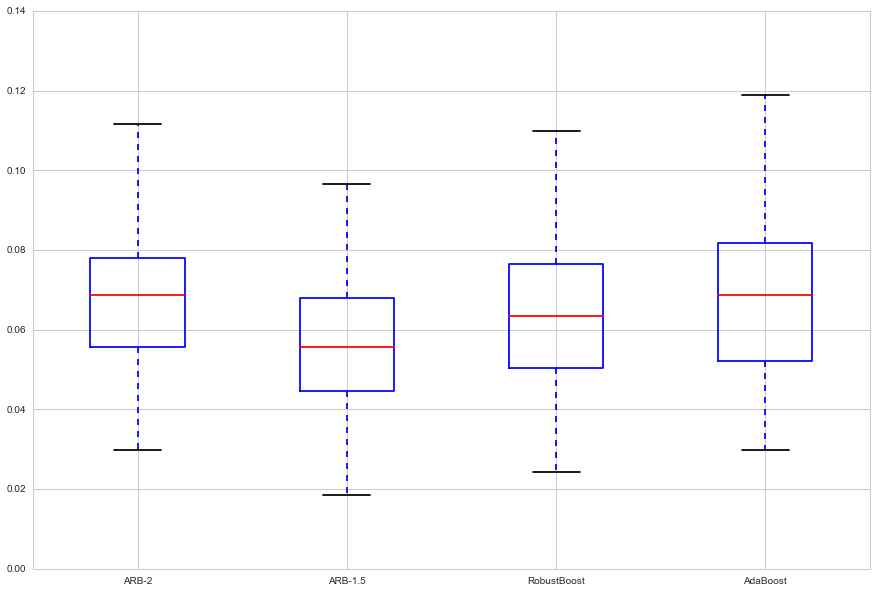}
  \caption{$15\%$.}
  \label{fig:wdbcsub4}
\end{subfigure}
\caption{Comparison of ARB-2, ARB-1.5, RobustBoost and Real AdaBoost on the UCI Wisconsin breast cancer dataset. In each subfigure, from left to right are the box plots for the test errors of ARB-2, ARB-1.5, RobustBoost and Real AdaBoost.}
\label{fig:wdbc}
\end{figure}
We observe that ARB-2 behaves the best on the original data set, and ARB-1.5 outperforms others when there is noise.
Compared to \cite{S14} who obtain the best test error rate of about $4\%$, all of our methods uniformly achieve smaller test error rate, on the clean and  comparable test error rates on the perturbed datasets.

\subsubsection{Sensorless drive diagnosis data set}
We compare ARB-2, ARB-1.5, RobustBoost and Real AdaBoost on the dataset sensorless drive diagnosis, which is also available on the UCI machine learning repository: \url{https://archive.ics.uci.edu/ml/datasets/Dataset+for+Sensorless+Drive+Diagnosis}. 

 \begin{table}[h!]
    \centering
    \footnotesize
    \caption{Comparison of the average test errors and sample deviation of four algorithms on the Sensorless drive diagnosis dataset.}
  \label{tab:real}
    \begin{tabular}{@{}nd{14.1}*{15}{d{14.2}}d{14.7}d{14.2}d{14.2}@{}}
      \toprule
        \multicolumn{1}{@{}N}{Percentage of flipped labels} &
        \multicolumn{3}{N@{}}{Methods} &
        \\
      \cmidrule(lr){2-5}
        &
        \multicolumn{1}{V{6.5em}}{ARB-$2$} &
        \multicolumn{1}{V{6.5em}}{ARB-$1.5$} 
            &
        \multicolumn{1}{V{6.5em}}{Robust Boost} &
        \multicolumn{1}{V{6.5em}}{Ada Boost} 
         \\
      \cmidrule(lr){2-2}\cmidrule(lr){3-3}   \cmidrule(lr){4-4} \cmidrule(lr){5-5} 
         0$\%$ & ${\bf{5.84\%}}(0.43\%)$ &$6.52\%(0.43\%)$  & $7.37\%(0.36\%)$& $7.04\%(0.44\%)$ \\ 
     5$\%$ & $9.68\%(0.58\%)$  & ${\bf{8.79\%}}(0.57\%)$ & $8.96\%(0.51\%)$ & $11.52\%(0.78\%)$\\ 
     10$\%$ & $12.35\%(1.03\%)$  & $10.88\%(0.94\%)$ & ${\bf{10.81\%}}(0.74\%)$ & $14.63\%(0.98\%)$\\ 
     15$\%$ & $14.16\%(0.85\%)$  & $12.50\%(0.76\%)$ & ${\bf{11.95\%}}(0.99\%)$ & $17.88\%(0.81\%)$\\ 
      \bottomrule
    \end{tabular}
  \end{table}
  This dataset contains 58509 instances and each of them has 49 features all extracted from the electric current drive signals. A range of typical defects in drive train applications are considered with  11 different classes present. We combine the data points with label $\le 6$ into one class and the rest into the other class. Then at each time, we randomly choose 14000 points and use 2000 for training, 2000 for cross validation and 10000 for testing. We use the cross validation set to choose the stopping time ($\le 1000$ iterations) and step sizes. According to the noise levels, a certain proportion of the labels of the training data points will be randomly flipped. We summarized the test errors using box plots in Figure \ref{fig:real} and calculated the mean and sample deviation in Table \ref{tab:real}. 
  
   \begin{figure}[h]
\centering
\begin{subfigure}{.3\textwidth}
  \centering
  \includegraphics[width=1.\linewidth]{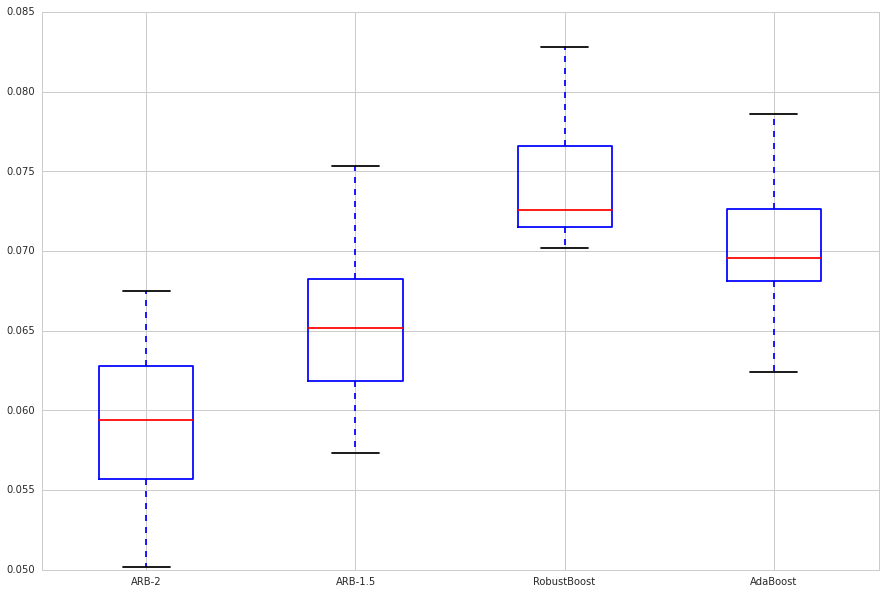}
  \caption{$0\%$.}
  \label{fig:realsub1}
\end{subfigure}%
\begin{subfigure}{.3\textwidth}
  \centering
  \includegraphics[width=1.\linewidth]{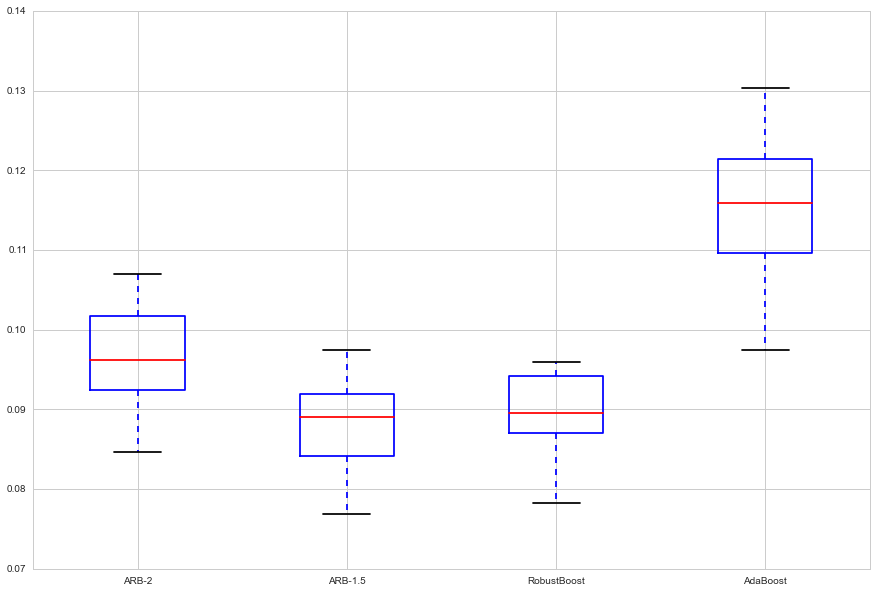}
  \caption{$5\%$.}
  \label{fig:realsub2}
\end{subfigure}

\begin{subfigure}{.3\textwidth}
  \centering
  \includegraphics[width=1.\linewidth]{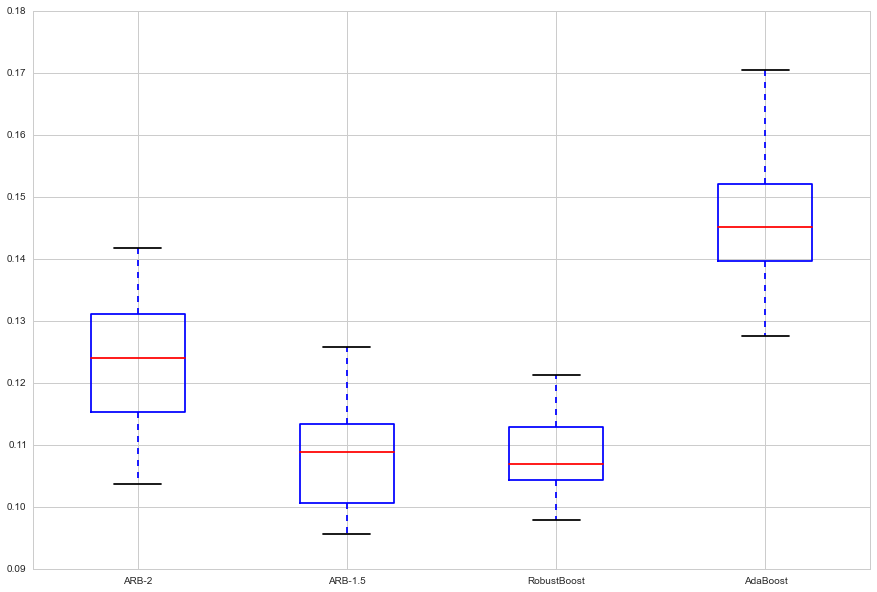}
  \caption{$10\%$.}
  \label{fig:realsub3}
\end{subfigure}%
\begin{subfigure}{.3\textwidth}
  \centering
  \includegraphics[width=1.\linewidth]{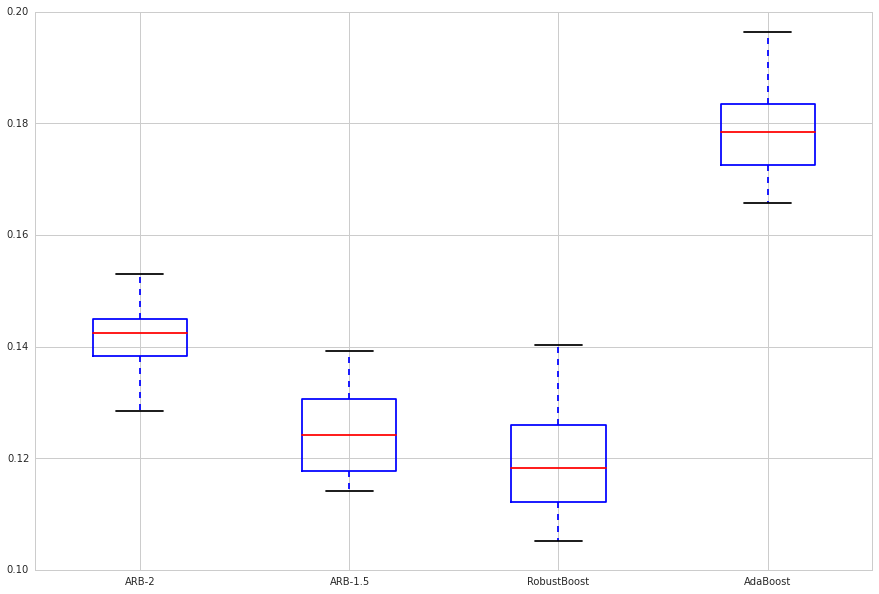}
  \caption{$15\%$.}
  \label{fig:realsub4}
\end{subfigure}

\caption{Comparison of ARB-2, ARB-1.5, RobustBoost and Real AdaBoost on the UCI sensorless drive diagnosis dataset. In each subfigure, from left to right are the box plots for the test errors of ARB-2, ARB-1.5, RobustBoost and AdaBoost.}
\label{fig:real}
\end{figure}

We observed that ARB-2 performs the best on the original data without adding any extra noise. RobustBoost again behaves worse than others on the original dataset. One reason is that it needs more time to terminate when target error is near 0, and the other reason is that it cannot distinguish outliers and hard inliers \citep{Kobe:13}. When we flipped $10\%$ of the labels, ARB-1.5 outperformed others and when we flipped $10\%$ or $15\%$ of the labels, RobustBoost behaved the best. But in all of the three cases with noise, the test errors of ARB-1.5 and RobustBoost are very close. However, ARB-1.5 does not need to fine  tune any target parameters  at each different noise levels.

\subsubsection{MAQC-II Project: human breast cancer (BR) data set}
We next test our Algorithms on a dataset that is part of the 'MicroArray quality control II' project. It is available from the gene expression omnibus database with accession number GSE20194: \url{http://www.ncbi.nlm.nih.gov/geo/query/acc.cgi?acc=GSE20194}. 
 \begin{table}[h!]
    \centering
    \footnotesize
    \caption{Comparison of the average test errors and sample deviation of four algorithms on the GSE20194 gene dataset.}
  \label{tab:GSE20194}
    \begin{tabular}{@{}nd{14.1}*{15}{d{14.2}}d{14.7}d{14.2}d{14.2}@{}}
      \toprule
        \multicolumn{1}{@{}N}{Percentage of flipped labels} &
        \multicolumn{3}{N@{}}{Methods} &
        \\
      \cmidrule(lr){2-5}
        &
        \multicolumn{1}{V{6.5em}}{ARB-$2$} &
        \multicolumn{1}{V{6.5em}}{ARB-$1.5$} 
            &
        \multicolumn{1}{V{6.5em}}{RobustBoost} &
        \multicolumn{1}{V{6.5em}}{AdaBoost} 
         \\
      \cmidrule(lr){2-2}\cmidrule(lr){3-3}   \cmidrule(lr){4-4} \cmidrule(lr){5-5} 
      
         0$\%$ & $9.66\%(1.90\%)$ &${\bf{9.45\%}}(1.93\%)$  & $10.24\%(2.10\%)$& $10.51\%(2.20\%)$ \\ 
     5$\%$ & $11.35\%(2.50\%)$  & ${\bf{11.16\%}}(2.20\%)$ & $11.82\%(2.95\%)$ & $12.40\%(2.89\%)$\\ 
     10$\%$ & $12.92\%(3.43\%)$  & ${\bf{12.17\%}}(3.50\%)$ & $12.22\%(3.04\%)$ & $14.76\%(3.54\%)$\\ 
     15$\%$ & $13.46\%(4.91\%)$  & $13.33\%(5.03\%)$ & ${\bf{13.29\%}}(3.13\%)$ & $15.85\%(6.85\%)$\\ 
      \bottomrule
    \end{tabular}
  \end{table}

 The dataset contains 278 newly diagnosed breast cancer  patients, aged from 26 to 79 years with population spanning all three major races and their mixtures. Patients received 6 months of preoperative   chemotherapy followed by surgical resection of the cancer. Estrogen-receptor status helps guide treatment for breast cancer patients because breast cancer contains many estrogen receptors. Of 278 patients,  164 had positive estrogen-receptor status and 114 have negative estrogen-receptor status. Each sample is described by 22283  biomarker probe-sets.  
 \begin{figure}[h]
\centering
\begin{subfigure}{.3\textwidth}
  \centering
  \includegraphics[width=1.\linewidth]{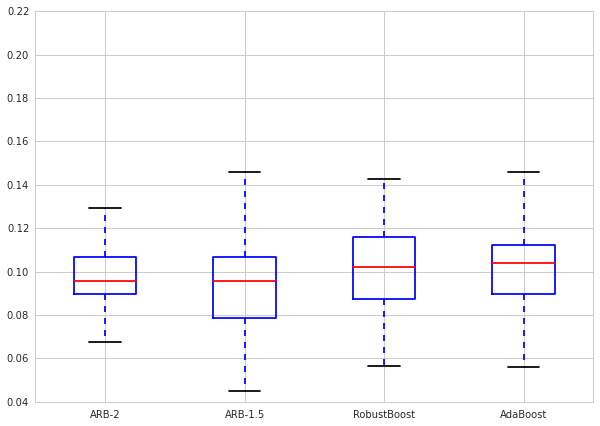}
  \caption{$0\%$.}
  \label{fig:GSEsub1}
\end{subfigure}%
\begin{subfigure}{.3\textwidth}
  \centering
  \includegraphics[width=1.\linewidth]{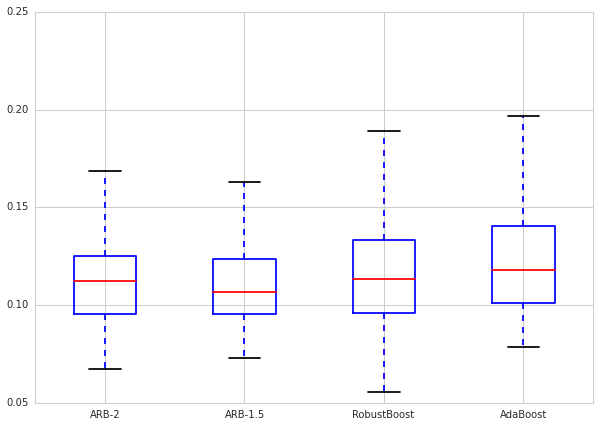}
  \caption{$5\%$.}
  \label{fig:GSEsub2}
\end{subfigure}
\begin{subfigure}{.3\textwidth}
  \centering
  \includegraphics[width=1.\linewidth]{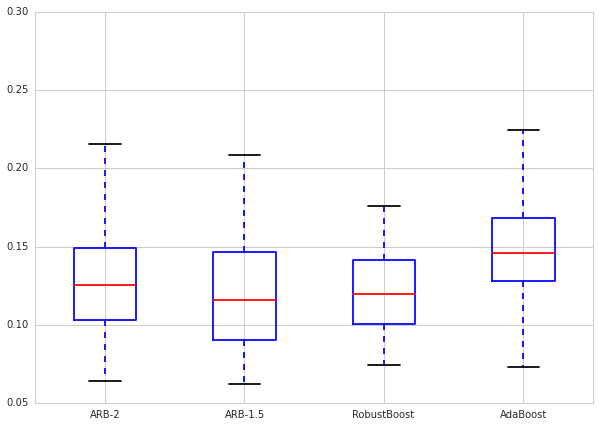}
  \caption{$10\%$.}
  \label{fig:GSEsub3}
\end{subfigure}%
\begin{subfigure}{.3\textwidth}
  \centering
  \includegraphics[width=1.\linewidth]{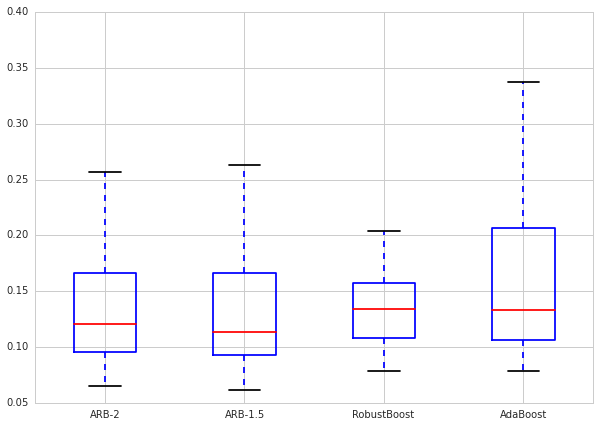}
  \caption{$15\%$.}
  \label{fig:GSEsub4}
\end{subfigure}

\caption{Comparison of ARB-2, ARB-1.5, RobustBoost and Real AdaBoost on the GSE20194 gene dataset, from left to right are the box plots for the test errors of ARB-2, ARB-1.5, RobustBoost and AdaBoost.}
\label{fig:GSE}
\end{figure}

To alleviate computational burden  we  choose 3000 probe-sets with the smallest p-values in the two-sample t-test and standardize each feature. Such simplification is  often considered in high dimensional data (e.g. \cite{ZWWL:14}). We randomly choose 50 samples with positive estrogen receptor status and 50 samples with negative estrogen receptor status for a training set and use the rest for the  testing set. We randomly flip labels of the samples in training set according to the preassigned noise level and repeat the analysis 100 times. Then a five-fold cross-validation is implemented on the training set to select the stopping time ($\le 100$) and step sizes. We summarize the results in Table \ref{tab:GSE20194} and Figure \ref{fig:GSE}. This dataset was previously analyzed in \cite{DM14} and \cite{ZWWL:14} where the best obtained test error was $15\%$ and $9\%$, respectively. However, our methods achieve  error comparable to those even when the labels were perturbed at random. This suggests that  our method is extremely stable even in high-dimensional models.


\subsection{Discussion}
We showed that ArchBoost -$\gamma$ is  a robust alternative to the popular Gradient Boost -type algorithms. 
 The algorithmic part, presented by Theorem1, works for quite
a very general class of    loss functions that  satisfy the Arch-Boost loss properties, presented in Definition 1. The differentiability condition is imposed artificially and we believe can be avoided by considering appropriate sub-differential analysis. However, the Condition (ii) is crucial for the analysis and  we believe that it cannot be relaxed. Moreover, the robustness properties depend crucially on this condition too.  Additionally,  the robustness part of the analysis, summarized in Theorems \ref{thm:breakdown} and  \ref{thm:influence}, works for quite for 
an arbitrary Lipschitz loss function. Hence, it presents novel proof of why is LogitBoost more robust than the AdaBoost, a folklore observation made by many experts in the field.   For example, Theorem \ref{thm:breakdown} is more likely to hold for LogitBoost than the AdaBoost and similarly more likely to hold for ArchBoost than the LogitBoost.

 Note that Arch Boost framework can be easily explored to define an estimate of the conditional probability $\PP (Y=1|x)$. A special case is to plug in the exponential loss function, in which case we will get Real AdaBoost algorithm, and it has conditional probability estimation $\hat P_{AdaBoost}(Y=1|X=x) = 1/(1+e^{-2 F_T(x)})$. In contrast, many of the existing boosting methods, based on the Gradient Boosting ideas, cannot be directly applied for this purpose.  In order to propose an estimator of $\PP (Y=1|x)$, we find and explore a recursive relationship between $F_t$ and $F_{t-1}$. We observe that 
 by rewriting equation \eqref{eq:9a}, we have
\begin{eqnarray*}
\frac{\phi^{'}(-F_{t}(x))}{\phi^{'}(F_{t}(x))}  = \frac{\phi^{'}(-F_{t-1}(x))}{\phi^{'}(F_{t-1}(x))} \left( \frac{\Prob_{w_t} (Y=1|x)}{\Prob_{w_t} (Y=-1|x)} \right).
\end{eqnarray*}
Suppose $F_0 \equiv 0$. By solving these equations recursively, 
 after $T$ iterations we obtain
\begin{eqnarray*}
\frac{\phi^{'}(-F_{T}(x))}{\phi^{'}(F_{T}(x))}  = \prod_{t=1}^T \frac{\Prob_{w_t} (Y=1|x)}{\Prob_{w_t} (Y=-1|x)}.
\end{eqnarray*}
Now define $\tilde\PP (Y=1|x)$ as
\begin{eqnarray} \label{eq:estP}
\tilde{\Prob}(Y=1|x)
= \frac{\prod_{t=1}^T \Prob_{w_t} (Y=1|x)}{\prod_{t=1}^T \Prob_{w_t} (Y=1|x) + \prod_{t=1}^T \Prob_{w_t} (Y=-1|x)} \in [0,1].
\end{eqnarray}
Then, we have the following relation
\begin{eqnarray} \label{eq:10}
\frac{\phi^{'}(-F_{T}(x))}{\phi^{'}(F_{T}(x))} = \frac{\tilde{\Prob}(Y=1|x)}{\tilde{\Prob} (Y=-1|x)}.
\end{eqnarray}
By comparing equation \eqref{eq:5} and \eqref{eq:10}, if $F_T(x) $ is close to  $F^*(x)$, then $\tilde{\Prob}(Y=1|x)$ will be a good approximation of the conditional probability $\PP (Y=1|x)$. Observe that due to a nature of weak classifiers, $\tilde{\Prob}(Y=1|x) $
is guaranteed to be bounded in $[0,1]$ for any differentiable loss function $\phi$. Moreover, non-convex loss functions $\phi$ seems to be better candidates for the class membership probability estimation as well as for the classifier estimation.

 The statistical consistency proof is centered around ``tilted'' loss functions that are  non-convex in particular.
We believe that non-convex losses have great and unexplored potential for robust high dimensional statistics. The framework of ``tilted'' loss functions is very general and can very well be explored for robust variable selection and estimation, through an appropriate penalization scheme.  
Moreover, it is very well known that the impact of outliers is multiplied in case of inferential problems, such are confidence intervals and testing. 
By screening out many large outliers, ``tilted'' losses may significantly improve upon asymptotic efficiency of existing procedures.


\begin{appendices}

\section{Proofs}

\subsection{Derivation of ARB-$\gamma$ algorithms} \label{A:1}
Note that $\phi_{a,\gamma}^{'}(v) = \frac{-a\gamma 2^\gamma e^{av}}{(1+e^{av})^{\gamma+1}}$, and $\frac{\phi_{a,\gamma}^{'}(-v)}{\phi_{a,\gamma}^{'}(v)} = e^{a(\gamma-1)v}$, and when $\gamma > 1$,

\begin{itemize}

\item [(i)]
for any data $(x,y)$, the optimal $F^*(x)$ satisfies
\begin{eqnarray*}
\frac{\phi_{a,\gamma}^{'}(-F^*(x))}{\phi_{a,\gamma}^{'}(F^*(x))} = e^{a(\gamma-1)F^*(x)} = \frac{\Prob(Y=1|x)}{\Prob(Y=-1|x)},
\end{eqnarray*}
that is,
\begin{eqnarray*}
F^*(x) = \frac{1}{a(\gamma-1)} \log \frac{\Prob(Y=1|x)}{\Prob(Y=-1|x)}, \; \gamma>1.
\end{eqnarray*}

\item [(ii)]
After iteration $t$, we will update the weights to be
\begin{eqnarray*}
w_{t+1}(x,y) = -\phi_{a,\gamma}^{'}(yF(x) + yh_t(x))
= \frac{a\gamma 2^\gamma e^{a(yF(x) + yh_t(x))}}{(1+e^{a(yF(x) + yh_t(x))})^{\gamma+1}}.
\end{eqnarray*}

Since the constant $a\gamma 2^\gamma$ will not influence the normalized weights, we can just update the weights to be
\begin{eqnarray} \label{eq:weight}
w_{t+1}(x,y) = \frac{e^{a(yF(x) + yh_t(x))}}{(1+e^{a(yF(x) + yh_t(x))})^{\gamma+1}}.
\end{eqnarray}

\item [(iii)]
At iteration $t$, we will update the hypothesis $F_{t-1}(x)$ to be $F_{t-1}(x) + h_t(x)$ such that
\begin{eqnarray*}
\frac{\phi_{a,\gamma}^{'} (F_{t-1}(x)) \phi_{a,\gamma}^{'}(-F_{t-1}(x)-h_t(x))}{\phi_{a,\gamma}^{'}(-F_{t-1}(x)) \phi_{a,\gamma}^{'} (F_{t-1}(x)+h_t(x))} = e^{a(\gamma-1)h_t(x)} 
= \frac{\Prob_w(Y=1|x)}{\Prob_w(Y=-1|x)},
\end{eqnarray*}
that is,
\begin{eqnarray*}
h_t(x) = \frac{1}{a(\gamma-1)} \log \frac{\Prob_w(Y=1|x)}{\Prob_w(Y=-1|x)}, \; \gamma > 1.
\end{eqnarray*}

\end{itemize}

The algorithms for different $\gamma$ (set $a=1$) is given in Algorithm \ref{alg:3}. Since $\frac{1}{\gamma-1}$ is just a constant, we can simply absorb it into $\alpha_t$ and leave $h_t = \log \frac{\Prob(Y=1|x)}{\Prob(Y=-1|x)}$. So intuitively, the larger $\gamma$ is, the smaller the step size $\alpha_t$ will be. If we use constant step size, then a rule of thumb is to set $\alpha_t = \frac{\alpha}{\gamma-1}$ for ARB-$\gamma$ where $\alpha$ is a tuning parameter for the step size of ARB-2.

\subsection{Proof of Lemma \ref{lemma:2}}
For any $\alpha \in (0,1)$, let $D(v) = \frac{d}{dv} [\alpha \phi(v) + (1-\alpha)\phi(-v)] = \alpha \phi^{'}(v) - (1-\alpha)\phi^{'}(-v)$. Since $\phi$ is convex, $\phi^{'}(v)$ is monotone increasing, and $\phi^{'}(-v)$ is monotone decreasing. So $D(\cdot)$ is monotone increasing.

Then note that there exists $v^*$ such that $\phi^{'}(v^*) \neq \phi^{'}(-v^*)$. Otherwise, $\phi^{'}$ will be both increasing and decreasing, that is, a constant. Contradiction to our assumption. Therefore, $D(v^*)D(-v^*) < 0$, and combined with the monotonicity and continuity of $D(\cdot)$, we know $D(v) = 0$ has one and only one solution, and it is a global minimum of $\phi$. \\

\subsection{Proof of Remark 2}
Note that $\phi_{a,1}(v) = 1-\tanh(\frac{1}{2}v)$ is a sigmoid function. Let $C_{\eta}(v) = \eta \phi(v) + (1-\eta)\phi(-v) = 1+ (1-2\eta)\tanh(\frac{1}{2}v)$, then $\inf_{v\in \realR} C_{\eta}(v) = 2 \min(\eta, 1-\eta) < 1 = \inf_{v:v(2\eta-1)\le 0} C_{\eta}(v)$ for all $\eta \neq \frac{1}{2}$. So condition (iii) is satisfied.

But since $C_{\eta}(v) = 1 + (1-2\eta)\tanh(v)$, we know if $\eta > \frac{1}{2}$, then $C_{\eta}(v)$ is monotone decreasing. And when $\eta < \frac{1}{2}$, $C_{\eta}(v)$ is monotone increasing. When $\eta = \frac{1}{2}$, $C_{\frac{1}{2}}(v) \equiv 1$. Hence, for every $\eta \in [0,1]$, there is no unique global minimum in $\realR$.\\

\subsection{Proof of Lemma \ref{lem:iii}}
Define $f(v) = a\phi(v) + (1-a)\phi(-v)$, then $f^{'}(v) = a\phi^{'}(v) - (1-a)\phi^{'}(v)$ and $f^{''}(v) = a\phi^{''}(v) + (1-a)\phi^{''}(-v)$. For any $a \in (\frac{1}{2},1)$, $f^{'}(v) = 0$ if and only if $g(v) = \frac{1-a}{a} \in (0,1)$. Since $g:(0,\infty) \to (0,1)$ is a bijection, there exists one and only one $v^* \in \mathbb{R}^*$ such that $v^* = g^{-1}(\frac{1-a}{a})$. Note that if $v > v^*$, then $g(v) < g(v^*)$, that is $f^{'}(v) > f^{'}(v^*)$; if $v < v^*$, then similarly we have $f^{'}(v) < f^{'}(v^*)$. So $v^*$ is a minimum for $\phi$.

For $a \in (0,\frac{1}{2})$, since $g(-v) = \frac{1}{g(v)}$, we only need to solve $g(v) = \frac{a}{1-a}$ for $v^*$, then $-v^*$ will be solution for $f^{'}(v) = 0$ for this $a$. And the minimum claim follows similarly as above.

\subsection{Solution of \eqref{eq:14}}
Here, we show the derivation of one possible family of solutions. We do so by employing integrating factor method and adapting it to the nonlinear ordinary differential equation \eqref{eq:14}.
 Since we have $\frac{\phi^{'}(-v)}{\phi^{'}(v)} = e^{(\gamma-1)v}$, we made a reasonable guess that $\phi(v) = f(e^v)$. After plugging this into \eqref{eq:14} and let $x = e^v$, we have
\begin{eqnarray} \label{eq:lem3}
\frac{f^{'}(1/x)}{f^{'}(x)} = x^{\gamma+1}.
\end{eqnarray}
Furthermore, for a non-convex function $\phi$ satisfying Assumption 2, we know $\phi^{'}(v) \to 0$ as $v \to \pm \infty$, thereafter $x f^{'}(x) \to 0$ as $x \to \infty$, that is, $f^{'}(x) = o(\frac{1}{x})$. By rewriting equation \eqref{eq:lem3}, we have
\begin{eqnarray} \label{eq:lem3-2}
\left(\frac{1}{x}\right)^{\frac{\gamma+1}{2}} f^{'}(\frac{1}{x}) = x^{\frac{\gamma+1}{2}} f^{'}(x).
\end{eqnarray}
Let $G(x):= x^{\frac{\gamma+1}{2}}f^{'}(x)$, from \eqref{eq:lem3-2}, we have $G(\frac{1}{x}) = G(x)$ and $\lim_{x \to \infty} G(x) = \lim_{x \to \infty} G(\frac{1}{x}) = \lim_{x\to \infty} \left(\frac{1}{x}\right)^{\frac{\gamma+1}{2}} f^{'}(\frac{1}{x}) = 0$ provided $f^{'}(0) < \infty$ and $f^{'}(x)$ is continuous at $0$. One such choice is 
$$G(x) = \frac{x}{(1+x)^2}.$$
Then by integrating $f^{'}(x)$ and substitution of parameter, we get one solution to \eqref{eq:14} is
$$\phi(v) = \frac{c}{(1+e^v)^{\gamma}},$$
for any positive constant $c$.
The numerator $2^{\gamma}$ in \eqref{eq:15} is just introduced to make the function $\phi$ an upper bound of the $0-1$ loss and $\phi(0) = 1$.

\subsection{Proof of Lemma \ref{lemma:3.3}}
 
(i) and (ii) are easy to verify.

For (iii), given any $\alpha \in (0,1)$, let $D(v) = \frac{\alpha 2^\gamma}{(1+e^{av})^\gamma} + \frac{(1-\alpha)2^\gamma}{(1+e^{-av})^\gamma}$. Let $D^{'}(v) = 0$, we have the only solution $v^*=\frac{1}{a(\gamma-1)} \log \frac{\alpha}{1-\alpha}$. Note that $D^{'}(v) =  (-\gamma a)\left( \frac{\alpha 2^\gamma}{(1+e^{av})^\gamma} \frac{1}{1+e^{-av}}+ \frac{(1-\alpha)2^\gamma}{(1+e^{-av})^\gamma} \frac{-1}{1+e^{av}} \right) = (-\gamma a)\left(\alpha \phi_{a,\gamma}(v) \frac{1}{1+e^{-av}}+(1-\alpha) \phi_{a,\gamma}(-v) \frac{-1}{1+e^{av}} \right) $. Since $\phi_{a,\gamma}$ is decreasing, when $v > v^*$, we have $\left(\alpha \phi_{a,\gamma}(v) \frac{1}{1+e^{-av}}+(1-\alpha) \phi_{a,\gamma}(-v) \frac{-1}{1+e^{av}} \right) < 0$, and hence $D^{'}(v) > 0$. Similarly, when $v < v^*$, $D^{'}(v) < 0$. Therefore, $v^*$ is indeed a global minimum point.

For (iv), let $C_{\eta}(\alpha) = 2^{\gamma}\left( \eta \frac{1}{(1+e^{\alpha})^{\gamma}} + (1-\eta)\frac{1}{(1+e^{-\alpha})^{\gamma}} \right)$. Then by setting $C_{\eta}^{'}(\alpha) = 0$, we get the minimum point $\alpha^* = \frac{1}{\gamma-1} \log \frac{\eta}{1-\eta}$, and $C_{\eta}(\alpha^*) = \eta\left( \frac{2}{(1+(\frac{\eta}{1-\eta})^{\gamma}} \right)^{\gamma} + (1-\eta)\left( \frac{2}{(1+(\frac{1-\eta}{\eta})^{\gamma}} \right)^{\gamma}$ which can be shown to attain the global maximum when $\gamma = \frac{1}{2}$, and $C_{1/2}(\alpha^*) = 1$. We also have $\inf_{\alpha:\alpha(2\eta-1)\le 0} C_{\eta}(\alpha) = C_{\eta}(0) = 1 > C_{\eta}(\alpha^*)$ when $\eta \neq \frac{1}{2}$. By \cite{Bartlett:06}, we have (iv) holds.

\subsection{Proof of Theorem \ref{numerical convergence}}
 
The proof is completed by showing that  at each iteration $t$, as long as  the empirical margin $\hat{\mu}(w_t,h_t)$ is positive,   the empirical risk  decreases by adding the weak hypotheses $h_t$ to the current estimate. Then, we show that the weak hypothesis returned by our Arch Boost algorithm,  always has a positive empirical margin before convergence.
\begin{enumerate} [(i)]
\item
On the sample $\mathcal{S}_n = \{(X_1,Y_1), \cdots, (X_n,Y_n)\}$, at each iteration $t$, denote
\begin{eqnarray*}
{ \mathbf{F}}_{t-1} = \left( F_{t-1}(x_1), \cdots, F_{t-1}(x_n) \right).
\end{eqnarray*}
Recall that the empirical risk $\hat{R}_{\phi,n}(F) = \frac{1}{n} \sum_{i=1}^n \phi(Y_i F(X_i))$ and it can be viewed as a multivariate function of ${ \mathbf{F}} = (F(X_1), \cdots, F(X_n))$. Denote the partial derivative w.r.t. $F(X_i)$ at iteration $t$ as
$$g_t(X_i) = \left[\frac{\partial \hat{R}_{\phi,n}({ \mathbf{F}})}{\partial F(X_i)}\right]_{F(X_i) = F_{t-1}(X_i)}
= \frac{1}{n} Y_i \phi^{'}(Y_i F_{t-1}(X_i)).$$

Then the gradient of $\hat{R}_{\phi,n}$ at ${ \mathbf{F}}_{t-1}$ is
\begin{eqnarray*}
\nabla \hat{R}_{\phi,n}({ \mathbf{F}}_{t-1}) = \frac{1}{n}
\begin{pmatrix}
g_t(X_1) \\
\vdots \\
g_t(X_n)
\end{pmatrix}.
\end{eqnarray*}

At iteration $t$, the weight on $(X_i,Y_i)$ is updated to be
\begin{eqnarray*}
w_t(X_i,Y_i) = -\phi^{'}(Y_iF_{t-1}(X_i)), \; i=1, \cdots, n.
\end{eqnarray*}
Suppose we choose a weak hypothesis $h_t$ with positive empirical margin w.r.t. weights $w_t$, that is, $\hat \mu(h_t, w_t) > 0$, and denote $ {\mathbf h}_t = (h_t(X_1), \cdots, h_t(X_n))$. Note that

\begin{eqnarray*}
\langle -\nabla \hat{R}_{\phi,n}({ \mathbf{F}}_{t-1}),  {\mathbf h}_t \rangle 
&=& \frac{1}{n} \sum_{i=1}^n Y_i h_t(X_i) (-\phi^{'}(Y_iF_{t-1}(X_i))) \\
&=& \frac{1}{n} \sum_{i=1}^n Y_i h_t(X_i) w_t(X_i,Y_i) \\
&=& \hat{\mu}(h_t,w_t) > 0,
\end{eqnarray*}
where $\langle \cdot, \cdot \rangle$ is the standard inner product in $\mathbb{R}^n$. Therefore, we know
\begin{eqnarray*}
\langle -\nabla \hat{R}_{\phi,n}({ \mathbf{F}}_{t-1}),  {\mathbf h}_t \rangle >0 \Longleftrightarrow \hat\mu(w_t,h_t) > 0.
\end{eqnarray*}

But if $\langle -\nabla \hat{R}_{\phi,n}({ \mathbf{F}}_{t-1}),  {\mathbf h}_t \rangle >0$, then
$\begin{pmatrix}
h_t(X_1) \\
\vdots \\
h_t(X_n)
\end{pmatrix}$
is a descending direction of $\hat R_{\phi,n}({ \mathbf{F}})$ at ${ \mathbf{F}}_{t-1}$, therefore
\begin{eqnarray*}
\hat R_{\phi,n} \left[
\begin{pmatrix}
F_{t}(X_1) \\
\vdots \\
F_{t}(X_n)
\end{pmatrix}
\right]
=
\hat R_{\phi,n} \left[
\begin{pmatrix}
F_{t-1}(X_1) \\
\vdots \\
F_{t-1}(X_n)
\end{pmatrix}
+ \alpha_t
\begin{pmatrix}
h_t(X_1) \\
\vdots \\
h_t(X_n)
\end{pmatrix}
\right]
<
\hat R_{\phi,n} \left[
\begin{pmatrix}
F_{t-1}(X_1) \\
\vdots \\
F_{t-1}(X_n)
\end{pmatrix}
\right]
\end{eqnarray*}
with an appropriate step size $\alpha_t$ which can be found by line search
$$ \alpha_t = \argmin_{\alpha}
\hat R_{\phi,n} \left[
\begin{pmatrix}
F_{t-1}(X_1) \\
\vdots \\
F_{t-1}(X_n)
\end{pmatrix}
+ \alpha
\begin{pmatrix}
h_t(X_1) \\
\vdots \\
h_t(X_n)
\end{pmatrix}
\right].
$$

In summary, we have
\begin{eqnarray} \label{eq:emp}
\hat R_{\phi,n}(F_t) < \hat R_{\phi,n}(F_{t-1})
\end{eqnarray}
if at step $t$, we choose a base learner $h_t$ such that $\hat \mu(w_t,h_t) > 0$ and choose a suitable step size $\alpha_t$ either by line search or set to be appropriately small. Therefore, $\hat{R}_{\phi,n}$ will converge in $\mathbb{R}$.

\item
In any region $R_t^j$, we know $h_t \equiv \gamma_t^j$. Then $\begin{pmatrix}
-g_t(X_1) \\
\vdots \\
-g_t(X_N)
\end{pmatrix}
\begin{pmatrix}
h_t(X_1) \\
\vdots \\
h_t(X_N)
\end{pmatrix}
$ is equal to 
\begin{eqnarray*}
  && \sum_{j=1}^{J_t} \sum_{i\in R_t^j} Y_i w_t(X_i,Y_i) \gamma_t^j \\
&=& \sum_{j=1}^{J_t} \gamma_t^j  \left( \mathbb{P}_{w_t}(Y=1|X\in R_t^j) - \mathbb{P}_{w_t}(Y=-1|X\in R_t^j)\right)\sum_{i\in R_t^j} w_t(X_i,Y_i) \\
&=& \sum_{j=1}^{J_t} \theta(\mathbb{P}_{w_t}(Y=1|X\in R_t^j))  \left( 2 \mathbb{P}_{w_t}(Y=1|X\in R_t^j) - 1 \right)\sum_{i\in R_t^j} w_t(X_i,Y_i) \\
&\stackrel{(i)}{\ge}& 0.
\end{eqnarray*}
The last inequality $(i)$ is because $\theta(\mathbb{P}_{w_t}(Y=1|X\in R_t^j))$ is strictly increasing and has the only root at $\frac{1}{2}$, and hence always has the same sign as $2 \mathbb{P}_{w_t}(Y=1|X\in R_t^j) - 1$,    and ``='' holds if and only if
$\mathbb{P}_{w_t}(Y=1|X\in R_t^j) = \frac{1}{2}$ for all $j = 1, \cdots, J_t$.

\item
From \eqref{eq:9a}, we have
\begin{eqnarray*}
\frac{\phi^{'}(-F_t(x))}{\phi^{'}(F_t(x))} = \frac{\Prob_{w_t}(Y=1|x)}{\Prob_{w_t}(Y=-1|x)} \frac{\phi^{'}(-F_{t-1}(x))}{\phi^{'}(F_{t-1}(x))}.
\end{eqnarray*}
If $\Prob_{w_t}(Y=1|x) > \Prob_{w_t}(Y=-1|x)$, then $\frac{\phi^{'}(-F_t(x))}{\phi^{'}(F_t(x))} > \frac{\phi^{'}(-F_{t-1}(x))}{\phi^{'}(F_{t-1}(x))}$. By Lemma \ref{lem:iii}, $F_t(x) > F_{t-1}(x)$, that is, $h_t(x) > 0$.

\item
Here, we develop   ideas much similar to the proof of Lemma 4.1 and Lemma 4.2 in \cite{ZY05}. There are two differences  here in comparison to \cite{ZY05}. First, the loss is  non-convex  function  and second,  the optimal hypothesis is chosen differently. For $f_1, f_2 \in \cup_{T=1}^{\infty} \mathcal{F}^T$, let $H_f \subset \mathcal{H}$ be the set that contains all weak hypotheses in $f_1$ and $f_2$. For example, $f_1 = \sum_{h \in H_f} \alpha_1^{h} h$ and $f_2 = \sum_{h \in H_f} \alpha_2^{h} h$. Then denote 
$$ ||f_1 - f_2 ||_2^2 := \sum_{h \in H_f} ( \alpha_1^{(h)} - \alpha_2^{(h)})^2 \le \frac{1}{|H_f|} \left( \sum_{h \in H_f} | \alpha_1^{(h)} - \alpha_2^{(h)}| \right)^2 =: \frac{1}{|H_f|} ||f_1-f_2||_1^2.$$ Now let $\bar{f}_t$ be any reference function in $\cup_{T=1}^{\infty} \mathcal{F}^T$ satisfying 
$$\hat{R}_{\phi,n}(\bar{f}_t) < \inf_{f \in \cup_{T=1}^{\infty} \mathcal{F}^T} \hat{R}_{\phi,n} (f) + \frac{1}{t}$$ and $F_t$ be the classifier returned by Arch Boost at step $t$. Moreover, denote 
$$\bar{f}_t = \sum_{h \in H_t} \omega_t^{h} h, \qquad F_t = \sum_{h \in H_t} \alpha_t^{h} h.$$
 For notation simplicity, we denote $R = \hat{R}_{\phi,n}$ since we have fixed a loss function $\phi$ and sample size $n$. Let $s^h = sign(\omega_t^h - \alpha_t^h)$. By Taylor expansion, we have
$$
R(F_t + \alpha_{t+1} s^h h) \le R(F_t) + \alpha_{t+1} s^h \langle \nabla R(F_t), h \rangle + \frac{\alpha_{t+1}^2}{2} \sup_{\xi \in [0,1]} R^{''}_{F_t, h}(\xi \alpha_{t+1} s^h),
$$
where $R_{F_t,h}(\alpha) := R(F_t + \alpha h)$. Since $\theta$ is bounded, from part (ii) we know there exists $M > 0$ s.t. $ \sup_{\xi \in [0,1]} R^{''}_{F_t, h}(\xi \alpha_{t+1} s^h) < M$ if we choose $h$ by \eqref{eq:9a}. Therefore,
$$
R(F_t + \alpha_{t+1} s^h h) \le R(F_t) + \alpha_{t+1} s^h \langle \nabla R(F_t), h \rangle + \frac{\alpha_{t+1}^2}{2} M.
$$

By Algorithm \ref{boosting} we know that $R (F_{t+1})= R(F_t + \alpha_{t+1} h_{t+1})$. Moreover, by \eqref{eq:9a}, $h_{t+1}$ is chosen as the $\arg \min_{h \in \mathcal{H}_t} \EE_w \left[R(F_{t} + \alpha_{t+1} h)\right]$. Hence,  for any $h \in \mathcal{H}_t$, $\EE_w \left[R(F_{t} + \alpha_{t+1} h_{t+1})\right] \leq \EE_w \left[R(F_{t} + \alpha_{t+1} h)\right]$. 
Moreover,  for any bounded random variable $Z$, $\left| \EE_w[Z] - \EE [Z]\right| \leq K$ for a positive constant $K$.
Combining the above, we have  \[
R (F_{t+1})  \leq  R(F_{t} + \alpha_{t+1} s^h h) + 2 \epsilon_t + 2K,
\]
for 
\[
\epsilon_t =\sup_{h \in \mathcal{H}_t} \biggl|R(F_{t} + \alpha_{t+1} s^h h) - \EE \left[ R(F_{t} + \alpha_{t+1}s^h h)\right] \biggl|.
\]
By the arguments very much similar to Lemmas \ref{lem:6} and \ref{lem:7}, it  easy to obtain $\epsilon_t = o_P(1)$.

%

 Since $||\bar{f}_t - F_t||_1 = o(\log t)$, and $||\bar{f}_t - F_T||_2^2 \le \frac{||\bar{f}_t - F_t||_1^2}{t^{c_t}}$ where $c_t \in (0,1)$ and $c_t \to 0$ as $t \to \infty$, we have $\frac{||\bar{f}_t - F_t||_1^2}{t^{c_t}} = o(\frac{\log t}{t^{c_t}}||\bar{f}_t - F_t||_1)$. Hence,

\begin{eqnarray} \label{eq:W}
&&||\bar{f}_t - F_t||_2^2 (R(F_{t+1}) - 2\epsilon_t - 2K) \nonumber \\
&=& o\left[\frac{\log t}{t^{c_t}}\sum_{h \in H_t} |\alpha_t^h - \omega_t^h| R(F_t + \alpha_{t+1}s^h h)\right] \nonumber \\
&=& o\left[ \frac{\log t}{t^{c_t}} \sum_{h \in H_t} |\alpha_t^h - \omega_t^h| \left( R(F_t) +  \alpha_{t+1}s^h \langle \nabla R(F_t), h \rangle + \frac{\alpha_{t+1}^2}{2} M \right) \right] \nonumber \\
&=& o\left[ \frac{\log t}{t^{c_t}} ||\bar{f}_t - F_t||_1 R(F_t) + \frac{\alpha_{t+1}\log t}{t^{c_t}} \langle \nabla R(F_t), \bar{f}_t - F_t \rangle + \frac{M\alpha_{t+1}^2 \log t}{2 t^{c_t}} ||\bar{f}_t - F_t||_1 \right]
\end{eqnarray}

Now we look at the situation when $\hat{\mu}(h_k,w_k) = 0$. From part (ii), we know this happens if and only if $\PP_{w_k}(Y=1| X \in R_k^j) = \frac{1}{2}$ in every region $j$. In another word, $\nabla R(F_k) \perp \mathcal{H}$. Now since $\hat{\mu}(h_t, w_t) \to 0$, $\nabla R(F_t)$ is more and more perpendicular to $\mathcal{H}$ and hence perpendicular to $\cup_{T=1}^{\infty} \mathcal{F}^T$, and $\langle \nabla R(F_t) - \nabla R(\bar{f}_t), \bar{f}_t - F_t \rangle \to 0$ since $\bar{f}_t - F_t \in \cup_{T=1}^{\infty} \mathcal{F}^T$.

Since $\phi$ is Lipschitz differentiable, we know there exists $L > 0$ ($L$ is the Lipschitz constant s.t. $||\nabla R(f_1) - \nabla R(f_2)||_2 \le L ||f_1 - f_2 ||_2$ for all $f_1$ and $f_2$) s.t.
\begin{eqnarray*}
R(F_t) - R(\bar{f}_t) \le \langle \nabla R(\bar{f}_t), F_t - \bar{f}_t \rangle + \frac{L}{2} || \bar{f}_t - F_t ||_2^2.
\end{eqnarray*}
Then $\langle \nabla R(\bar{f}_t), \bar{f}_t - F_t \rangle \le R(\bar{f}_t) - R(F_t) + \frac{L}{2} || \bar{f}_t - F_t ||_2^2$. When $t$ is large enough, we know there exists sequence $\tilde \epsilon_t \to 0$ s.t. 
$$\langle \nabla R(F_t), \bar{f}_t - F_t \rangle \le R(\bar{f}_t) - R(F_t) + \frac{L}{2} || \bar{f}_t - F_t ||_2^2 + \tilde \epsilon_t.$$ Then by \eqref{eq:W},
\begin{eqnarray} \label{eq:eta}
&&||\bar{f}_t - F_t||_2^2 (R(F_{t+1}) - 2\epsilon_t -2K) \nonumber \\
&=& o\left[\frac{\log t}{t^{c_t}} ||\bar{f}_t - F_t||_1 R(F_t) + \frac{\alpha_{t+1}\log t}{t^{c_t}} \langle \nabla R(F_t), \bar{f}_t - F_t \rangle + \frac{\alpha_{t+1}^2 \log t}{2 t^{c_t}} ||\bar{f}_t - F_t||_1 M \right] \nonumber \\
&=& o\left[\frac{\log t}{t^{c_t}} ||\bar{f}_t - F_t||_1 R(F_t) + \frac{\alpha_{t+1}\log t}{t^{c_t}} \left( R(\bar{f}_t) - R(F_t) + \frac{L}{2} || \bar{f}_t - F_t ||_2^2 + \tilde \epsilon_t \right) + \frac{\alpha_{t+1}^2 \log t}{2 t^{c_t}} ||\bar{f}_t - F_t||_1 M \right]\nonumber \\
&=& o\left[ \frac{\log t}{t^{c_t}} ||\bar{f}_t - F_t||_1 R(F_t) + \frac{\alpha_{t+1}\log t}{t^{c_t}} \left( R(\bar{f}_t) - R(F_t) \right) + \eta_t \right],
\end{eqnarray}
where $\eta_t :=  \frac{\alpha_{t+1}\log t}{t^{c_t}} \left( \frac{L}{2} || \bar{f}_t - F_t ||_2^2 + \tilde \epsilon_t \right) + \frac{\alpha_{t+1}^2\log t}{2 t^{c_t}} ||\bar{f}_t - F_t||_1 M$.

Then by dividing $||\bar{f}_t - F_t||_2^2$ on both sides of \eqref{eq:eta}, we get
\begin{eqnarray*}
R(F_{t+1}) &=& o\left[ \frac{\log t}{t^{c_t}} \frac{||\bar{f}_t - F_t||_1}{||\bar{f}_t - F_t||_2^2} R(F_t) + \frac{\alpha_{t+1}\log t}{t^{c_t} ||\bar{f}_t - F_t||_2^2} \left( R(\bar{f}_t) - R(F_t) \right) + \bar{\eta}_t +2 \epsilon_t +2K \right] \\
&=& o\left[ \frac{\log t}{t^{c_t/2}||\bar{f}_t - F_t||_2} R(F_t) + \frac{\alpha_{t+1}\log t}{t^{c_t} ||\bar{f}_t - F_t||_2^2} \left( R(\bar{f}_t) - R(F_t) \right) + \bar{\eta}_t + 2\epsilon_t + 2K \right],
\end{eqnarray*}
where $\bar{\eta}_t :=  \frac{\alpha_{t+1} \log t}{t^{c^t}} \left( \frac{L}{2} + \frac{\tilde \epsilon_t}{|| \bar{f}_t - F_t ||_2^2} \right) + \frac{\alpha_{t+1}^2 \log t}{2t^{c_t/2}||\bar{f}_t - F_t||_2} M$.

Therefore,
\begin{eqnarray*}
R(F_{t+1}) - R(\bar{f}_t) &=& o\left[ \frac{\log t}{t^{c_t}||\bar{f}_t - F_t||_2} R(F_t) + \frac{\alpha_{t+1}\log t}{t^{c_t} ||\bar{f}_t - F_t||_2^2} \left( R(\bar{f}_t) - R(F_t) \right) + \bar{\eta}_t +2 \epsilon_t + 2K \right] \\
&\le& \frac{\xi_t \log t}{t^{c_t}||\bar{f}_t - F_t||_2} R(F_t) + \frac{\alpha_{t+1}\xi_t\log t}{t^{c_t} ||\bar{f}_t - F_t||_2^2} \left( R(\bar{f}_t) - R(F_t) \right) + \xi_t\bar{\eta}_t +2 \xi_t\epsilon_t + 2 K\xi_t,
\end{eqnarray*}
for some sequence $\xi_t \to 0$ as $t \to \infty$.
Now if we assume $c_t \to 0$ slowly enough, then by choosing $\alpha_t$ s.t. $\sum_{t=1}^{\infty} \alpha_t = \infty$, $\sum_{t=1}^{\infty} \alpha_t^2 < \infty$ and $\sum_{t=1}^{\infty} \frac{\alpha_{t+1}\xi_t \log t }{ t^{c_t}} < \infty$, and by Lemma 4.2 in \cite{ZY05}, we have $R(F_{t+1}) - R(\bar{f}_t) \to 0$ as $t \to \infty$, and hence $R(F_t) \to R^*_{\phi,n}$ as $t \to \infty$. 
\end{enumerate}

\subsection{Proof of Theorem \ref{thm:breakdown}}

The proof of Theorem consists of a careful decomposition of the inner product between the gradient vector and the weak hypothesis obtained on the complete data. The decomposition is done in such a manner that one of the factors is the inner product between the gradient computed on the noise-free data and the weak hypothesis. 
Then, the proof is completed by showing that the signs of the two inner products above match.
 
On the original dataset $\mathcal{S}$, suppose after some iteration we obtain a weak hypothesis $h$ such that $h(x) = \theta(\Prob_{w}(Y=1|x))$ where $\theta$ is defined in Theorem \ref{numerical convergence}. In the rest of the proof we will exploit the decomposition  proved in Theorem \ref{numerical convergence},
\begin{eqnarray*}
- \langle  {\mathbf {g}_o} ,  {\mathbf h} \rangle &=& \sum_{j=1}^{J} \theta(\mathbb{P}_{w}(Y=1|X\in R^j))  \left( 2 \mathbb{P}_{w}(Y=1|X\in R^j) - 1 \right)\sum_{i\in R^j} w(x_i,y_i) \\
&=& \sum_{j=1}^{J} \theta(p_j)  \left( 2 p_j - 1 \right)\sum_{i\in R^j} w(x_i,y_i)
\end{eqnarray*}
where $R^j, j=1, \cdots, J$ are the regions corresponding to terminal nodes, and as long as $p_j = \mathbb{P}_{w}(Y=1|X\in R^j)$ is not $\frac{1}{2}$ in some region, $- \langle  {\mathbf {g}} ,  {\mathbf h} \rangle > 0$, that is, $ {\mathbf h}$ points to the descending direction of the empirical risk.

For $ {\mathbf {g}_o}$, we get $- \langle  {\mathbf {g}_o} ,  {\mathbf h} \rangle= \sum_{j=1}^{J} \theta(p_j)  \left( 2 \tilde{p_j} - 1 \right)\sum_{i\in R^j\setminus \mathcal{O}} w(x_i,y_i)$ with a different conditional probability estimation $\tilde{p_j}$. Recall that
\begin{eqnarray*}
2p_j-1 = \Prob_{w}(Y=1|X\in R^j) - \Prob_{w}(Y=-1|X\in R^j) = \frac{\sum_{i \in R^j} y_i w(x_i,y_i)}{\sum_{i \in R^j} w(x_i,y_i)},
\end{eqnarray*}
\begin{eqnarray*}
2\tilde{p_j}-1 = \frac{\sum_{i \in R^j\setminus \mathcal{O}} y_i w(x_i,y_i)}{\sum_{i \in R^j \setminus \mathcal{O}} w(x_i,y_i)}.
\end{eqnarray*}
Then we have that $\left( 2 \tilde{p_j} - 1 \right)\sum_{i\in R^j\setminus \mathcal{O}} w(x_i,y_i)
=  \sum_{i \in R^j\setminus \mathcal{O}} y_i w(x_i,y_i) = (2p_j-1)\sum_{i\in R^j} w(x_i,y_i) - \sum_{i\in \mathcal{O} \cap R^j} y_i w(x_i,y_i)$. Therefore,
\begin{eqnarray*}\label{rob:35}
- \langle  {\mathbf {g}_o} ,  {\mathbf h} \rangle &=&- \langle  {\mathbf {g}} ,  {\mathbf h} \rangle - \sum_{j=1}^J \theta(p_j) \sum_{i\in \mathcal{O} \cap R^j} y_i w(x_i,y_i) \nonumber \\
&=& \sum_{j=1}^J \theta(p_j) \left[ (2p_j-1) \sum_{i\in R^j} w(x_i,y_i) - \sum_{i \in \mathcal{O} \cap R^j} y_i w(x_i,y_i) \right].
\end{eqnarray*}
From previous equation, a sufficient condition for $- \langle  {\mathbf {g}_o} ,  {\mathbf h} \rangle \ge 0$ is that each summand reminds non-negative.
Since $\theta(p_j) \ge 0$ if and only if $p_j \ge \frac{1}{2}$,  the sufficient condition becomes equivalent to
$$
(p_j - \frac{1}{2}) \left[ (2p_j-1) \sum_{i\in R^j} w(x_i,y_i) - \sum_{i \in \mathcal{O} \cap R^j} y_i w(x_i,y_i) \right] \ge 0.
$$
Furthermore,  this inequality can be reformulated as
\begin{eqnarray*} \label{rob:38}
2(p_j-\frac{1}{2})^2 \sum_{i \in R^j \setminus \mathcal{O}} w(x_i,y_i) \ge (p_j-\frac{1}{2}) \sum_{i \in \mathcal{O}\cap R^j} (1+y_i -2p_j) w(x_i,y_i).
\end{eqnarray*}
When $p_j \neq \frac{1}{2}$, a sufficient condition for the inequality above  is
\begin{eqnarray*} \label{rob:39}
2|p_j-\frac{1}{2}| \sum_{i \in R^j \setminus \mathcal{O}} w(x_i,y_i) &\ge&  \sum_{i \in \mathcal{O}\cap R^j}  \max_{y_i = \pm 1} \left[ sign(p_j-\frac{1}{2})(1+y_i -2p_j) \right] w(x_i,y_i) \nonumber \\
&=& \sum_{i \in \mathcal{O}\cap R^j} 2\min (p_j,1-p_j) w(x_i,y_i).
\end{eqnarray*}

\subsection{Proof of Lemma \ref{lemma:hessian}}
 
Since $\phi$ is an Arch boosting loss function, we know for any $p \in (0,1)$, $p \phi(v) + (1-p)\phi(-v)$ has only one critical point $v^*$ that is the global minimum. Hence, $p_x \phi^{''}(F^*(x)) + (1-p_x) \phi^{''}(-F^*(x)) \ge 0$ where $p_x = \PP(Y=1|X=x)$. Note that here we treat $Yf(X)$ as the input of $\phi$, and by the chain rule, we know $p_x \phi^{''}(1,F^*(X)) + (1-p_x) \phi^{''}(-1,F^*(X)) \ge 0$, where the derivative now is w.r.t. the second argument. Then $\E_{\PP}\left[ \phi^{''}(Y,F^*(X))q^2(X) | X = x\right] = q^2(x)[\PP(Y=1|X=x)\phi^{''}(1,F^*(x)) + \PP(Y=-1|X=x)\phi^{''}(-1,F^*(x))] \ge 0$ since $\PP(Y=1|X=x)\phi^{''}(1,F^*(x)) + \PP(Y=-1|X=x)\phi^{''}(-1,F^*(x)) \ge 0$ for all $\PP(Y=1|X=x) \in (0,1)$ and is also nonnegative at the end points $\{0,1\}$ because of the continuity of $\phi^{''}$.

Now if $\PP$ and $\mathcal{X}$ satisfy that $\PP(Y=1|X=x) = p_x \in [\delta,1-\delta]$ for all $x \in \mathcal{X}$ for some $\delta \in (0,\frac{1}{2})$, and for all $p \in [\delta,\1-\delta]$, $p\phi^{''}(v^*) + (1-p)\phi^{''}(-v^*) > 0$ , that is, it is locally convex near $v^*$, then by the continuity of $\phi^{''}$, for each $p\in [\delta,1-\delta]$, there exists $r_p > 0$ such that $p\phi^{''}(v) + (1-p)\phi^{''}(-v) \ge 0$ for all $v \in \mathbb{R}$ such that $|v - v^*| < r_p$. But since $[\delta,1-\delta]$ is compact, there exists $r > 0$ such that $p\phi^{''}(v) + (1-p)\phi^{''}(-v) \ge 0$ for all $|v - v^*| < r$ and for all $p \in [\delta,1-\delta]$.

Now for any $x \in \mathcal{X}$, we know the corresponding $p_x \in [\delta,1-\delta]$, and for this $p_x$, we have $p_x\phi^{''}(G(x)) + (1-p_x)\phi^{''}(-G(x)) \ge 0$ for all measurable function $G$ with $|G(x) - F^*(x)| < r$. Therefore, if we take measurable function $G$ s.t. $||G - F^*||_{\infty} < r$, then $p_x\phi^{''}(G(x)) + (1-p_x)\phi^{''}(-G(x)) \ge 0$ for all $x \in \mathcal{X}$, that is, $\E_{\PP} \left[ \phi^{''}(Y,G(X))q^2(X)|X=x \right] \ge 0$ for all $x \in \mathcal{X}$. Taking expectation w.r.t. $X$, we get $\E_{\PP} \left[ \phi^{''}(Y,G(X))q^2(X)\right] \ge 0$.

\subsection{Proof of Theorem \ref{thm:influence}}
 
By Lemma \ref{lemma:hessian}, there exists $r > 0$ such that if $||f_{\PP,\lambda} - F^*|| < r$, then $\E_{\Prob} \phi^{''}(Y,f_{\PP,\lambda}(X))q^2(X) \ge 0$ for any measurable function $q$. Therefore one can show that the influence function still exists in form of \eqref{eq:17} by the Theorem 4 in \cite{Christmann:04} since in that Theorem, the only place one needs convexity is to show $\E_{\PP} \phi^{''}(Y,f_{\PP,\lambda}(X))q^2(X) \ge 0$.
Rewriting \eqref{eq:17}, and let $IF(z;T,\Prob) = g_z \in H$, we have
\begin{eqnarray*}
2\lambda g_z + \E_{\Prob} \phi^{''}(Y,f_{\PP,\lambda}(X)) g_z(X) \Psi(X) = \E_{\Prob} \phi^{'}(Y,f_{\PP,\lambda}(X)) \Psi(X) - \phi^{'}(z_y,f_{\PP,\lambda}(z_x)) \Psi(z_x).
\end{eqnarray*}
By taking inner product $\langle \cdot, \cdot \rangle_H$ with $g_z$ itself, we have
\begin{eqnarray} \label{eq:A29}
2\lambda ||g_z||_H^2 + \E_{\Prob} \phi^{''}(Y,f_{\PP,\lambda}(X)) g_z^2(X) = \E_{\Prob} \phi^{'}(Y,f_{\PP,\lambda}(X)) g_z(X) - \phi^{'}(z_y,f_{\PP,\lambda}(z_x)) g_z(z_x).
\end{eqnarray}
Assume $\lambda ||f||^2_{H} + R_{\phi}(f)$ has only one global minimum $f_{\PP,\lambda}$ in $H$, then at the minimum, we know the Frechet derivative at $f_{\PP,\lambda}$ is a zero mapping, that is,
\begin{eqnarray*}
2\lambda \langle f_{\PP,\lambda}, \cdot \rangle_{H} + \E_{\PP} \phi^{'}(Y,f_{\PP,\lambda}(X)) \Psi(X) = \b{0},
\end{eqnarray*}
where $\b{0}: H \to \mathbb{R}$ is the zero functional. And hence
\begin{eqnarray} \label{eq:g}
2\lambda \langle f_{\PP,\lambda}, g_z \rangle_{H} + \E_{\PP} \phi^{'}(Y,f_{\PP,\lambda}(X)) g_z(X) = 0.
\end{eqnarray}
We also note that since $f_{\PP,\lambda}$ is the global minimum, then $\lambda ||f_{\PP,\lambda}||_H^2 + R_{\phi}(f_{\PP,\lambda}) \le \lambda ||0_H||_H^2 + R_{\phi}(0_H) = C_{\phi}$ where $C_{\phi} = R_{\phi}(0_H)= \phi(0,0)$ is a constant, that is,
\begin{eqnarray} \label{eq:fC}
\lambda ||f_{\PP,\lambda}||_{H}^2 \le \lambda ||f_{\PP,\lambda}||_{H}^2 + \E_{\PP} \phi(Y,f_{\PP,\lambda}(X)) \le C_{\phi}.
\end{eqnarray}

Finally, we have
\begin{eqnarray*}
2\lambda ||g_z||_H^2 &\le& 2\lambda ||g_z||_H^2 + \E_{\Prob} \phi^{''}(Y,f_{\PP,\lambda}(X)) g_z^2(X) \\
&\stackrel{(i)}{=}& \E_{\Prob} \phi^{'}(Y,f_{\PP,\lambda}(X)) g_z(X) - \phi^{'}(z_y,f_{\PP,\lambda}(z_x)) g_z(z_x) \\
&\stackrel{(ii)}{=}& -2\lambda \langle f_{\PP,\lambda}, g_z \rangle_{H} - \phi^{'}(z_y,f_{\PP,\lambda}(z_x)) g_z(z_x) \\
&\stackrel{(iii)}{\le}&  2 \lambda \| f_{\PP,\lambda}\|_H \|g_z \|_H-  \phi^{'}(z_y,f_{\PP,\lambda}(z_x))) g_z(z_x) \\
&\stackrel{(iv)}{\leq}& 2 \sqrt{\lambda C_{\phi}}\| g_z\|_H + |\phi^{'}(z_y,f_{\PP,\lambda}(z_x))| |g_z(z_x)| \\
&=& 2 \sqrt{\lambda C_{\phi}}\| g_z\|_H   +|\phi^{'}(z_y,f_{\PP,\lambda}(z_x))| \langle g_z, k(z_x,\cdot) \rangle_H\\
&\stackrel{(v)}{\leq}& 2 \sqrt{\lambda C_{\phi}}\| g_z\|_H + |\phi^{'}(z_y,f_{\PP,\lambda}(z_x))| \sqrt{\langle g_z, g_z \rangle _H} \sqrt{\langle k(z_x,\cdot), k(z_x,\cdot) \rangle_H}\\
&=&2 \sqrt{\lambda C_{\phi}}\| g_z\|_H + |\phi^{'}(z_y,f_{\PP,\lambda}(z_x))| ||g_z||_H |k(z_x,z_x)|.
\end{eqnarray*}
where 
$(i)$ is due to \eqref{eq:A29}; $(ii)$ due to \eqref{eq:g}; $(iii)$ is due to the Cauchy-Schwartz inequality; $(iv)$ is due to \eqref{eq:fC}; $(v)$ is again due to the Cauchy-Schwartz inequality.

Since $k$ is a bounded kernel, there exists a $M_k>0$ such that $|k(x_1,x_2)| \le M_k$ for all $x_1,x_2 \in \mathcal{X}$. Hence, $|\phi^{'}(z_y,f_{\PP,\lambda}(z_x))| ||g_z||_H |k(z_x,z_x)| \le M_k|\phi^{'}(z_y,f_{\PP,\lambda}(z_x))| ||g_z||_H$, which in turn leads to 
   \[
    2\lambda ||g_z||^2_H  \leq     2\sqrt{\lambda C_{\phi}}\| g_z \|_H + M_k|\phi^{'}(z_y,f_{\PP,\lambda})(z_x)| \|g_z\|_H 
   \]
  and hence
    \[
    ||g_z||_H  \leq    \sqrt{ \frac{C_{\phi}}{\lambda}} + \frac{M_k |\phi^{'}(z_y,f_{\PP,\lambda})(z_x)|}{2 \lambda}
   \]

Hence , for small $\lambda$, we obtain  $ \|g_z\|_H  = O(\lambda^{-1/2})$, whereas for large lambda it is $\| g_z \|_H =O(\lambda^{-1})$.

\subsection{Proof of Theorem \ref{preconsistency}}
 

First, we  prove Theorem \ref{preconsistency} (a). The result follows from the following Lemma \ref{lem:6}, which builds on the Lemma 4 in \cite{Bartlett:07} supplemented by the uniform convergence of the  truncated loss to $\phi$-loss presented in Lemma \ref{lem:7}.  

\begin{lemma}\label{lem:6}
For an Arch boosting loss function $\phi$ satisfying Assumption 2, let $L_{\phi}$ be the Lipschitz constant of $\phi$, that is
$
L_{\phi} = \inf \{ L>0: |\phi(x) - \phi(y)| \le L|x-y|, x,y \in \realR \},
$
and $M_{\phi}$ be the maximum value of $\phi$, i.e.
$M_{\phi} = \sup_{x \in \realR} \phi(x).$
Then for $V = d_{VC} (\mathcal{H})$, $c = 24\int_0^1 \sqrt{\log \frac{8e}{\mu^2}} d \mu$ and sequences $T_n \to \infty$ and $\zeta_n \to \infty$ as $n \to \infty$, and $\delta_n \to 0$ as $n \to \infty$, with probability at least $1-\delta_n$,
\begin{eqnarray} \label{eq:71}
\sup_{f \in \pi_{\zeta_n}\circ \mathcal{F}^{T_n}} |\hat{R}_{\phi,n}(f) - R_{\phi}(f)| \le c \zeta_n L_{\phi} \sqrt{\frac{(V+1)(T_n+1)\log_2(\frac{2(T_n+1)}{\log2})}{n}} + M_{\phi} \sqrt{\frac{\log \frac{1}{\delta_n}}{2n}},
\end{eqnarray}
where the truncation function $\pi_l(\cdot)$ is defined as
$
\pi_l(x) =  l \mathbbm{1} \{x >l\} + x \mathbbm{1} \{|x| \leq l\}  -l \mathbbm{1} \{x < -l\}.
$ Furthermore, if we choose $\delta_n$, $T_n$ and $\zeta_n$ such that $\sum_{n=1}^{\infty} \delta_n < \infty$, and right hand side of \eqref{eq:71} converges to zero as $n \to \infty$, then $\sup_{f \in \pi_{\zeta_n}\circ \mathcal{F}^{T_n}} |\hat{R}_{\phi,n}(f) - R_{\phi}(f)| \to 0 \; a.s.$ as $n \to \infty$.
\end{lemma}

\begin{proof} [Proof of Lemma \ref{lem:6}]
The result mainly follows from Lemma 3 and 4 in \cite{Bartlett:07}. Since $\phi$ satisfies Assumption 2, we know $L_{\phi}$ and $M_{\phi}$ both exist and are finite, and $\max_{v \in [-\zeta_n,\zeta_n]} |\phi(v)| = M_{\phi}$ for all $n$. The almost surely convergence follows from Borel-Cantelli Lemma.
\end{proof}

In order to make the RHS of \eqref{eq:71} converging to zero, we can choose $\delta_n = \frac{1}{n^2}$, $\zeta_n = \kappa \log n$ and $T_n = n^{1-\epsilon}$ with $\kappa > 0$, $\epsilon \in (0,1)$ and $\kappa < \frac{1}{2} \epsilon$.
Then by Borel-Cantelli lemma, since $\sum_{n=1}^{\infty} \delta_n = \sum_{n=1}^{\infty}\frac{1}{n^2} < \infty$, we obtain $\sup_{f \in \pi_{\zeta_n}\circ \mathcal{F}^{T_n}} |\hat{R}_{\phi,n}(f) - R_{\phi}(f)| \to 0 \; .a.s.$ as $n \to \infty$. In order to get rid of the truncation $\pi_{\zeta_n}$, we will need the next Lemma \ref{lem:7}.

\begin{lemma}\label{lem:7}
Let $\phi$ be an Arch boosting loss function satisfying Assumption 2. Then for any increasing positive sequence $\zeta_n \to \infty$, we have
\begin{eqnarray*}
\sup_{f \in \mathcal{F}^{T_n}} \bigl| |\hat{R}_{\phi,n}(f) - R_{\phi}(f)| - |\hat{R}_{\phi,n}(\pi_{\zeta_n} \circ f) - R_{\phi}(\pi_{\zeta_n}\circ f)| \bigr| \to 0 \; a.s.
\end{eqnarray*}
when $n \to \infty$.
\end{lemma}

 \begin{proof}[Proof of Lemma \ref{lem:7}]
First note that
\begin{align} \label{eq:38}
\bigl| |\hat{R}_{\phi,n}(f) - R_{\phi}(f)| - |\hat{R}_{\phi,n}(\pi_{\zeta_n} \circ f) - R_{\phi}(\pi_{\zeta_n}\circ f)| \bigr| \nonumber \\
\le \bigl| \hat{R}_{\phi,n}(f) - R_{\phi}(f) - (\hat{R}_{\phi,n}(\pi_{\zeta_n} \circ f) - R_{\phi}(\pi_{\zeta_n}\circ f)) \bigr| \nonumber \\
\le \bigl| R_{\phi}(f) - R_{\phi}(\pi_{\zeta_n} \circ f) \bigr| + \bigl| \hat{R}_{\phi,n}(f) - \hat{R}_{\phi,n}(\pi_{\zeta_n}\circ f) \bigr|.
\end{align}

And for any Arch boosting loss function $\phi$ satisfying Assumption 2, since $\lim_{v \to \infty}\phi(v)$ and $\lim_{v \to -\infty} \phi(v)$ both exist in $\mathbb{R}$, given any $\epsilon > 0$, we know there exists $M > 0$ such that $\left| \phi(u) - \phi(w) \right| < \epsilon/2$ for all $|u|, |w| > M$ and $u w > 0$ by the definition of convergence.

Now for any $(X,Y) \in \mathcal{X} \times \mathcal{Y}$ and any $f \in \mathcal{F}^{T_n}$, if we choose $\zeta > M$, then when $|f(X)| < \zeta$, $|Y \pi_{\zeta} \circ f(X)| = |Yf(X)|$ and hence $\left| \phi(Yf(X)) - \phi(Y \pi_{\zeta} \circ f(X)) \right| = 0 < \epsilon/2$. When $|f(X)| \ge \zeta$, we have $|Y \pi_{\zeta} \circ f(X)| = |Y \zeta| = \zeta > M$, $|Yf(X)| = |f(X)| \ge \zeta > M$, and $Yf(X)Y \pi_{\zeta} \circ f(X) = f(X) \pi_{\zeta} \circ f(X) > 0$, so again $\left| \phi(Yf(X)) - \phi(Y \pi_{\zeta} \circ f(X)) \right| < \epsilon/2$ by the discussion in previous paragraph. In summary, we have
\begin{eqnarray} \label{eq:unifbound}
\left| \phi(Yf(X)) - \phi(Y \pi_{\zeta} \circ f(X)) \right| < \epsilon/2
\end{eqnarray}
for all $\zeta > M$, $(X,Y) \in \mathcal{X} \times \mathcal{Y}$ and $f \in \mathcal{F}^{T_n}$.
Now in \eqref{eq:38}, 
\[
\bigl| R_{\phi}(f) - R_{\phi}(\pi_{\zeta_n} \circ f) \bigr| 
\le \E_{X,Y} \bigl|\phi(Yf(X)) - \phi(Y\pi_{\zeta_n} \circ f(X)) \bigr|.
\]
 Since $\zeta_n$ is increasing and $\zeta_n \to \infty$ as $n \to \infty$, there exists $N \in \mathbb{N}$ such that for all $n > N$, we have $\zeta_n > M$, and hence $\E_{X,Y} \bigl|\phi(Yf(X)) - \phi(Y\pi_{\zeta_n} \circ f(X)) \bigr| < \epsilon/2$ for all $n > N$ by \eqref{eq:unifbound}.
For the empirical risk part, we have 
\[\bigl| \hat{R}_{\phi,n}(f) - \hat{R}_{\phi,n}(\pi_{\zeta_n}\circ f) \bigr| 
\le \sum_{i=1}^n 
\frac{1}{n} \bigl| \phi(Y_if(X_i)) - \phi(Y_i \pi_{\zeta_n} \circ f(X_i)) \bigr|.
\]
 But for each $i = 1, \cdots, n$, we know $(X_i, Y_i) \in \mathcal{X} \times \mathcal{Y}$, and hence by \eqref{eq:unifbound}, $\left| \phi(Y_if(X_i)) - \phi(Y_i \pi_{\zeta} \circ f(X_i)) \right| < \epsilon/2$ for all $n > N$ and for all $i$. Therefore, the right hand side of the above equation is smaller than $\sum_{i=1}^n \frac{1}{n} \frac{\epsilon}{2} = \epsilon/2$ for all $n > N$, and this is true for any i.i.d. observations $(X_1,Y_1), \cdots, (X_n,Y_n)$ from $\mathcal{X} \times \mathcal{Y}$ and any $f \in \mathcal{F}^{T_n}$.
Therefore, we have for all $n > N$, $f \in \mathcal{F}^{T_n}$ and any i.i.d. observations $(X_1,Y_1), \cdots, (X_n,Y_n)$,
$$\bigl| R_{\phi}(f) - R_{\phi}(\pi_{\zeta_n} \circ f) \bigr| + \bigl| \hat{R}_{\phi,n}(f) - \hat{R}_{\phi,n}(\pi_{\zeta_n}\circ f) \bigr| < \epsilon,$$
that is,
$\sup_{f \in \mathcal{F}^{T_n}}\bigl| R_{\phi}(f) - R_{\phi}(\pi_{\zeta_n} \circ f) \bigr| + \bigl| \hat{R}_{\phi,n}(f) - \hat{R}_{\phi,n}(\pi_{\zeta_n}\circ f) \bigr| \le \epsilon$
for all $n > N$ and $(X_1,Y_1), \cdots, (X_n,Y_n) \in \mathcal{X} \times \mathcal{Y}$.
Since $\epsilon$ is arbitrarily chosen, we know $\sup_{f \in \mathcal{F}^{T_n}}\bigl| R_{\phi}(f) - R_{\phi}(\pi_{\zeta_n} \circ f) \bigr| + \bigl| \hat{R}_{\phi,n}(f) - \hat{R}_{\phi,n}(\pi_{\zeta_n}\circ f) \bigr| \to 0$ as $n \to \infty$, and hence $\sup_{f \in \mathcal{F}^{T_n}} \bigl| |\hat{R}_{\phi,n}(f) - R_{\phi}(f)| - |\hat{R}_{\phi,n}(\pi_{\zeta_n} \circ f) - R_{\phi}(\pi_{\zeta_n}\circ f)| \bigr| \to 0$.
 
\end{proof}

By Lemma \ref{lem:7},
\begin{eqnarray*}
0 &\le& \left| \sup_{f \in \mathcal{F}^{T_n}} |\hat{R}_{\phi,n}(f) - R_{\phi}(f)| - \sup_{f \in \mathcal{F}^{T_n}} |\hat{R}_{\phi,n}(\pi_{\zeta_n}\circ f) - R_{\phi}(\pi_{\zeta_n}\circ f)| \right|\\
&\le& \sup_{f \in \mathcal{F}^{T_n}} \left| |\hat{R}_{\phi,n}(f) - R_{\phi}(f)| - |\hat{R}_{\phi,n}(\pi_{\zeta_n} \circ f) - R_{\phi}(\pi_{\zeta_n}\circ f)| \right| \to 0 \; a.s. \mbox{ as $n \to \infty$}.
\end{eqnarray*}

Therefore, if $\sup_{f \in \pi_{\zeta_n}\circ \mathcal{F}^{T_n}} |\hat{R}_{\phi,n}(f) - R_{\phi}(f)| \to 0 \; .a.s.$, that is, $\sup_{f \in \mathcal{F}^{T_n}} |\hat{R}_{\phi,n}(\pi_{\zeta_n}\circ f) - R_{\phi}(\pi_{\zeta_n}\circ f)| \to 0 \; a.s.$, then $ \sup_{f \in \mathcal{F}^{T_n}} |\hat{R}_{\phi,n}(f) - R_{\phi}(f)| \to 0 \; a.s.$ as $n \to \infty$. Theorem \ref{preconsistency} (a) is proved.

Here we want to emphasize on the boundedness of loss function $\phi$ as a key prerequisite for Theorem \ref{preconsistency} (a). For an unbounded function (e.g. exponential loss), without truncating on the classifier $f$, we may not have the above conclusion about the uniform deviation between $\hat{R}_{\phi,n}(f)$ and $R_{\phi}(f)$ in $\mathcal{F}^{T_n}$ even when $T_n \to \infty$ slow enough.

Before proving Theorem \ref{preconsistency} (b), we need  the following inequality.
\begin{lemma} \label{lem:fundamental}
For any family of functions $\mathcal{F}$ and loss function $\phi$, and for any sample $\mathcal{S}_n$, let $f_n^*= \argmin_{f \in \mathcal{F}} \hat{R}_{\phi,n}(f)$, we have
$$0 \le R_{\phi}(f_n^*) - \inf_{f\in \mathcal{F}} R_{\phi}(f) \le 2 \sup_{f \in \mathcal{F}} |\hat{R}_{\phi,n}(f) - R_{\phi}(f)|.$$
\end{lemma}
 

 \begin{proof} [Proof of Lemma \ref{lem:fundamental}]
By the choice of $f_n^*$, we know $|\hat{R}_{\phi,n}(f_n^*) - R_{\phi}(f_n^*)| \le \sup_{f \in \mathcal{F}} |\hat{R}_{\phi,n}(f) - R_{\phi}(f)|$. Then
\begin{eqnarray*}
R_{\phi}(f_n^*) - \inf_{f\in \mathcal{F}} R_{\phi}(f) &\le&  R_{\phi}(f_n^*) - \hat{R}_{\phi,n}(f_n^*) +  \hat{R}_{\phi,n}(f_n^*) - \inf_{f\in \mathcal{F}} R_{\phi}(f) \\
&\le& \sup_{f \in \mathcal{F}} |\hat{R}_{\phi,n}(f) - R_{\phi}(f)| + \sup_{f \in \mathcal{F}} \left(  \hat{R}_{\phi,n}(f_n^*) - R_{\phi}(f) \right) \\
&\le&  2 \sup_{f \in \mathcal{F}} |\hat{R}_{\phi,n}(f) - R_{\phi}(f)|.
\end{eqnarray*}
\end{proof} 

Let $\mathcal{F} = \mathcal{F}^{T_n}$ in the above lemma, then $f_n^*= \argmin_{f \in \mathcal{F}^{T_n}} \hat{R}_{\phi,n}(f)$, and hence we can bound $R_{\phi}(f^*_n) - \inf_{f \in F^{T_n}} R_{\phi}(f)$ by 
\begin{eqnarray} \label{eq:fund2}
0 \le R_{\phi}(f^*_n) - \inf_{f \in F^{T_n}} R_{\phi}(f)\le 2 \sup_{f \in \mathcal{F}^{T_n}} |\hat{R}_{\phi,n}(f) - R_{\phi}(f)|.
\end{eqnarray}
Therefore, if we can show that
\begin{eqnarray} \label{eq:37}
\sup_{f \in \mathcal{F}^{T_n}} |\hat{R}_{\phi,n}(f) - R_{\phi}(f)| \to 0 \; \; a.s.
\end{eqnarray}
as $n \to \infty$, then $R_{\phi}(f^*_n) \to R_{\phi}^*$ a.s. by inequality \eqref{eq:fund2} and Assumption 1. But \eqref{eq:37} is just Theorem \ref{preconsistency} (a).

\subsection{Proof of Lemma \ref{preconsistency2}}

For part (a), we will use Hoeffding's inequality for bounded random variables to obtain
\begin{eqnarray*}
\Prob(\hat{R}_{\phi,n}(\tilde{f}_n) - R_{\phi}(\tilde{f}_n) \ge t_n) \le \exp \left( - \frac{2n t_n^2}{M_{\phi}^2} \right) = \delta_n,
\end{eqnarray*}
where $M_{\phi} = \sup_{x \in \realR} \phi(x)$. Since $\tilde{f}_n$ only depends on $n$, we know $\{ \phi(Y_i \tilde{f}_n(X_i)) \}_{i=1}^n$ are independent. We require $t_n \to 0$ as $n \to \infty$ and $\sum_{n=1}^{\infty} \delta_n < \infty$. For example, we take $t_n = \frac{1}{n^{1/4}}$. Then by Borel-Contelli Lemma, we have Lemma \ref{preconsistency2} (a).

For part (b), by the numerical convergence Theorem \ref{numerical convergence}, we know $\hat{R}_{\phi,n}(F_{T}) \to R_{\phi,n}^* \; a.s.$ as $T \to \infty$, and the convergence rate only depends on $T$. Now let $\{T_n \}_{n=1}^{\infty}$ be a sequence with $T_n \to \infty$, we then have $\hat{R}_{\phi,n}(F_{T_n}) -R_{\phi,n}^* \to 0$ as $n \to \infty$. But $\hat{R}_{\phi,n}(\tilde{f}_n) \ge R_{\phi,n}^*$, therefore, 
$\bigl( \hat{R}_{\phi,n}(F_{T_n}) - \hat{R}_{\phi,n}(\tilde{f}_n) \bigr)_+ \to 0 \;\; a.s.$ when $n \to \infty$.

 \subsection{Proof of Corollary \ref{consistency}}
  
%
%
%
%
%

We will now adapt  the  method  of   \cite{Bartlett:07} (see Theorem 1 therein) to incorporate non-convex losses. For almost every $w \in (w, \mathcal{S}, \PP)$, we can have sequences $\epsilon_n^1(w) \to 0$, $\epsilon_n^2(w) \to 0$ and $\epsilon_n^3(w) \to 0$ such that
\begin{eqnarray}
R_{\phi} (F_{T_n}) &\le& R_{\phi,n}(F_{T_n})+ \epsilon_n^1(w) \label{eq:23} \\
&\le& R_{\phi,n}(\tilde{f}_{n}) + \epsilon_n^1(w)  + \epsilon_n^2(w) \label{eq:25} \\
&\le& R_{\phi}(\tilde{f}_{n}) + \epsilon_n^1(w)  + \epsilon_n^2(w) + \epsilon_n^3(w). \label{eq:26}
\end{eqnarray}
\eqref{eq:23} follows from Theorem \ref{preconsistency} (a), \eqref{eq:25} follows from Theorem \ref{preconsistency2} (b), and \eqref{eq:26} follows from Theorem \ref{preconsistency2} (a). 

Finally by Theorem \ref{preconsistency} (b) and how we chosen the reference sequence $\{ \tilde{f}_n \}$, we have $R_{\phi}(\tilde{f}_n) \to R_{\phi}^*$ as $n \to \infty$. It follows from \eqref{eq:26} that $R_{\phi} (F_{T_n}) \to R_{\phi}^*$ a.s. as $n \to \infty$. Since $\phi$ is an Arch boosting function, we know $\phi$ is classification-calibrated, and hence by Theorem 3 in \cite{Bartlett:06},
\begin{eqnarray*}
L(sign(F_{T_n})) \to L^* \; \; a.s.
\end{eqnarray*}
as $n \to \infty$.

 \end{appendices}

\end{document}